\documentclass[twoside,11pt]{article}

\usepackage{blindtext}

\usepackage[abbrvbib, preprint]{jmlr2e} 

\usepackage{jmlr2e}
\usepackage{mathtools}

\usepackage{enumitem}
\newtheorem{assumption}{Assumption}
\usepackage{amsmath}
\usepackage{amssymb}
\usepackage{mathrsfs}
\usepackage{bm}
\usepackage{booktabs}
\usepackage{multirow}
\usepackage{bbding}
\usepackage{verbatim}
\usepackage{float}

\usepackage{microtype}
\usepackage{graphicx}
\usepackage{subfigure}
\usepackage{booktabs} 

\usepackage{color}  

\ShortHeadings{Optimal Rates of Kernel Ridge Regression under Source Condition in Large Dimensions}{Zhang, Li, Lu and Lin}
\firstpageno{1}

\begin{document}
\raggedbottom  
\title{Optimal Rates of Kernel Ridge Regression under Source Condition in Large Dimensions}

\author{\name Haobo Zhang \email zhang-hb21@mails.tsinghua.edu.cn \\ 
\AND
\name Yicheng Li  \email liyc22@mails.tsinghua.edu.cn\\ 
\AND
\name Weihao Lu  \email luwh19@mails.tsinghua.edu.cn\\ 
\AND
\name Qian Lin \thanks{Corresponding author} \email qianlin@tsinghua.edu.cn \\ 
      \addr Center for Statistical Science, Department of Industrial Engineering\\
      Tsinghua University
      }
      
\editor{My editor}

\maketitle

\begin{abstract}

Motivated by the studies of neural networks (e.g.,the neural tangent kernel theory), we perform a study on the large-dimensional behavior of kernel ridge regression (KRR)  where the sample size $n \asymp d^{\gamma}$ for some $\gamma > 0$. 
Given an RKHS $\mathcal{H}$ associated with an inner product kernel defined on the sphere $\mathbb{S}^{d}$, we suppose that the true function $f_{\rho}^{*} \in [\mathcal{H}]^{s}$, the interpolation space of $\mathcal{H}$ with source condition $s>0$. We first determined the exact order (both upper and lower bound) of the generalization error of kernel ridge regression for the optimally chosen regularization parameter $\lambda$. We then further showed that when $0<s\le1$, KRR is minimax optimal; and when $s>1$, KRR is not 
minimax optimal (a.k.a. {\it the saturation effect}). 
Our results illustrate that the curves of rate varying along $\gamma$ exhibit the \textit{periodic plateau behavior} and the \textit{multiple descent behavior} and show how the curves evolve with $s>0$.
Interestingly, our work provides a unified viewpoint of several recent works on kernel regression in the large-dimensional setting, which correspond to $s=0$ and $s=1$ respectively.

\end{abstract}


\section{Introduction}
The recent studies of neural network theory have brought the renaissance of kernel methods, since the neural tangent kernel \citep{jacot2018_NeuralTangent} provides a natural surrogate to understand the wide neural network \citep{Arora_on_2019,lee2019wide,jianfa2022generalization}.
When the dimension of data is fixed, there has been extensive literature studying the generalization behavior of kernel ridge regression (KRR), one of the most popular kernel methods, e.g., \cite{Caponnetto2007OptimalRF,fischer2020_SobolevNorm,Cui2021GeneralizationER}, etc.. 
Researchers usually use two crucial factors to characterize KRR's generalization behavior: \textit{capacity condition} and \textit{source condition}. 
Supposing eigenvalues associated with the RKHS $\mathcal{H}$ are  $ \{\lambda_{i}\}_{i=1}^{\infty}$, the capacity condition (also known as effective dimension) assumes $ \mathcal{N}_{1}(\lambda) := \sum_{i=1}^{\infty} \lambda_{i} / (\lambda_{i} + \lambda) \asymp \lambda^{-\frac{1}{\beta}} $ for some $\beta > 1$, where $\lambda$ will represent the regularization parameter in KRR. 
The capacity condition characterizes the size of $\mathcal{H}$ and is frequently stated as an equivalent eigenvalue decay condition: $ \lambda_{i} \asymp i^{-\beta}, \beta > 1$. 
The source condition assumes that the true function $f_{\rho}^{*}$ falls into $ [\mathcal{H}]^{s}$, an interpolation space of $ \mathcal{H}$ for some $s>0$. It characterizes the relative smoothness of $f_{\rho}^{*}$ with respect to $\mathcal{H}$: the larger $ s $ is, the ``smoother" $f_{\rho}^{*}$ is and the easier it can be estimated.
Under this framework, many interesting topics about KRR's generalization behavior were studied. 
For instance, the  minimax optimality of KRR \citep{fischer2020_SobolevNorm,zhang2023optimality_2} when $0<s\leq 2$, the saturation effect of KRR \citep{bauer2007_RegularizationAlgorithms,li2023_SaturationEffect} when $s>2$, the generalization ability of kernel interpolation \citep{beaglehole2023inconsistency,Li2023KernelIG} and the learning curve of KRR \citep{Cui2021GeneralizationER,li2023asymptotic}, etc. We refer to Section \ref{subsection related work} for more related work about these topics. These results help us clarify several puzzle points. 
For example, \cite{li2023_SaturationEffect} implies that when the true function is smooth enough (e.g., $s>2$), the early stopping kernel gradient flow is often better than the kernel ridge regression and \cite{Li2023KernelIG} asserts that if a wide neural network overfits the data, it generalizes poorly.

Since neural networks often perform well on data with large dimensionality where $n\asymp d^{\gamma}$ for some $\gamma>0$, we expect that the studies of kernel regression in large-dimensional data can provide us more guidance about the generalization behavior of neural network for large-dimensional data. 
However, in contrast to the rich theoretical results about kernel regression in the fixed-dimensional setting, much less is known about the aforementioned topics in the large-dimensional setting. 
Since \cite{Karoui_spectrum_2010}  provided an approximation of the kernel random matrix when $n\asymp d$, few works have been done for kernel regression in large-dimensional or high-dimensional settings until recently.
The first obstacle is that if  $d$ is large, the eigenvalues of the RKHS usually depend on $d$ in an unpleasant way. Therefore, the polynomial eigenvalue decay rate assumption $ \lambda_{i} \asymp i^{-\beta}$ must not be true (e.g., the inner product kernel on the sphere in Section \ref{section app inner}). 
Second, we find that $ \mathcal{N}_{1}(\lambda) $ is not enough to characterize a tight upper bound of the convergence rate, which is a significant difference from the fixed-dimensional setting. 
We will see that more information about the eigenvalues (RKHS) is needed, for instance, $ \mathcal{N}_{2}(\lambda):= \sum_{i=1}^{\infty} \left( \lambda_{i} / (\lambda_{i} + \lambda) \right)^{2} $ which will be introduced in \eqref{n1 n2 m1 m2}.

\begin{figure}
\centering
\subfigure[]{
\includegraphics[width=0.3\columnwidth]{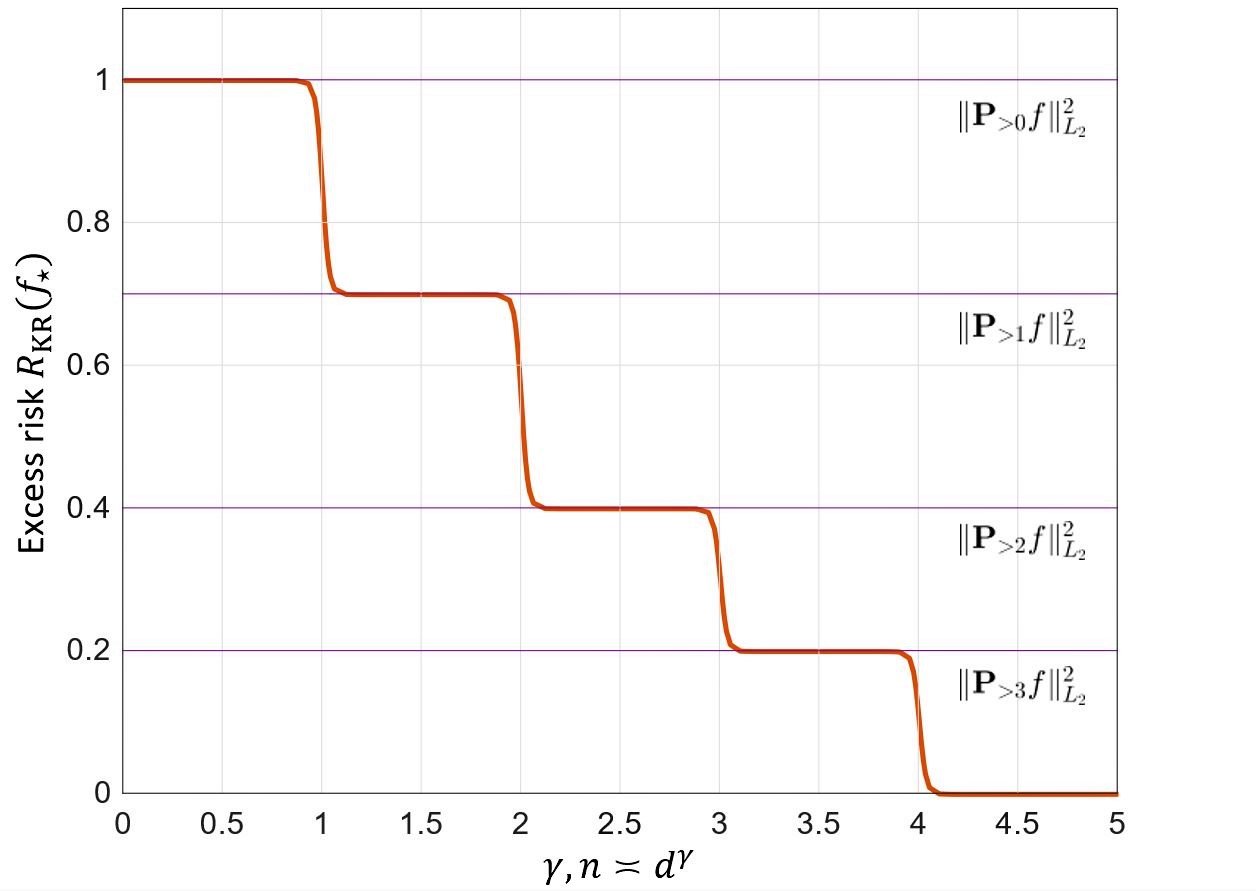}
\label{subfig:montanari}    
    }
\subfigure[]{
\includegraphics[width=0.3\columnwidth]{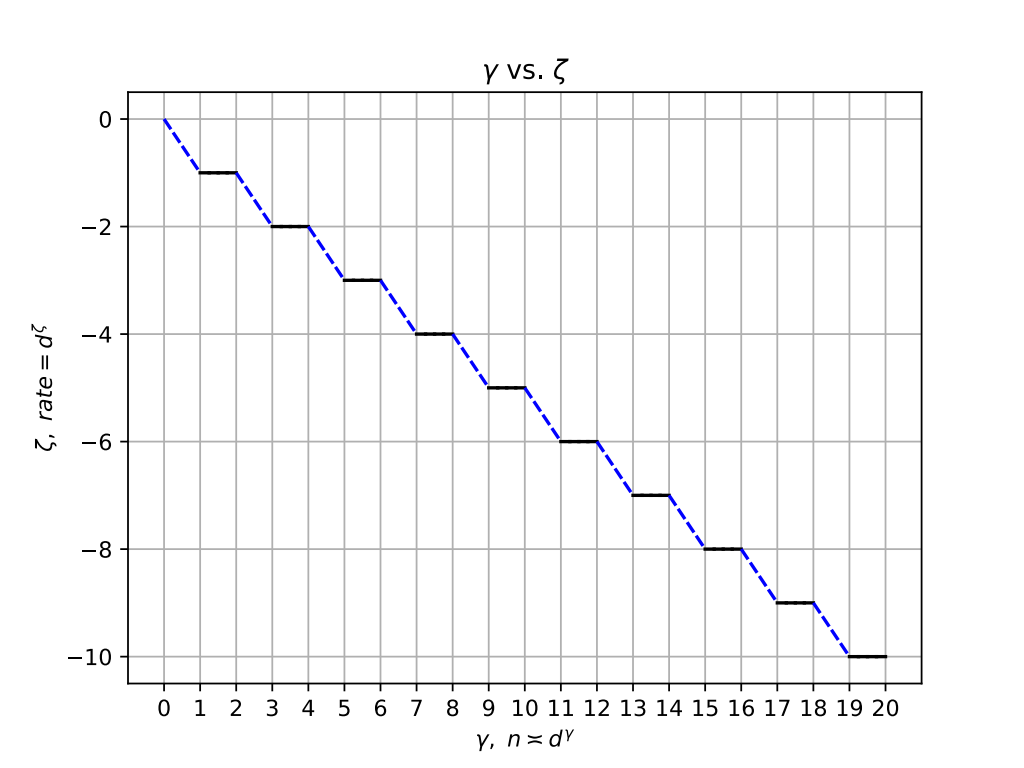}
\label{subfig:Lu}
}
 \caption{    Left: The curve of generalization error for estimating $f_{\rho}^{*} \in L^{2}$ (Figure 5 in \cite{Ghorbani2019LinearizedTN}).  Right: The curve of the minimax rates for estimating $f_{\rho}^{*} \in \mathcal{H}$ (Figure 2(b) in \cite{lu2023optimal}). } 
 \label{figure comparison}
\end{figure}

There are several recent works investigating the generalization error of kernel regression in the large-dimensional setting where $ n \asymp d^{\gamma}, \gamma > 0$. \cite{Ghorbani2019LinearizedTN} considers the square-integrable function space on the sphere $\mathbb{S}^{d}$ and proves that when $\gamma$ is a non-integer, KRR is consistent if and only if the true function is a polynomial with a fixed degree $\le \gamma$. They also qualitatively reveal that the excess risk exhibits a periodic plateau behavior (Figure \ref{subfig:montanari}).  \cite{Liu_kernel_2021} considers the setting $n \asymp d$ and assumes source condition to be $s \in (0,2]$. They give an upper bound of the generalization error with respect to bias and variance. Using this upper bound, they demonstrate that there could be multiple shapes of the generalization curve as the sample size increases. 
A more recent work \cite{lu2023optimal}  studies early stopping kernel gradient flow in the large-dimensional setting $ n \asymp d^{\gamma}$. Assuming $ f_{\rho}^{*} \in \mathcal{H}$  and considering the inner product kernel on the sphere $ \mathbb{S}^{d}$, they prove an upper bound of the convergence rate and show the minimax optimality of early stopping kernel gradient flow. Interestingly, their results indicate that the minimax rate of the kernel regression for $ f_{\rho}^{*} \in \mathcal H$ exhibits the similar periodic plateau phenomenon (Figure \ref{subfig:Lu}). This raises a natural and interesting question: is there a unified way to explain the periodic plateau behavior appeared in \cite{lu2023optimal} and \cite{Ghorbani2019LinearizedTN}?

Suppose that $\mathcal H$ is an RKHS associated with an inner product kernel defined on $\mathbb S^{d}$. The main focus of this paper aims to derive the matching upper and lower bounds of the generalization error and discuss the minimax optimality of KRR for general source condition $s > 0$, i.e., the true regression function $f_{\rho}^{*} \in [\mathcal H]^{s}$. 
Allowing $s$ to vary is not only a more reasonable assumption for the true function, but also provides 
 us a natural framework to clarify the relation between the results in \cite{lu2023optimal} and \cite{Ghorbani_When_2021}.
In fact, a little bit interpolation space theory suggests that both the results in \cite{Ghorbani2019LinearizedTN} and \cite{lu2023optimal}
 are two special cases of our results corresponding to $s=0$ and $s=1$ respectively.
This paper has the following contributions:



\begin{itemize}[leftmargin = 18pt]
    \item We consider a more general framework than the traditional capacity-source condition. We introduce $ \mathcal{N}_{1}(\lambda), \mathcal{N}_{2}(\lambda), \mathcal{M}_{1}(\lambda)$ and $ \mathcal{M}_{2}(\lambda)$ in \eqref{n1 n2 m1 m2}, which are key quantities depending on the RKHS, the true function and the regularization parameter $\lambda$. Under mild assumptions, we use these key quantities to express the matching upper and lower bounds of the generalization error as long as the regularization parameter satisfies some approximation conditions (Theorem \ref{main theorem}). This framework makes few assumptions on the eigenvalues of the RKHS and the true function, thus enabling us to handle the large-dimensional setting and general source condition later. In the fixed-dimensional setting, our results in Theorem \ref{main theorem} also recovers the state-of-the-art theoretical results about the exact convergence rates of KRR in \cite{li2023asymptotic}.

    

    \item We then add source condition into our new framework and consider the inner product kernel on the sphere $\mathbb{S}^{d}$. When $ n \asymp d^{\gamma},$ we derive exact convergence rates (both upper and lower bounds) of the generalization error under the best choice of regularization parameter for any source condition $s > 0$ and almost all $ \gamma > 0$ (Theorem \ref{theorem inner s ge 1} for $s \ge 1$ and Theorem \ref{theorem inner s le 1} for $ 0<s<1$). We will see that the curves of rate varying along $\gamma$ show similar \textit{periodic plateau} and \textit{multiple descent} behavior as in \cite{lu2023optimal}. Moreover, we will see that the shapes of curves vary with $s$ and are totally different when $0<s<1 $ and $ s \ge 1$, with even more intriguing results in the limiting case $s \to 0 $ and $ s \to 2$. 


    \item For the inner product kernel on the sphere $\mathbb{S}^{d}$, we further derive the corresponding minimax lower bound for all $ s > 0$ and $\gamma > 0$. When $ 0< s <1$, the exact rates in Theorem \ref{theorem inner s le 1} match the minimax lower bound, and thus we prove the minimax optimality of KRR. 
    When $ s > 1$, the KRR is not minimax optimal, i.e., we discover a new version of the saturation effect of KRR. In the fixed-dimensional setting, the saturation effect of KRR only happens when $s > 2$. In the large-dimensional setting, we prove that a similar phenomenon also happens for $ 1 < s \le 2$. Specifically, for any $s > 1$, there will be corresponding ranges of $\gamma$ such that the convergence rates of KRR can not achieve the minimax lower bound even under the best choice of regularization parameter. 
    

\end{itemize}

\subsection{Related work} \label{subsection related work}

In the introduction, we have mentioned several interesting topics about the KRR's generalization behavior. These topics have been well-studied in the fixed-dimensional setting. The first essential question is the minimax optimality of KRR. Under the framework of capacity condition and source condition ($s$), \cite{Caponnetto2007OptimalRF} proves the minimax optimality of KRR when $ 1 \le s \le 2$. Then, extensive literature (see, e.g., \citealt{steinwart2009_OptimalRates,lin2018_OptimalRates,fischer2020_SobolevNorm,zhang2023optimality,zhang2023optimality_2} and the reference therein) studies the mis-specified case ($0 < s < 1$), where \citet{zhang2023optimality} proves the minimax optimality for all $ 0< s \le 2 $ under further embedding index condition. The second question is the saturation effect of KRR when $s>2$, which is conjectured by \cite{bauer2007_RegularizationAlgorithms,gerfo2008_SpectralAlgorithms} and rigorously proved by \cite{li2023_SaturationEffect}. Thirdly, due to the fantastic performance of overparameterized neural networks, the generalization ability of kernel interpolation has also raised a lot of interest. The results in \cite{rakhlin2019consistency,buchholz2022_KernelInterpolation, beaglehole2023inconsistency,Li2023KernelIG} imply that kernel interpolation can not generalize in fixed-dimensional setting. Last but not least, \cite{Bordelon2020SpectrumDL,Cui2021GeneralizationER,li2023asymptotic} study the learning curve of KRR, i.e., the exact generalization error (or exact order) for any regularization parameter $\lambda > 0$.

In the large-dimensional setting, the answers to the above questions are not clear yet. Many researchers have studied these problems from different angles and settings. A line of work uses the tools of high-dimensional kernel random matrix approximation from \cite{Karoui_spectrum_2010} and studies the generalization ability of kernel interpolation \citep{Liang_Just_2019,liang2020multiple}. When $ n \asymp d$, \cite{Liang_Just_2019} gives an upper bound of the generalization error of kernel interpolation and claims that the upper bound tends to 0 when the data exhibits a low-dimensional structure. Further, when $ n \asymp d^{\gamma}, \gamma > 0$, \cite{liang2020multiple} gives an upper bound with a specific convergence rate, which implies that kernel interpolation can generalize if and only if $\gamma$ is not an integer. One closely related topic is the benign overfitting phenomenon, which we refer to \cite{bartlett2020benign,10.1214/21-AOS2133,9051968,tsigler2020benign}. 

Another line of work follows \cite{Ghorbani2019LinearizedTN}, which has been mentioned in the introduction. This line of work adopts the square-integrable assumption of the true function and aims to obtain the exact generalization error of kernel methods in various settings \citep{Ghorbani_When_2021,mei2022generalization,mei2022generalization2,Ghosh_three_2021,xiao2022precise,hu2022sharp,misiakiewicz_spectrum_2022,Donhauser_how_2021}. To our knowledge, \cite{lu2023optimal} is the only literature that provides the minimax optimality result for specific kernel methods. As discussed in the introduction, \cite{lu2023optimal} considers the $s=1$ case ($ f_{\rho}^{*} \in \mathcal{H}$) and kernel early stopping gradient flow. We will provide a detailed discussion on \cite{lu2023optimal} and \cite{Ghorbani2019LinearizedTN} in Section \ref{section discussion}. If we follow the line of research in the fixed-dimensional setting, considering general source condition $s > 0$ and studying the minimax optimality (saturation effect) of KRR are essential steps to understand KRR's generalization behavior in the large-dimensional setting.

\section{Preliminaries}\label{section pre}
Let a compact set $\mathcal{X} \subseteq \mathbb{R}^{d}$ be the input space and $ \mathcal{Y} \subseteq \mathbb{R}$ be the output space. Let $ \rho = \rho_{d} $ be an unknown probability distribution on $\mathcal{X} \times \mathcal{Y}$ satisfying $ \int_{\mathcal{X} \times \mathcal{Y}} y^{2} \mathrm{d}\rho(\boldsymbol{x},y) <\infty$ and denote the corresponding marginal distribution on $ \mathcal{X} $ as $\mu = \mu_{d}$. We use $L^{p}(\mathcal{X},\mu)$ (in short $L^{p}$) to represent the $L^{p}$-spaces. Throughout the paper, we make the following assumption:
\begin{assumption}\label{assumption kernel}
    Suppose that $ \mathcal{H} = \mathcal{H}_{d} $ is a separable RKHS on $\mathcal{X} \subset \mathbb{R}^{d}$ with respect to a continuous kernel function $ k = k_{d}$ satisfying 
    \begin{displaymath}
        \sup\limits_{\boldsymbol{x} \in \mathcal{X}} k_{d}(\boldsymbol{x},\boldsymbol{x}) \le \kappa^{2},
    \end{displaymath}
    where $ \kappa$ is an absolute constant.
\end{assumption}

Since we allow the dimension $d$ to diverge to infinity as $n \to \infty$, the RKHS may vary with $d$ and we suppose that Assumption \ref{assumption kernel} holds uniformly for all $d$. For brevity of notations, we frequently omit the index $d$ in the rest of this paper.\\

Suppose that the samples $\{ (\boldsymbol{x}_{i}, y_{i}) \}_{i=1}^{n}$ are i.i.d. sampled from $\rho$. Kernel ridge regression (KRR) constructs an estimator $\hat{f}_{\lambda}$ by solving the penalized least square problem
\begin{displaymath}
    \hat{f}_\lambda = \underset{f \in \mathcal{H}}{\arg \min } \left(\frac{1}{n} \sum_{i=1}^n\left(y_i-f\left( \boldsymbol{x}_{i} \right)\right)^2+\lambda\|f\|_{\mathcal{H}}^2\right),
\end{displaymath}
where $\lambda > 0$ is referred to as the regularization parameter. 

Denote the samples as $\boldsymbol{X}=\left( \boldsymbol{x}_{1},\cdots,\boldsymbol{x}_{n} \right)$ and $\textbf{y}=\left( y_{1},\cdots,y_{n} \right)^{\top}$. The representer theorem (see, e.g., \citealt{Steinwart2008SupportVM}) gives an explicit expression of the KRR estimator, i.e., 
\begin{equation}\label{krr estimator}
   \hat{f}_{\lambda}(\boldsymbol{x}) = \mathbb{K}(\boldsymbol{x}, \boldsymbol{X})(\mathbb{K}(\boldsymbol{X}, \boldsymbol{X})+n \lambda \mathbf{I})^{-1} \mathbf{y},
\end{equation}
where 
\begin{displaymath}
  \mathbb{K}(\boldsymbol{X}, \boldsymbol{X})=\left(k\left(\boldsymbol{x}_i, \boldsymbol{x}_j\right)\right)_{n \times n},~~ \mathbb{K}(\boldsymbol{x}, \boldsymbol{X})=\left(k\left(\boldsymbol{x}, \boldsymbol{x}_1\right), \cdots, k\left(\boldsymbol{x}, \boldsymbol{x}_n\right)\right).
\end{displaymath}
Denote the conditional mean:
\begin{displaymath}
  f_{\rho}^*(\boldsymbol{x}) = f_{\rho, d}^*(\boldsymbol{x}) \coloneqq \mathbb{E}_{\rho}[~y \;|\; \boldsymbol{x}~] = \int_{\mathcal{Y}} y ~ \mathrm{d} \rho(y|\boldsymbol{x}).
\end{displaymath}
We are interested in the convergence rates of the generalization error (excess risk) of $\hat{f}_{\lambda}$: 
\begin{displaymath}
    \mathbb{E}_{\boldsymbol{x} \sim \mu} \left[ \left( \hat{f}_{\lambda}(\boldsymbol{x}) - f_{\rho}^{*}(\boldsymbol{x}) \right)^{2} \right] = \left\|\hat{f}_\lambda-f_\rho^*\right\|_{L^2}^2.
\end{displaymath}

\textit{Notations.} We use asymptotic notations $O(\cdot),~o(\cdot),~\Omega(\cdot)$ and $\Theta(\cdot)$.
We also write $a_n \asymp b_n$ for $a_n = \Theta(b_n)$; $a_n \lesssim b_n$ for $a_n = O(b_n)$; $a_n \gtrsim b_n$ for $a_n = \Omega(b_n)$; $ a_n \ll b_n$ for $a_n = o(b_n)$. We will also use the probability versions of the asymptotic notations such as $O_{\mathbb{P}}(\cdot), o_{\mathbb{P}}(\cdot), \Omega_{\mathbb{P}}(\cdot), \Theta_{\mathbb{P}}(\cdot)$. For instance, we say the random variables $ X_{n}, Y_{n} $ satisfying $ X_{n} =  O_{\mathbb{P}}(Y_{n}) $ if and only if for any $\epsilon > 0$, there exist a constant $C_{\epsilon} $ and $ N_{\epsilon}$ such that $ P\left( |X_{n}| \ge C_{\epsilon} |Y_{n}| \right) \le \epsilon, \forall n > N_{\epsilon}$.



\subsection{Integral operator and interpolation space}
Denote the natural embedding inclusion operator as $S_{k}: \mathcal{H} \to L^{2}(\mathcal{X},\mu)$. Then its adjoint operator $S_{k}^{*}: L^{2}(\mathcal{X},\mu) \to \mathcal{H}  $ is an integral operator, i.e., for $f \in L^{2}(\mathcal{X},\mu)$ and $\boldsymbol{x} \in \mathcal{X}$, we have 
\begin{displaymath}
\left(S_{k}^{*} f\right)(\boldsymbol{x})=\int_\mathcal{X} k\left(\boldsymbol{x}, \boldsymbol{x}^{\prime}\right) f\left(\boldsymbol{x}^{\prime}\right) \mathrm{d} \mu\left(\boldsymbol{x}^{\prime}\right).
\end{displaymath}
Under Assumption \ref{assumption kernel}, $S_{k}$ and $S_{k}^{*}$ are Hilbert-Schmidt operators (thus compact) and the HS norms (denoted as $\left\| \cdot \right\|_{2}$) satisfy that
    \begin{displaymath}
        \left\| S_{k}^{*} \right\|_{2} = \left\| S_{k} \right\|_{2} = \|k\|_{L^{2}(\mathcal{X},\mu)}:=\left(\int_\mathcal{X} k(\boldsymbol{x}, \boldsymbol{x}) \mathrm{d} \mu(\boldsymbol{x})\right)^{1 / 2} \le \kappa.
    \end{displaymath}
Next, we can define two integral operators: 
\begin{equation}
    L_{k}:=S_{k} S_{k}^{*}: L^{2}(\mathcal{X},\mu) \rightarrow L^{2}(\mathcal{X},\mu), \quad  \quad T:=S_{k}^{*} S_{k}: \mathcal{H} \rightarrow \mathcal{H}.
\end{equation}
$L_{k}$ and $T$ are self-adjoint, positive-definite and trace class (thus Hilbert-Schmidt and compact) and the trace norms (denoted as $\left\| \cdot \right\|_{1}$) satisfy that
\begin{displaymath}
\left\|L_{k}\right\|_{1}=\left\|T\right\|_{1}=\left\|S_{k}\right\|_{2}^2=\left\|S_{k}^{*}\right\|_{2}^2.
\end{displaymath}
The spectral theorem for self-adjoint compact operators yields that there is an at most countable index set $N$, a non-increasing summable sequence $\{ \lambda_{i} \}_{i \in N} \subseteq (0,\infty)$ and a family $\{ e_{i} \}_{i \in N} \subseteq \mathcal{H} $, such that $\{ e_{i} \}_{i \in N}$ is an orthonormal basis (ONB) of $\overline{\operatorname{ran} S_{k}} \subseteq L^{2}(\mathcal{X},\mu)$ and $\{ \lambda_{i}^{1/2} e_{i} \}_{i \in N}$ is an ONB of $\mathcal{H}$. Further, the integral operators can be written as 
\begin{displaymath}
    L_{k}=\sum_{i \in N} \lambda_i\left\langle\cdot,e_{i} \right\rangle_{L^{2}}e_{i} \quad \text { and } \quad T=\sum_{i \in N} \lambda_i\left\langle\cdot, \lambda_i^{1/2} e_i\right\rangle_{\mathcal{H}} \lambda_i^{1/2} e_i.
\end{displaymath}
We refer to $\{ e_{i} \}_{i \in N}$ and $\{ \lambda_{i} \}_{i \in N}$ as the eigenfunctions and eigenvalues. The celebrated Mercer's theorem (see, e.g., \citealt[Theorem 4.49]{Steinwart2008SupportVM}) shows that 
\begin{displaymath}\label{Mercer decomposition}
    k\left(\boldsymbol{x}, \boldsymbol{x}^{\prime}\right)=\sum_{i \in N} \lambda_i e_i(\boldsymbol{x}) e_i\left(\boldsymbol{x}^{\prime}\right), \quad \boldsymbol{x}, \boldsymbol{x}^{\prime} \in \mathcal{X},
\end{displaymath}
where the convergence is absolute and uniform for $\boldsymbol{x},\boldsymbol{x}^{\prime}$. 
      
Since we are going to consider the source condition in this paper, we need to introduce the interpolation spaces (power spaces) of RKHS. For any $ s \ge 0$, the fractional power integral operator $L_{k}^{s}: L^{2}(\mathcal{X},\mu) \to L^{2}(\mathcal{X},\mu)$ is defined as 
\begin{displaymath}
  L_{k}^{s}(f)=\sum_{i \in N} \lambda_i^{s} \left\langle f, e_i\right\rangle_{L^2} e_i.
\end{displaymath}
Then the interpolation space (power space) $[\mathcal{H}]^s $ is defined as
\begin{equation}\label{def interpolation space}
  [\mathcal{H}]^s := \operatorname{Ran} L_{k}^{s/2} = \left\{\sum_{i \in N} a_i \lambda_i^{s / 2}e_{i}: \left(a_i\right)_{i \in N} \in \ell_2(N)\right\} \subseteq L^{2}(\mathcal{X},\mu),
\end{equation}
equipped with the inner product
\begin{displaymath}
    \langle f, g\rangle_{[\mathcal{H}]^s}=\left\langle L_{k}^{-\frac{s}{2}} f, L_{k}^{-\frac{s}{2}} g\right\rangle_{L^2} .
\end{displaymath}
It is easy to show that $[\mathcal{H}]^s $ is also a separable Hilbert space with orthogonal basis $ \{ \lambda_{i}^{s/2} e_{i}\}_{i \in N}$. Specially, we have $[\mathcal{H}]^0 \subseteq L^{2}(\mathcal{X},\mu) $ and $[\mathcal{H}]^1 = \mathcal{H}$. For $0 < s_{1} < s_{2}$, the embeddings $ [\mathcal{H}]^{s_{2}} \hookrightarrow[\mathcal{H}]^{s_{1}} \hookrightarrow[\mathcal{H}]^0 $ exist and are compact \citep{fischer2020_SobolevNorm}. For the functions in $[\mathcal{H}]^{s}$ with larger $s$, we say they have higher regularity (smoothness) with respect to the RKHS. 

In the following of this paper, we assume $|N| = \infty$. Also note that $\{ \lambda_{i} \}_{i=1}^{\infty
}$ and $ \{ e_{i} \}_{i=1}^{\infty
}$ are dependent on $ \mathcal{H}$, thus are dependent on $d$.

\section{Main results}\label{section main}
\subsection{KRR's generalization error in the general case}
In this subsection, we consider a general framework for studying the generalization error of KRR, where we make quite mild assumptions on the RKHS $\mathcal{H}$ and the true function $f_{\rho}^{*}$. 
Note that in this framework, we allow $d$ to diverge to infinity as sample size $n$ and allow $\mathcal{H}, f_{\rho}^{*}$ to change with $d$. Therefore, the results in this subsection are applicable in the large-dimensional setting $ n \asymp d^{\gamma}, \gamma >0$.

Given the RKHS $\mathcal{H}$ and denote the true function as $ f_{\rho}^{*} = \sum\limits_{i=1}^{\infty} f_{i} e_{i}(\boldsymbol{x}) \in L^{2}(\mathcal{X},\mu) $. We define the following important quantities:
\begin{equation}\label{n1 n2 m1 m2}
\begin{aligned}
\mathcal{N}_{1}(\lambda) &= \sum\limits_{i=1}^{\infty} \left( \frac{\lambda_{i}}{\lambda_{i} + \lambda} \right);~~ \mathcal{N}_{2}(\lambda) = \sum\limits_{i=1}^{\infty} \left( \frac{\lambda_{i}}{\lambda_{i} + \lambda}\right)^{2}; \\
   \mathcal{M}_{1}(\lambda) &= \operatorname*{ess~sup}_{\boldsymbol{x} \in \mathcal{X}} \left|\sum\limits_{i=1}^{\infty} \left( \frac{\lambda}{\lambda_{i} + \lambda} f_{i} e_{i}(\boldsymbol{x}) \right) \right| ;~~ \mathcal{M}_{2}(\lambda) = \sum\limits_{i=1}^{\infty} \left( \frac{\lambda}{\lambda_{i} + \lambda} f_{i}\right)^{2} .
\end{aligned}
\end{equation}

\begin{assumption}\label{assumption noise}
    Suppose that for some absolute constant $\sigma > 0$,
    \begin{displaymath}
    \mathbb{E}_{(\boldsymbol{x},y )\sim \rho} \left[ \left( y-f^{*}_{\rho}(\boldsymbol{x}) \right)^2 \;\Big|\; \boldsymbol{x} \right] = \sigma^2,
    \quad \mu\text{-a.e. } \boldsymbol{x} \in \mathcal{X}.
  \end{displaymath}
\end{assumption}

Assumption \ref{assumption noise} assumes that the noise is non-vanishing and it holds for common nonparametric regression model $y = f_{\rho}^{*}(\boldsymbol{x}) + \epsilon $ where $ \epsilon $ is an independent non-zero noise.

\begin{assumption}\label{assumption eigenfunction}
    Suppose that
    \begin{align}\label{assumption eigen - n2}
    \operatorname*{ess~sup}_{\boldsymbol{x} \in \mathcal{X}} \sum\limits_{i=1}^{\infty} \left( \frac{\lambda_{i}}{\lambda_{i} + \lambda} \right)^{2}e_{i}^{2}(\boldsymbol{x}) \le \mathcal{N}_{2}(\lambda),
  \end{align}
  and 
  \begin{align}\label{assumption eigen - n1}
    \operatorname*{ess~sup}_{\boldsymbol{x} \in \mathcal{X}} \sum\limits_{i=1}^{\infty} \frac{\lambda_{i}}{\lambda_{i} + \lambda} e_{i}^{2}(\boldsymbol{x}) \le \mathcal{N}_{1}(\lambda).
  \end{align}
\end{assumption}

In fact, it is allowed to multiply an absolute constant $C$ on the right sides of \eqref{assumption eigen - n2} and \eqref{assumption eigen - n1}. Without loss of generality, we consider the constant to be $C=1$. 
Assumption \ref{assumption eigenfunction} naturally holds for RKHSs with uniformly bounded eigenfunctions, i.e., $\sup_{i \ge 1} \sup_{\boldsymbol{x} \in \mathcal{X}} |e_{i}(\boldsymbol{x})| \le 1$. One can also show that RKHSs associated with inner product kernel on the sphere with uniform distribution satisfy Assumption \ref{assumption eigenfunction} (see Lemma \ref{lemma inner assumption holds}). \\


    

Now we begin to state the first important theorem in this paper.
\begin{theorem}\label{main theorem}
    Let $\mathcal{N}_{1}, \mathcal{N}_{2}, \mathcal{M}_{1}, \mathcal{M}_{2}$ be defined as \eqref{n1 n2 m1 m2}, and let $d=d(n)$ which is allowed to diverge with $n \to \infty$. Suppose that Assumption \ref{assumption kernel}, \ref{assumption noise} and \ref{assumption eigenfunction} hold. Let $\hat{f}_{\lambda}$ be the KRR estimator defined by \eqref{krr estimator}. If the following approximation conditions hold for some $\lambda = \lambda(d,n) \to 0$:
    \begin{align}\label{approximation conditions}
       \frac{\mathcal{N}_{1}(\lambda)}{n} \ln{n} &= o(1);~~ n^{-1} \mathcal{N}_{1}(\lambda)^{2} \ln{n} = o\left(\mathcal{N}_{2}(\lambda)\right); ~~ n^{-1} \mathcal{N}_{1}(\lambda)^{\frac{1}{2}} \mathcal{M}_{1}(\lambda) = o\left(\mathcal{M}_{2}(\lambda)^{\frac{1}{2}}\right),
    \end{align}
    then we have 
    \begin{equation}\label{learning curve results}
        \mathbb{E}\left[\left\|\hat{f}_\lambda-f_\rho^*\right\|_{L^2}^2 \;\Big|\; \boldsymbol{X} \right] = \Theta_{\mathbb{P}}\left( \frac{\sigma^{2} \mathcal{N}_{2}(\lambda)}{n} + \mathcal{M}_{2}(\lambda) \right).
    \end{equation}
    The notation $ \Theta_{\mathbb{P}} $ only involves absolute constants.
\end{theorem}


Note that the notation $o(\cdot)$ represents the limit as $n \to \infty$ and we allow $d=d(n)$ to diverge to infinity with $n$. Theorem \ref{main theorem} provides the matching upper and lower bounds \eqref{learning curve results} for all $\lambda$ satisfying the approximation conditions \eqref{approximation conditions}. Generally speaking, the conditions in \eqref{approximation conditions} are more likely to hold for larger $\lambda$. For instance, we will show in the proof of Theorem \ref{theorem inner s ge 1} that if the conditions in \eqref{approximation conditions} hold for $ \lambda_{0} = d^{-l_{0}} $ for some $ l_{0} >0$, then they hold for all $\lambda = d^{-l}, 0<l<l_{0}  $. 


Basically, in \eqref{learning curve results}, the term $\sigma^{2} \mathcal{N}_{2}(\lambda) / n $ corresponds to the variance and $\mathcal{M}_{2}(\lambda)$ corresponds to the bias. An obvious relation is $\mathcal{N}_{2}(\lambda) \le \mathcal{N}_{1}(\lambda)$, thus the first approximation condition in \eqref{approximation conditions} guarantees that $\lambda = \lambda(d,n)$ is not that small and the variance term tends to 0. In addition, if we take the traditional source condition assumption, i.e., $\left\| f_{\rho}^{*} \right\|_{[\mathcal{H}]^{s}} \le R$ for some constant $R$ and $s>0$, easy calculation shows that $ \mathcal{M}_{2}(\lambda) \le C \lambda^{\min\{s,2\}}$ for some constant $C$ only depending on $R$ and the constant $ \kappa $ in Assumption \ref{assumption kernel}. This implies that the bias term tends to 0 as $\lambda=\lambda(d,n) \to 0$.


Under the capacity-source condition framework in the fixed-dimensional setting (as discussed in the introduction), Theorem \ref{main theorem} also recovers the state-of-the-art results in \cite{li2023asymptotic}. We emphasize that proving such tight bounds of the generalization error in the large-dimensional setting is nontrivial. In addition, the original bounds of the key quantities in \eqref{n1 n2 m1 m2} under the capacity-source condition framework are no longer sufficient in the large-dimensional setting. Given the information about the kernel and the true function, detailed calculations of $\mathcal{N}_{1}(\lambda), \mathcal{N}_{2}(\lambda), \mathcal{M}_{1}(\lambda) $ and $  \mathcal{M}_{2}(\lambda) $ will be needed (see, e.g., Appendix \ref{section calcu of key} for inner product kernel on the sphere).

\subsection{Applications to inner product kernel on the sphere}\label{section app inner}

In this subsection, we consider the inner product kernel on the sphere with uniform distribution. In the large-dimensional setting $ n \asymp d^{\gamma}, \gamma >0 $ and under further source condition assumption, we apply Theorem \ref{main theorem} to prove the exact convergence rates of the generalization error of the KRR estimator. Then, we derive the corresponding minimax lower bound, which enables us to discuss the minimax optimality and the saturation effect of KRR. 


Suppose that $ \mathcal{X} = \mathbb{S}^{d}$ and $\mu$ is the uniform distribution on $ \mathbb{S}^{d} $. We consider the inner product kernel, i.e., there exists a function $ \Phi(t): [-1,1] \to \mathbb{R}$ such that $ k_{d}(\boldsymbol{x}, \boldsymbol{x}^{\prime}) = \Phi\left( \langle \boldsymbol{x}, \boldsymbol{x}^{\prime} \rangle\right), \forall \boldsymbol{x}, \boldsymbol{x}^{\prime} \in \mathbb{S}^{d}$. Then Mercer's decomposition for the inner product kernel is given in the basis of spherical harmonics:
\begin{displaymath}
  k_{d}(\boldsymbol{x},\boldsymbol{x^{\prime}}) = \sum_{k=0}^{\infty} \mu_k \sum_{l=1}^{N(d,k)} Y_{k,l}(\boldsymbol{x})Y_{k,l}(\boldsymbol{x^{\prime}}),
\end{displaymath}
where $\{Y_{k,l}\}_{l=1}^{N(d,k)}$ are spherical harmonic polynomials of degree $k$; $ \mu_{k}$ are the eigenvalues with multiplicity $N(d,0)=1$; $N(d, k) = \frac{2k+d-1}{k} \cdot \frac{(k+d-2)!}{(d-1)!(k-1)!}, k =1,2,\cdots$.

\begin{assumption}[Inner product kernel]\label{assumption inner product kernel}
    Suppose that $k = \{k_{d}\}_{d=1}^{\infty}$ satisfies 
    \begin{displaymath}
        k_{d}(\boldsymbol{x}, \boldsymbol{x}^{\prime}) = \Phi\left( \langle \boldsymbol{x}, \boldsymbol{x}^{\prime} \rangle\right), ~\forall \boldsymbol{x}, \boldsymbol{x}^{\prime} \in \mathbb{S}^{d},
    \end{displaymath}
    where $ \Phi(t) \in \mathcal{C}^{\infty} \left([-1,1]\right)$ is a fixed function independent of $d$ and
    \begin{displaymath}
        \Phi(t) = \sum_{j=0}^\infty a_j t^j, ~ a_{j} > 0, ~\forall j = 0, 1, 2,\cdots 
    \end{displaymath}
\end{assumption}

Assumption \ref{assumption inner product kernel} formally defines the kernel considered in this subsection. The purpose of assuming all the coefficients $a_{j}$ to be positive is to keep the main results and proofs clean. In fact, the proof is similar for other inner product kernels as long as we know which coefficients are positive, for instance, the neural tangent kernel in the following subsection. We assume $\Phi(t)$ (or $ \{a_{j}\}_{j=0}^{\infty}$) to be fixed and we will ignore the dependence of constants on it in the rest of our paper.

The inner product kernel has attracted a lot of research \citep[etc.]{liang2020multiple,Ghorbani2019LinearizedTN,misiakiewicz_spectrum_2022,xiao2022precise,lu2023optimal} and we have a concise characterization of $ \mu_{k}$ and $ N(d,k)$, which enables us to calculate the exact convergence rates of the key quantities in \eqref{n1 n2 m1 m2}. We refer to  Lemma \ref{lemma inner eigen}, \ref{lemma:monotone_of_eigenvalues_of_inner_product_kernels} and \ref{lemma Ndk} in Appendix \ref{section prelimi about inner} for details about $ \mu_{k}$ and $ N(d,k)$. The extension to general kernel can be extremely complicated and existing results also only consider the case where $\mathcal{X}$ is the sphere (as this paper) or discrete hypercube (see, e.g., \citealt{mei2022generalization,aerni2022strong}).  \\

 

In the next assumption, we formally introduce the source condition, which characterizes the relative smoothness of $f_{\rho}^{*}$ with respect to $\mathcal{H}$.
\begin{assumption}[Source condition]\label{assumption source condition}
$ $\\
\begin{itemize}
    \item[(a)] Suppose that $f_{\rho}^{*}(\boldsymbol{x}) = f_{\rho,d}^{*}(\boldsymbol{x}) = \sum\limits_{i=1}^{\infty} f_{i} e_{i}(\boldsymbol{x}) \in [\mathcal{H}]^{s}$ for some $s > 0$ and satisfies that,
    \begin{equation}\label{assumption source part 1}
        \left\| f_{\rho}^{*} \right\|_{[\mathcal{H}]^{s}} \le R_{\gamma}, 
    \end{equation}
    where $R_{\gamma}$ is a constant only depending on $\gamma$.
    \item[(b)] Denote $ q $ as the smallest integer such that $ q > \gamma$ and $\mu_{q} \neq 0$. Define $\mathcal{I}_{d,k}$ as the index set satisfying $\lambda_{i} \equiv \mu_{k}, i \in \mathcal{I}_{d,k}$. Further suppose that there exists an absolute constant $c_{0} > 0$ such that for any $ d $ and $ k \in \{0,1,\cdots,q\}$ with $\mu_{k} \neq 0$, we have
   \begin{equation}\label{ass of fi}
       \sum\limits_{i \in \mathcal{I}_{d,k}} \mu_{k}^{-s} f_{i}^{2} \ge c_{0}.
   \end{equation}
\end{itemize}
\end{assumption}


Assumption \ref{assumption source condition} (a) is usually used as the traditional source condition \citep[etc.]{caponnetto2006optimal,fischer2020_SobolevNorm}. In order to obtain a reasonable lower bound, we need Assumption \ref{assumption source condition} (b). It is equivalent to assume that the $ [\mathcal{H}]^{s}$ norm of the projection of $ f_{\rho}^{*}$ on the first $q$-th eigenspace is non-vanishing. Similar assumptions have been adopted when one interested in the lower bound of generalization error in the fixed-dimensional setting, e.g., Eq.(8) in \cite{Cui2021GeneralizationER} and Assumption 3 in \cite{li2023asymptotic}. In a word, recalling definition \ref{def interpolation space}, Assumption \ref{assumption source condition} implies that $ f_{\rho}^{*} \in [\mathcal{H}]^{s}$ and $ f_{\rho}^{*} \notin [\mathcal{H}]^{t}$ for any $t > s$. \\

Now we are ready to state two theorems about the exact convergence rates of the generalization error of KRR, which deal with two different ranges of source condition: $s \ge 1$ and $ 0<s<1$. 
\begin{theorem}[Exact convergence rates when $\mathbf{s \ge 1}$]\label{theorem inner s ge 1}
    Let $c_{1} d^{\gamma} \le n \le c_{2} d^{\gamma} $ for some fixed $ \gamma >0$ and absolute constants $ c_{1}, c_{2}$. Consider $\mathcal{X} = \mathbb{S}^{d}$ and the marginal distribution $\mu$ to be the uniform distribution. Let $k = \{ k_{d}\}_{d=1}^{\infty}$ be a sequence of inner product kernels on the sphere satisfying Assumption \ref{assumption kernel} and \ref{assumption inner product kernel}. Further suppose that Assumption \ref{assumption noise} holds and Assumption \ref{assumption source condition} holds for some $ s \ge 1$. Let $\hat{f}_{\lambda}$ be the KRR estimator defined by \eqref{krr estimator}. Define $ \tilde{s} = \min\{s,2\} $, then we have:
    \begin{itemize}
        \item[(i)] When $\gamma \in \left(p+p\tilde{s},~ p+p\tilde{s}+1 \right]$ for some $p \in \mathbb{N}$, by choosing $ \lambda = d^{-\frac{\gamma+p-p\tilde{s}}{2}} \cdot \mathbf{1}_{p>0} + d^{-\frac{\gamma}{2}} \ln{d} \cdot \mathbf{1}_{p=0} $, we have
        \begin{align}
             \mathbb{E}\left[\left\|\hat{f}_\lambda-f_\rho^*\right\|_{L^2}^2 \;\Big|\; \boldsymbol{X} \right]=\begin{cases}
            \Theta_{\mathbb{P}}\left( d^{-\gamma} \ln^{2}{d} \right) = \Theta_{\mathbb{P}}\left( n^{-1} \ln^{2}{n} \right),& p=0, \\[3pt]
            \Theta_{\mathbb{P}}\left( d^{-\gamma + p} \right) = \Theta_{\mathbb{P}}\left( n^{-1 + \frac{p}{\gamma}} \right), & p>0;
            \end{cases}
        \end{align}

        \item[(ii)] When $\gamma \in \left(p+p\tilde{s}+1,~ p+p\tilde{s}+2\tilde{s}-1 \right]$ for some $p \in \mathbb{N}$, by choosing $ \lambda = d^{-\frac{\gamma+3p-p\tilde{s}+1}{4}}$, we have
        \begin{equation}
            \mathbb{E}\left[\left\|\hat{f}_\lambda-f_\rho^*\right\|_{L^2}^2 \;\Big|\; \boldsymbol{X} \right] = \Theta_{\mathbb{P}}\left( d^{-\frac{\gamma-p+p\tilde{s}+1}{2}} \right) = \Theta_{\mathbb{P}}\left( n^{- \frac{\gamma-p+p\tilde{s}+1}{2 \gamma}} \right);
        \end{equation}

        \item[(iii)] When $\gamma \in \left(p+p\tilde{s}+2\tilde{s}-1,~ (p+1)+(p+1)\tilde{s} \right]$ for some $p \in \mathbb{N}$, by choosing $ \lambda = d^{-\frac{\gamma+(p+1)(1-\tilde{s})}{2}}$, we have
        \begin{equation}
             \mathbb{E}\left[\left\|\hat{f}_\lambda-f_\rho^*\right\|_{L^2}^2 \;\Big|\; \boldsymbol{X} \right] = \Theta_{\mathbb{P}}\left( d^{-(p+1)\tilde{s}} \right) = \Theta_{\mathbb{P}}\left( n^{- \frac{(p+1)\tilde{s}}{\gamma}} \right).
        \end{equation}
    \end{itemize}
    The notation $ \Theta_{\mathbb{P}} $ involves constants only depending on $ s, \sigma, \gamma, c_{0}, \kappa, c_{1}$ and $ c_{2} $. In addition, the convergence rates of the generalization error of KRR can not be faster than above for any choice of regularization parameter $ \lambda = \lambda(d,n) \to 0$.
    
\end{theorem}

\begin{theorem}[Exact convergence rates when $\mathbf{0 < s < 1}$]\label{theorem inner s le 1}
    Let $c_{1} d^{\gamma} \le n \le c_{2} d^{\gamma} $ for some fixed $ \gamma >0$ and absolute constants $ c_{1}, c_{2}$. Consider $\mathcal{X} = \mathbb{S}^{d}$ and the marginal distribution $\mu$ to be the uniform distribution. Let $k = \{ k_{d}\}_{d=1}^{\infty}$ be a sequence of inner product kernels on the sphere satisfying Assumption \ref{assumption kernel} and \ref{assumption inner product kernel}. Further suppose that Assumption \ref{assumption noise} holds and Assumption \ref{assumption source condition} holds for some $ 0 < s < 1$. Let $\hat{f}_{\lambda}$ be the KRR estimator defined by \eqref{krr estimator}. Then we have:
    \begin{itemize}[leftmargin = 18pt]
        \item If $ \frac{1}{2} < s < 1$:
        \begin{itemize}[leftmargin = 15pt]
        \item[(i)] When $\gamma \in \left(p+ps,~ p+ps+s \right]$ for some $p \in \mathbb{N}$, by choosing $ \lambda = d^{-\frac{\gamma+p-ps}{2}} \cdot \mathbf{1}_{p>0} + d^{-\frac{\gamma}{2}} \ln{d} \cdot \mathbf{1}_{p=0} $, we have
        \begin{align}
             \mathbb{E}\left[\left\|\hat{f}_\lambda-f_\rho^*\right\|_{L^2}^2 \;\Big|\; \boldsymbol{X} \right]=\begin{cases}
            \Theta_{\mathbb{P}}\left( d^{-\gamma} \ln^{2}{d} \right) = \Theta_{\mathbb{P}}\left( n^{-1} \ln^{2}{n} \right),& p=0, \\[3pt]
            \Theta_{\mathbb{P}}\left( d^{-\gamma + p} \right) = \Theta_{\mathbb{P}}\left( n^{-1 + \frac{p}{\gamma}} \right), & p>0;
            \end{cases}
        \end{align}

        \item[(ii)] When $\gamma \in \left(p+ps+s,~ (p+1)+(p+1)s \right]$ for some $p \in \mathbb{N}$, by choosing $ \lambda = d^{-\frac{2p+s}{2}}$, we have
        \begin{equation}
            \mathbb{E}\left[\left\|\hat{f}_\lambda-f_\rho^*\right\|_{L^2}^2 \;\Big|\; \boldsymbol{X} \right] = \Theta_{\mathbb{P}}\left( d^{-(p+1)s} \right) = \Theta_{\mathbb{P}}\left( n^{- \frac{(p+1)s}{\gamma}} \right);
        \end{equation}

        \end{itemize}
    The notation $ \Theta_{\mathbb{P}} $ involves constants only depending on $ s, \sigma, \gamma, c_{0}, \kappa, c_{1}$ and $ c_{2} $.

    \item If $ 0 < s \le \frac{1}{2}$:  we have the same convergence rates as the case $ s \in (\frac{1}{2},1)$ for those 
    \begin{displaymath}
         \gamma > \frac{3s}{2(s+1)}.
    \end{displaymath}
        
    \end{itemize}
\end{theorem}

\begin{remark}\label{remark s le 1}
    For technical reasons, when $0 < s \le 1/2$, we only prove the convergence rates for those $\gamma > 3s / 2(s+1)$. Note that we have $ 3s / 2(s+1) < 1/2 $ when $0 < s \le 1/2$; and $ 3s / 2(s+1) \to 0$ when $ s \to 0$. Therefore, we have actually proved for almost all $ \gamma > 0$.
\end{remark}

Note that Theorem \ref{theorem inner s ge 1} and Theorem \ref{theorem inner s le 1} show exact convergence rates (both upper and lower bounds) of KRR's generalization error, which is a much stronger result than only proving an upper bound. As we will see in Appendix \ref{section proof of inner le 1}, since $\|f_{\rho}^{*}\|_{L^{\infty}}$ could be infinite when $s < 1$ thus $\mathcal{M}_{1}(\lambda)$ could be infinite, the proof of Theorem \ref{theorem inner s le 1} requires a little more technique. In addition, we will prove in Theorem \ref{theorem lower bound} that the rates in Theorem \ref{theorem inner s le 1} ($ s \le 1$) achieve the minimax lower bound. Together with the statement at the end of Theorem \ref{theorem inner s ge 1}, we actually prove that the rates in Theorem \ref{theorem inner s ge 1} and Theorem \ref{theorem inner s le 1} are the fastest convergence rates that KRR can achieve.\\

Next, we will state the minimax lower bound in the same large-dimensional and source condition setting as Theorem \ref{theorem inner s ge 1} and Theorem \ref{theorem inner s le 1}. 

\begin{theorem}[Minimax lower bound]\label{theorem lower bound}
    Let $c_{1} d^{\gamma} \le n \le c_{2} d^{\gamma} $ for some fixed $ \gamma >0$ and absolute constants $ c_{1}, c_{2}$. Consider $\mathcal{X} = \mathbb{S}^{d}$ and the marginal distribution $\mu$ to be the uniform distribution. Let $k = \{ k_{d}\}_{d=1}^{\infty}$ be a sequence of inner product kernels on the sphere satisfying Assumption \ref{assumption kernel} and \ref{assumption inner product kernel}. Let $\mathcal{P}$ consist of all the distributions $\rho$ on $\mathcal{X} \times \mathcal{Y}$ such that Assumption \ref{assumption noise} holds and Assumption \ref{assumption source condition} holds for some $ s > 0$. Then we have:

        \begin{itemize}[leftmargin = 15pt]
            \item[(i)] When $\gamma \in \left(p+ps, p+ps+s \right]$ for some $p \in \mathbb{N}$, for any $\epsilon > 0$, there exist constants $\mathfrak{C}_1$ and $\mathfrak{C}$ only depending on $s, \epsilon, \gamma, \sigma, \kappa, c_{1}$ and $ c_{2} $ such that for any $d \geq \mathfrak{C}$, we have:
            \begin{equation}\label{minimax lower eq 1}
            \min _{\hat{f}} \max _{\rho \in \mathcal{P}} \mathbb{E}_{(\boldsymbol{X}, \mathbf{y}) \sim \rho^{\otimes n}}
            \left\|\hat{f} - f_{\rho}^{*}\right\|_{L^2}^2
            \geq \mathfrak{C}_1 d^{-\gamma + p - \epsilon};
            \end{equation}

            \item[(ii)] When $\gamma \in \left(p+ps+s, (p+1)+(p+1)s \right]$ for some $p \in \mathbb{N}$, there exist constants $\mathfrak{C}_1$ and $\mathfrak{C}$ only depending on $ s, \gamma, \sigma, \kappa, c_{1}$ and $ c_{2} $ such that for any $d \geq \mathfrak{C}$, we have:
            \begin{equation}\label{minimax lower eq 2}
            \min _{\hat{f}} \max _{\rho \in \mathcal{P}} \mathbb{E}_{(\boldsymbol{X}, \mathbf{y}) \sim \rho^{\otimes n}}
            \left\|\hat{f} - f_{\rho}^{*}\right\|_{L^2}^2
            \geq \mathfrak{C}_1 d^{-(p+1)s};
            \end{equation}

        \end{itemize}

\end{theorem}

Theorem \ref{theorem lower bound} states that there is no estimator (or learning method) that can achieve faster convergence rates than \eqref{minimax lower eq 1} and \eqref{minimax lower eq 2}.\\


\begin{figure}[ht]
\vskip 0.05in
\centering
\subfigure[]{\includegraphics[width=0.32\columnwidth]{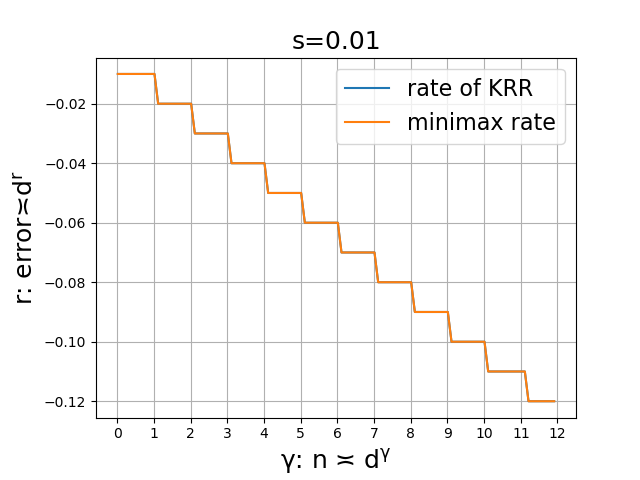}}
\subfigure[]{\includegraphics[width=0.32\columnwidth]{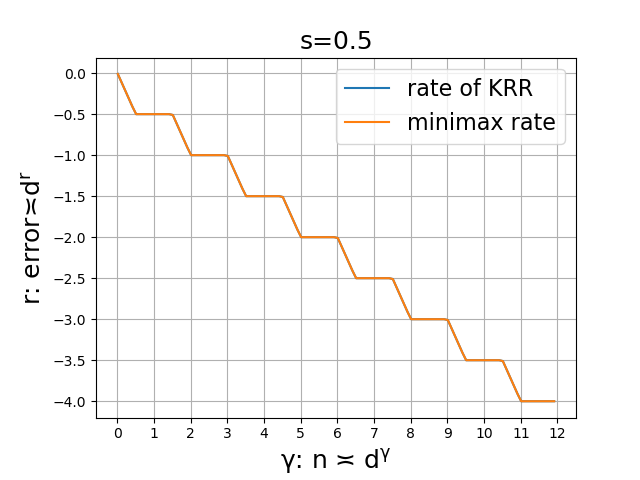}}
\subfigure[]{\includegraphics[width=0.32\columnwidth]{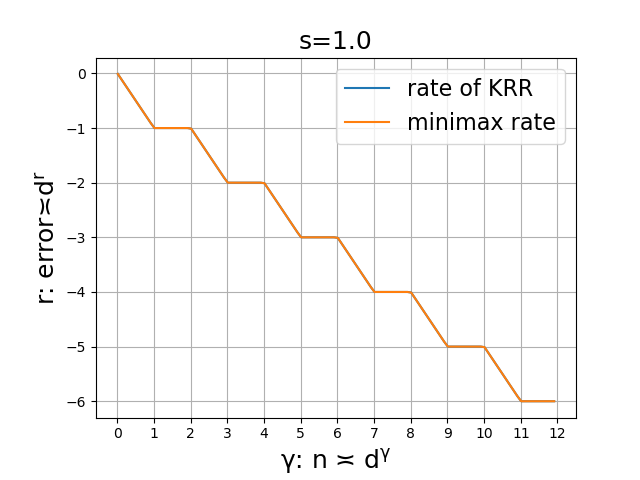}}

\subfigure[]{\includegraphics[width=0.32\columnwidth]{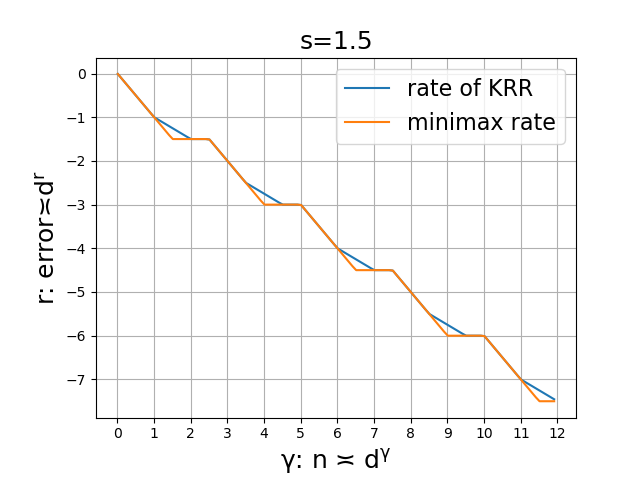}}
\subfigure[]{\includegraphics[width=0.32\columnwidth]{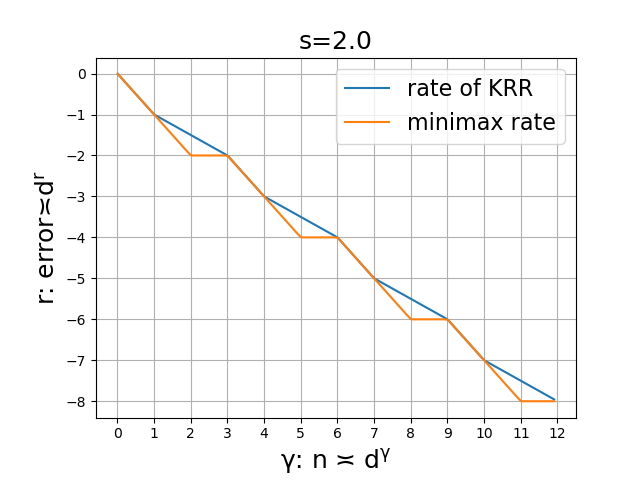}}
\subfigure[]{\includegraphics[width=0.32\columnwidth]{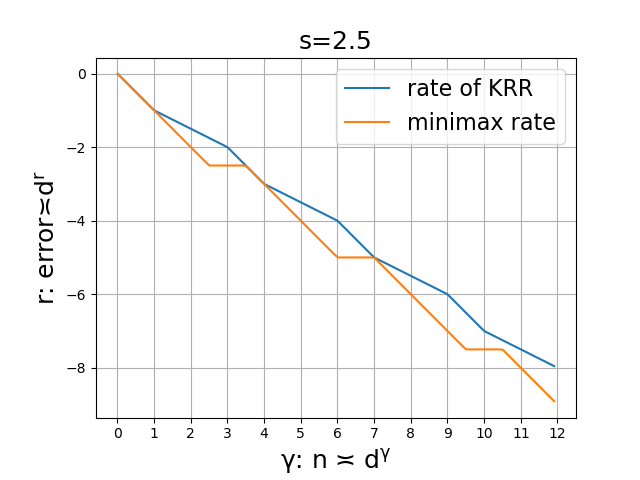}}

\caption{Convergence rates of KRR in Theorem \ref{theorem inner s ge 1}, Theorem \ref{theorem inner s le 1} and corresponding minimax lower rates in Theorem \ref{theorem lower bound} (ignoring a $\epsilon$-difference) \textit{with respect to dimension $d$}. We present 6 graphs corresponding to 6 kinds of source conditions: $ s = 0.01, 0.5, 1.0, 1.5, 2.0, 2.5$. The x-axis represents asymptotic scaling, $\gamma: n \asymp d^{\gamma}$; the y-axis represents the convergence rate of generalization error, $ r: \text{error} \asymp d^{r}$.              }
\label{figure 6 rates w.r.t d}
\vskip 0.05in
\end{figure}

Summarizing the results in Theorem \ref{theorem inner s ge 1}, Theorem \ref{theorem inner s le 1} and Theorem \ref{theorem lower bound}, Figure \ref{figure 6 rates w.r.t d} shows the convergence rates of KRR and corresponding minimax lower rates \textit{with respect to dimension $d$} for any $\gamma >0$. We can see that the rates decrease when the scaling $\gamma$ increases, indicating that the performance becomes better when the sample size $n$ grows. Moreover, we can observe several intriguing phenomena.\\

\textit{Curve's evolution with source condition.} Since we consider source condition $s>0$, we can compare the rate curves in Figure \ref{figure 6 rates w.r.t d} for different $s$ and see how they evolve with $s$. 

Let us first see the minimax lower rates. For any $s>0$, there are 2 periods with respect to the value of $\gamma$: The length of the first period, i.e., $ (p+ps, p+ps+s], p \in \mathbb{N}$, will decrease to 0 as $s$ getting close to 0; The length of the second period, i.e., $ (p+ps+s, (p+1)+(p+1)s]$, equals $1$ for all $s>0$.

Next, we see the convergence rates of KRR, which is more intriguing.
\begin{itemize}
    \item When $ 0 < s \le 1$, there are 2 periods with respect to the value of $\gamma$ and the curve is the same as the minimax lower rates. (In fact, Theorem \ref{theorem inner s le 1} only proves the results for $\gamma > 3s /2(s+1)$ when $s \le 1/2$, we write $\gamma>0$ with a little bit of notation abusement.)

    \item  When $ 1 < s < 2$, there are 3 periods with respect to the value of $\gamma$: The length of the first period, i.e., $ (p+ps, p+ps+1] $, equals 1 for all $ 1 < s < 2$; The length of the second period, i.e., $ (p+ps+1, p+ps+2s-1] $ is $ 2s-2$, thus this period will degenerate as $s$ getting close to 1; The length of the third period, i.e., $ (p+ps+2s-1, (p+1)+(p+1)s] $ is $ 2-s$, thus this period will degenerate as $s$ getting close to 2. 

    \item When $ s \ge 2$, the curve does not change with $s$ and there are 2 periods with respect to the value of $\gamma$: The length of the first period, i.e., $ (3p, 3p+1] $, equals 1 for all $ s \ge 2$; The length of the second period, i.e., $ (3p+1, 3p+3] $, equals 2 for all $ s \ge 2$.\\
\end{itemize}

\textit{Minimax optimality and new saturation effect of KRR.} As can be seen in Figure \ref{figure 6 rates w.r.t d} (a)(b)(c), the convergence rates of KRR match the minimax lower bound for all $\gamma > 0$, thus we prove the minimax optimality of KRR when $ 0 < s \le 1$. In contrast, when $s > 1$, Figure \ref{figure 6 rates w.r.t d} (d)(e)(f) illustrate that KRR can not achieve the minimax lower bound in Theorem \ref{theorem lower bound} for certain ranges of $\gamma$, which we refer to as the new saturation effect of KRR. We will discuss the implications of this new saturation effect in the following.

In the fixed-dimensional setting, kernel ridge regression has been studied as a special kind of \textit{spectral algorithm} \citep{gerfo2008_SpectralAlgorithms}. Each spectral algorithm is determined by a \textit{filter function} and the difference between spectral algorithms is characterized by an index called the ``qualification" ($\tau$) of the filter function (see, e.g., \citealt[Definition 1]{zhang2023optimality}). Since KRR has qualification $\tau = 1$ and gradient flow has qualification $\tau = \infty$ \citep[Example 1, 2]{zhang2023optimality}, the main difference between KRR and gradient flow is the \textit{saturation effect}~\citep{li2023_SaturationEffect}. It says when $ s > 2 \tau = 2$, no matter how carefully one tunes the KRR, the convergence rate can not be faster than $ n^{-\frac{2 \beta}{2 \beta + 1}}$, thus can not achieve the minimax lower bound $ n^{-\frac{s \beta}{s \beta + 1}} $.

In the large-dimensional setting and for inner product kernel on the sphere, our results show that the saturation effect of KRR happens in a new regime $ 1 < s \le 2 = 2 \tau$. In addition, we conjecture that there are other spectral algorithms (e.g., gradient flow) that can achieve the minimax lower bound in Theorem \ref{theorem lower bound} for all $s > 0$. This new saturation effect strongly suggests that qualification ($\tau$) itself is insufficient to characterize the filter function (or spectral algorithm) in the large-dimensional setting.\\

\textit{Periodic plateau behavior.} If $ 0 < s < 2$, Figure \ref{figure 6 rates w.r.t d} (a)(b)(c) show that when $\gamma$ varies within certain ranges, the value of vertical axis, r, does not change. We refer to such ranges of $\gamma$ as the plateau period. When $s$ exceeds 2, the plateau period of KRR's convergence rates degenerates and the plateau period of minimax lower rates still exists. Also note that the length of each plateau period varies with the values $s>0$.

For these plateau periods, if we fix a large dimension $d$ and increase $\gamma$ (or equivalently, increase the sample size $n$), the convergence rates of KRR or minimax lower rates stay invariant in certain ranges. Therefore, in order to improve the rate, one has to increase the sample size above a certain threshold. \\


Figure \ref{figure 3 rates w.r.t n} provides an alternative representation of our results, which shows the convergence rates of KRR and corresponding minimax lower rates \textit{with respect to sample size $n$}. We can observe the ``multiple descent behavior" ( for both the convergence rates of KRR and the minimax lower rates) from Figure \ref{figure 3 rates w.r.t n}.\\

\textit{Multiple descent behavior.} Let us first see the minimax lower rates. For any $s>0$, the curve achieves its peaks at $\gamma = p+ps, p \in \mathbb{N}^{+} $, and achieve its isolated valleys at $ \gamma = p+ps+s, p \in \mathbb{N}^{+}$.

For the convergence rates of KRR:
\begin{itemize}
    \item When $ 0 < s \le 1$, the curve is the same as the curve of minimax lower rates.

    \item When $ 1 < s < 2$, the curve achieves its peaks at $\gamma = p+p\tilde{s}, p \in \mathbb{N}^{+} $; achieve its isolated valleys at $ \gamma = p+p\tilde{s}+1, p \in \mathbb{N}^{+} $ and achieve its hillside at $ \gamma = p+p\tilde{s}+2\tilde{s}-1, p \in \mathbb{N}^{+} $.

    \item  When $s \ge 2$, the curve does not change with $s$, which achieves its peaks at $\gamma = 3p, p \in \mathbb{N}^{+} $, and achieve its isolated valleys at $ \gamma = 3p+1, p \in \mathbb{N}^{+}$.
\end{itemize}

\begin{figure}[ht]
\vskip 0.05in
\centering
\subfigure[]{\includegraphics[width=0.32\columnwidth]{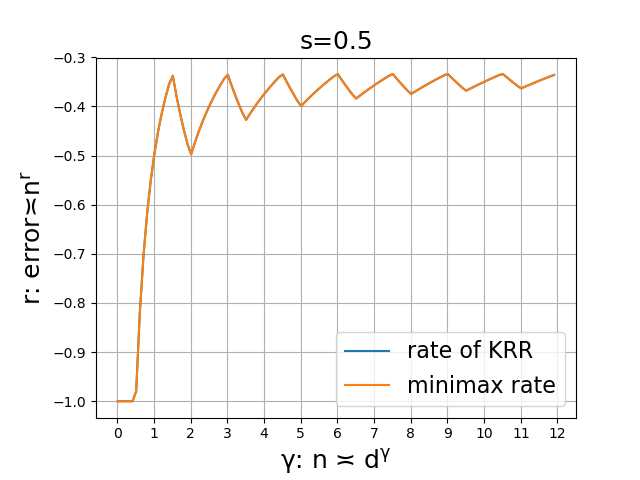}}
\subfigure[]{\includegraphics[width=0.32\columnwidth]{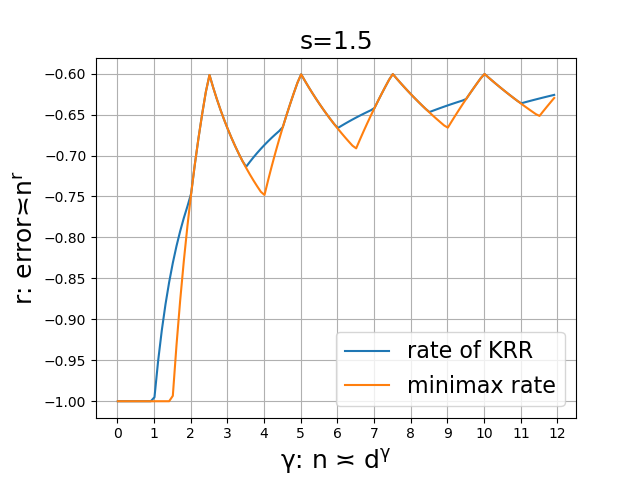}}
\subfigure[]{\includegraphics[width=0.32\columnwidth]{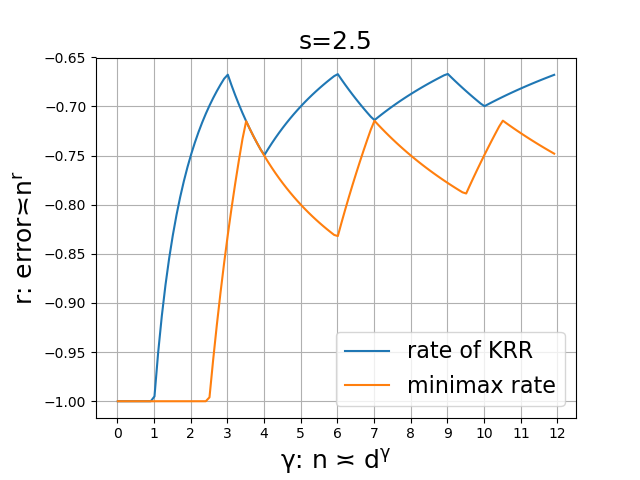}}

\caption{Convergence rates of KRR in Theorem \ref{theorem inner s ge 1}, Theorem \ref{theorem inner s le 1} and corresponding minimax lower rates in Theorem \ref{theorem lower bound} (ignoring a $\epsilon$-difference) \textit{with respect to sample size $n$}. We present 3 graphs corresponding to 3 kinds of source conditions: $ s =0.5, 1.5, 2.5$. The x-axis represents asymptotic scaling, $\gamma: n \asymp d^{\gamma}$; the y-axis represents the convergence rate of generalization error, $ r: \text{error} \asymp n^{r}$.          } 
\label{figure 3 rates w.r.t n}
\vskip 0.05in
\end{figure}

\subsection{Applications to neural tangent kernel}

In this subsection, we look at a specific example, i.e., the neural tangent kernel (NTK) of a two-layer fully connected ReLU neural network $k_{d} = k_{d}^{\mathrm{NT}}$. Still, we suppose that $\mathcal{X} = \mathbb{S}^{d}$ and $\mu$ is uniform distribution on $ \mathbb{S}^{d} $. It has been shown in \cite{Bietti2019OnTI,lu2023optimal} that NTK is an example of inner product kernel satisfying
\begin{displaymath}
        k_{d}^{\mathrm{NT}}(\boldsymbol{x}, \boldsymbol{x}^{\prime}) = \Phi\left( \langle \boldsymbol{x}, \boldsymbol{x}^{\prime} \rangle\right), ~\forall \boldsymbol{x}, \boldsymbol{x}^{\prime} \in \mathbb{S}^{d},
\end{displaymath}
where $ \Phi(t) = \sum_{j=0}^\infty a_j t^j \in \mathcal{C}^{\infty} \left([-1,1]\right)$ is a fixed function independent of $d$ and
\begin{displaymath}
    a_{0} > 0; ~~a_{1} >0; ~~ a_{j} > 0, ~\forall j = 2,4,6\cdots ;~~ a_{j} = 0, ~\forall j = 3,5,7,\cdots
\end{displaymath}
Lemma 5 in \cite{lu2023optimal} also showed that $ \sup\limits_{\boldsymbol{x} \in \mathcal{X}} k_{d}^{\mathrm{NT}}(\boldsymbol{x},\boldsymbol{x}) \le 1$.

For the neural tangent kernel $k_{d}^{\mathrm{NT}}$, the following theorems provide the exact convergence rates of the generalization error of KRR and the corresponding minimax lower bound. Since the proofs are similar to Theorem \ref{theorem inner s ge 1}, Theorem \ref{theorem inner s le 1} and Theorem \ref{theorem lower bound}, we omit the proofs of the following theorems.

\begin{theorem}[NTK: exact convergence rates when $\mathbf{s \ge 1}$]\label{theorem ntk s ge 1}
    Let $c_{1} d^{\gamma} \le n \le c_{2} d^{\gamma} $ for some fixed $ \gamma >0$ and absolute constants $ c_{1}, c_{2}$. Consider $\mathcal{X} = \mathbb{S}^{d}$ and the marginal distribution $\mu$ to be the uniform distribution. Let $k = \{ k_{d}^{\mathrm{NT}}\}_{d=1}^{\infty}$ be a sequence of neural tangent kernels of a two-layer fully connected ReLU neural network on the sphere. Further suppose that Assumption \ref{assumption noise} holds and Assumption \ref{assumption source condition} holds for some $ s \ge 1$. Let $\hat{f}_{\lambda}$ be the KRR estimator defined by \eqref{krr estimator}. For any $ p \in \mathcal{I}_{p}$, we define $p^{\prime} = p+2 $, if $ p \ge 2$; and define $p^{\prime} = p+1 $, if $ p \le 1$. Define $ \tilde{s} = \min\{s,2\} $, then we have:
    \begin{itemize}

    \item[(i)] When $\gamma \in \left(p+p\tilde{s}, p^{\prime}+p\tilde{s} \right]$ for some $p \in \mathcal{I}_{p}$, by choosing $ \lambda = d^{-\frac{\gamma+p-p\tilde{s}}{2}} \cdot \mathbf{1}_{p>0} + d^{-\frac{\gamma}{2}} \ln{d} \cdot \mathbf{1}_{p=0} $, we have
        \begin{align}
             \mathbb{E}\left[\left\|\hat{f}_\lambda-f_\rho^*\right\|_{L^2}^2 \;\Big|\; \boldsymbol{X} \right]=\begin{cases}
            \Theta_{\mathbb{P}}\left( d^{-\gamma} \ln^{2}{d} \right) = \Theta_{\mathbb{P}}\left( n^{-1} \ln^{2}{n} \right),& p=0, \\[3pt]
            \Theta_{\mathbb{P}}\left( d^{-\gamma + p} \right) = \Theta_{\mathbb{P}}\left( n^{-1 + \frac{p}{\gamma}} \right), & p>0;
            \end{cases}
        \end{align}

    \item[(ii)] When $ \gamma \in (p^{\prime}+p\tilde{s}, 2p^{\prime}\tilde{s} - p^{\prime} + 2p -p\tilde{s}]$ for some $ p \in \mathcal{I}_{p}$, by choosing 
    $ \lambda = d^{-\frac{\gamma+p^{\prime}+2p-p\tilde{s}}{4}} $, we have
    \begin{equation}
      \mathbb{E}\left[\left\|\hat{f}_\lambda-f_\rho^*\right\|_{L^2}^2 \;\Big|\; \boldsymbol{X} \right] = \Theta_{\mathbb{P}}\left(d^{-\frac{\gamma}{2}+\frac{p^{\prime}}{2}-\frac{p\tilde{s}}{2}-2}\right) = \Theta_{\mathbb{P}}\left( n^{-\frac{1}{2}+\frac{p^{\prime}}{2\gamma}-\frac{p\tilde{s}}{2\gamma}-\frac{2}{\gamma}} \right);
    \end{equation}

    \item[(iii)] When $\gamma \in (2p^{\prime}\tilde{s} - p^{\prime} + 2p -p\tilde{s}, p^{\prime}+p^{\prime}\tilde{s}]$ for some $ p \in \mathcal{I}_{p}$, by choosing 
    $ \lambda = d^{-\frac{\gamma+p^{\prime}(1-\tilde{s})}{2}} $, we have
    \begin{equation}
      \mathbb{E}\left[\left\|\hat{f}_\lambda-f_\rho^*\right\|_{L^2}^2 \;\Big|\; \boldsymbol{X} \right] = \Theta_{\mathbb{P}}\left(d^{-p^{\prime}\tilde{s}}\right) = \Theta_{\mathbb{P}}\left(n^{-\frac{p^{\prime}\tilde{s}}{\gamma}} \right).
    \end{equation}
    
\end{itemize}
    The notation $ \Theta_{\mathbb{P}} $ involves constants only depending on $ s, \sigma, \gamma, c_{0}, c_{1}$ and $ c_{2} $.
\end{theorem}

\begin{theorem}[NTK: exact convergence rates when $\mathbf{0 < s < 1}$]\label{theorem ntk s le 1}
    Let $c_{1} d^{\gamma} \le n \le c_{2} d^{\gamma} $ for some fixed $ \gamma >0$ and absolute constants $ c_{1}, c_{2}$. Consider $\mathcal{X} = \mathbb{S}^{d}$ and the marginal distribution $\mu$ to be the uniform distribution. Let $k = \{ k_{d}^{\mathrm{NT}}\}_{d=1}^{\infty}$ be a sequence of neural tangent kernels of a two-layer fully connected ReLU neural network on the sphere. Further suppose that Assumption \ref{assumption noise} holds and Assumption \ref{assumption source condition} holds for some $ 0 < s < 1$. Let $\hat{f}_{\lambda}$ be the KRR estimator defined by \eqref{krr estimator}. Denote $ \mathcal{I}_{p} = \{0, 1\} \cup \{2,4,6,\cdots\}$. For any $ p \in \mathcal{I}_{p}$, we define $p^{\prime} = p+2 $, if $ p \ge 2$; and define $p^{\prime} = p+1 $, if $ p \le 1$. Then we have:
    \begin{itemize}[leftmargin = 18pt]
        \item If $ \frac{1}{2} < s < 1$:
        \begin{itemize}[leftmargin = 15pt]
    \item[(i)]  When $ \gamma \in (p+ps, p+p^{\prime}s]$ for some $ p \in \mathcal{I}_{p}$, by choosing $ \lambda = d^{-\frac{\gamma+p-ps}{2}} \cdot \mathbf{1}_{p>0} + d^{-\frac{\gamma}{2}} \ln{d} \cdot \mathbf{1}_{p=0} $, we have
    \begin{align}
             \mathbb{E}\left[\left\|\hat{f}_\lambda-f_\rho^*\right\|_{L^2}^2 \;\Big|\; \boldsymbol{X} \right]=\begin{cases}
            \Theta_{\mathbb{P}}\left( d^{-\gamma} \ln^{2}{d} \right) = \Theta_{\mathbb{P}}\left( n^{-1} \ln^{2}{n} \right),& p=0, \\[3pt]
            \Theta_{\mathbb{P}}\left( d^{-\gamma + p} \right) = \Theta_{\mathbb{P}}\left( n^{-1 + \frac{p}{\gamma}} \right), & p>0;
            \end{cases}
        \end{align}

    \item[(ii)]  When $ \gamma \in (p+p^{\prime}s, p^{\prime}+p^{\prime}s]$ for some $ p \in \mathcal{I}_{p}$, by choosing $ \lambda = d ^{-p-\frac{(p^{\prime}-p)s}{2}}$, we have
    \begin{equation}
 \mathbb{E}\left[\left\|\hat{f}_\lambda-f_\rho^*\right\|_{L^2}^2 \;\Big|\; \boldsymbol{X} \right] = \Theta_{\mathbb{P}}\left( d^{-p^{\prime}s} \right) = \Theta_{\mathbb{P}}\left( n^{-\frac{p^{\prime}s}{\gamma}} \right);
    \end{equation}
    \end{itemize}
    
    The notation $ \Theta_{\mathbb{P}} $ involves constants only depending on $s, \sigma, \gamma, c_{0}, c_{1}$ and $ c_{2} $.
    \item If $ 0< s \le \frac{1}{2}$: we have the same convergence rates as the case $ s \in (\frac{1}{2},1)$ for those 
    \begin{displaymath}
         \gamma > \frac{3s}{2(s+1)}.
    \end{displaymath}
    \end{itemize}
\end{theorem}

\begin{theorem}[NTK: minimax lower bound]\label{theorem lower bound ntk}
    Let $c_{1} d^{\gamma} \le n \le c_{2} d^{\gamma} $ for some fixed $ \gamma >0$ and absolute constants $ c_{1}, c_{2}$. Consider $\mathcal{X} = \mathbb{S}^{d}$ and the marginal distribution $\mu$ to be the uniform distribution. Let $k = \{ k_{d}^{\mathrm{NT}}\}_{d=1}^{\infty}$ be a sequence of neural tangent kernels of a two-layer fully connected ReLU neural network on the sphere. Let $\mathcal{P}$ consist of all the distributions $\rho$ on $\mathcal{X} \times \mathcal{Y}$ such that Assumption \ref{assumption noise} holds and Assumption \ref{assumption source condition} holds for some $ s > 0$. Then we have:

        \begin{itemize}[leftmargin = 15pt]
            \item[(i)] When $ \gamma \in (p+ps, p+p^{\prime}s]$ for some $ p \in \mathcal{I}_{p}$, for any $\epsilon > 0$, there exist constants $\mathfrak{C}_1$ and $\mathfrak{C}$ only depending on $s, \epsilon, \gamma, \sigma, c_{1}$ and $ c_{2} $ such that for any $d \geq \mathfrak{C}$, we have:
            \begin{equation}\label{minimax lower eq 1 ntk}
            \min _{\hat{f}} \max _{\rho \in \mathcal{P}} \mathbb{E}_{(\boldsymbol{X}, \mathbf{y}) \sim \rho^{\otimes n}}
            \left\|\hat{f} - f_{\rho}^{*}\right\|_{L^2}^2
            \geq \mathfrak{C}_1 d^{-\gamma + p - \epsilon};
            \end{equation}

            \item[(ii)] When $ \gamma \in (p+p^{\prime}s, p^{\prime}+p^{\prime}s]$ for some $ p \in \mathcal{I}_{p}$, there exist constants $\mathfrak{C}_1$ and $\mathfrak{C}$ only depending on $ s, \gamma, \sigma, c_{1}$ and $ c_{2} $ such that for any $d \geq \mathfrak{C}$, we have:
            \begin{equation}\label{minimax lower eq 2 ntk}
            \min _{\hat{f}} \max _{\rho \in \mathcal{P}} \mathbb{E}_{(\boldsymbol{X}, \mathbf{y}) \sim \rho^{\otimes n}}
            \left\|\hat{f} - f_{\rho}^{*}\right\|_{L^2}^2
            \geq \mathfrak{C}_1 d^{-p^{\prime}s};
            \end{equation}

        \end{itemize}

\end{theorem}

Note that the results in this subsection will be the same as the theorems in Section \ref{section app inner} if we change the above definition of $ p^{\prime}$ to $ p^{\prime} = p+1, \forall p = 0,1,2,\cdots$.

\section{Conclusion and discussion}\label{section discussion}

In this paper, we first establish a new framework for studying the asymptotic generalization error of kernel ridge regression (Theorem \ref{main theorem}). This framework makes few assumptions on the RKHS, the true function and the relation between $d$ and $n$, thus it is suitable for studying various topics about KRR's generalization error in both fixed-dimensional and large-dimensional settings. Moreover, the results in Theorem \ref{main theorem} provides the matching upper and lower bounds of the generalization error with certain regularization parameter, which is more instructive than just the upper bound. 

Based on this framework, we then consider inner product kernel on the sphere and the large-dimensional setting ($n \asymp d^{\gamma}, \gamma >0$). Given that $f_{\rho}^{*}$ falls into $[\mathcal{H}]^{s}$, an interpolation space of RKHS, Theorem \ref{theorem inner s ge 1} and Theorem \ref{theorem inner s le 1} prove the exact convergence rates of KRR's generalization error under the best choice of regularization parameter and Theorem \ref{theorem lower bound} proves the corresponding minimax lower bound. These results show the minimax optimality of KRR when $0 < s \le 1$ and the new saturation effect of KRR $s>1$. We also discuss how the rate curves (varying along $\gamma$) evolve with the value of $s$ and discuss the ``periodic plateau behavior" and ``multiple descent behavior" of KRR in the large-dimensional setting.

Similar periodic behavior has been observed for kernel methods in the large-dimensional setting in related literature, we next make some discussion on these works. There is a line of work studying the inconsistency of kernel methods with inner product kernels in the large-dimensional setting $ n \asymp d^{\gamma}, \gamma > 0$ \citep[etc.]{Ghorbani2019LinearizedTN,mei2022generalization,misiakiewicz_spectrum_2022}. Assuming the true function $ f_{\rho}^{*}$ to be square-integrable (or equivalently $s=0$ in our setting) on the sphere, \citet[Theorem 4]{Ghorbani2019LinearizedTN} proves that the generalization error of KRR $R_{\mathrm{KR}}\left(f_{\rho}^{*}, \boldsymbol{X}, \lambda\right)$ satisfies (with high probability)
\begin{equation}\label{result of linearized}
    \left|R_{\mathrm{KR}}\left(f_{\rho}^{*}, \boldsymbol{X}, \lambda\right)-\left\|\mathrm{P}_{>\ell} f_{\rho}^{*}\right\|_{L^2}^2\right| \leq \epsilon\left(\left\|f_{\rho}^{*}\right\|_{L^2}^2+\sigma^2\right),  \forall~ 0 < \lambda < \lambda^{*},
\end{equation}
where $\ell = \lfloor \gamma \rfloor$ is the greatest integer that is less or equal to $\gamma$, $\mathrm{P}_{>\ell}$ means the projection onto polynomials with degree $>\ell$, $\epsilon$ is any positive real number and $\lambda^{*}$ is defined as \citealt[Eq.(20)]{Ghorbani2019LinearizedTN}. \eqref{result of linearized} implies that generalization error will drop when $ \gamma$ exceeds an integer and stay invariant for other $\gamma$ (see the cartoon representation in \citealt[Figure 5]{Ghorbani2019LinearizedTN}). In our paper, when $s>0$ and efficiently close to 0, similar behavior has been observed in Figure \ref{figure 6 rates w.r.t d} (a) that the convergence rate drops abruptly around each integer $ \gamma \in \mathbb{N}$. 

A more recent work \cite{lu2023optimal} considers the optimality of early stopping kernel gradient flow in the same large-dimensional setting $ c_{1} d^{\gamma} \le n \le  c_{2} d^{\gamma}, \gamma >0$. They also consider inner product kernel on the sphere and assume that the true function falls into the RKHS $ f_{\rho}^{*} \in \mathcal{H} $ (or equivalently, $s=1$). Denoting $p = \lfloor \gamma / 2 \rfloor$, \citealt[Theorem 4.3]{lu2023optimal} proves that by properly choosing the early stopping time $\widehat{T}$, the upper bound of the convergence rate is:

\begin{itemize}

\item When $\gamma \in \{2, 4, 6, \cdots\}$, then, there exist constants $\mathfrak{C}$ and $\mathfrak{C}_{i}$, where $i=1,2,3$, only depending on $\gamma$, $c_1$, and $c_2$, such that for any $d \geq \mathfrak{C}$,  we have
\begin{equation}
\begin{aligned}
\left\|{f}_{\widehat{T}} - f_{\star}\right\|_{L^2}^2
\leq \mathfrak{C}_1 n^{-\frac{1}{2}}
\end{aligned}
\end{equation}
holds with probability at least $1 -\mathfrak{C}_2\exp\{-\mathfrak{C}_3 n^{1/2} \}$.

    \item When $\gamma \in \bigcup_{j=0}^{\infty} (2j, 2j+1]$, for any $\delta>0$, there exist constants $\mathfrak{C}$ and $\mathfrak{C}_{i}$, where $i=1,2,3$, only depending on $\gamma$, $\delta$, $c_1$, and $c_2$, such that for any $d \geq \mathfrak{C}$,  we have 
\begin{equation}
\begin{aligned}
\left\|{f}_{\widehat{T}} - f_{\rho}^{*}\right\|_{L^2}^2
\leq \mathfrak{C}_1 n^{-\frac{\gamma - p}{\gamma}} \log(n) 
\end{aligned}
\end{equation}
holds with probability at least $1-\delta-\mathfrak{C}_2\exp\{-\mathfrak{C}_3 n^{p/\gamma} \log(n)\}$.

\item  When $\gamma \in \bigcup_{j=0}^{\infty} (2j+1, 2j+2)$, then, for any $\delta>0$, there exist constants $\mathfrak{C}$ and $\mathfrak{C}_{i}$, where $i=1,2,3$, only depending on $\gamma$, $\delta$, $c_1$, and $c_2$, such that for any $d \geq \mathfrak{C}$,  we have  
\begin{equation}
\begin{aligned}
\left\|{f}_{\widehat{T}} - f_{\star}\right\|_{L^2}^2
\leq \mathfrak{C}_1 n^{-\frac{p+1}{\gamma}}
\end{aligned}
\end{equation}
holds with probability at least $1-\delta-\mathfrak{C}_2\exp\{-\mathfrak{C}_3 n^{1 - (p+1)/\gamma}\}$.

\end{itemize}
Ignoring the log term, simple calculation shows that the convergence rate is consistent with the rate in Theorem \ref{theorem inner s ge 1} when $s=1$ (see, e.g., Figure \ref{figure 6 rates w.r.t d} (c)). \cite{lu2023optimal} also proves that the above upper bound matches the minimax lower bound, thus proving the minimax optimality of early stopping kernel gradient flow under the assumption $f_{\rho}^{*} \in \mathcal{H}$. In contrast to \cite{lu2023optimal}, we provide not only the upper bound but also the lower bound of the convergence rates of KRR under general source condition $s > 0$. We have seen that the ``periodic plateau behavior" and ``multiple descent behavior" observed in \cite{lu2023optimal} still exist for $s>0$ and the plateau length will change with the value of $s$. By considering source condition $s>0$, our restriction on the true function is much milder, thus we provide a more complete characterization of the generalization error of KRR.

The periodic behavior in large dimension has also been observed for ``kernel interpolation estimator", for instance, \cite{liang2020multiple} for the inner product kernel and \cite{aerni2022strong} for the convolutional kernel. Although technically complicated, a direct follow-up question is the convergence rate of generalization error for general kernels and domains. We believe that it is an interesting research direction to study the generalization behavior of kernel methods in the large-dimensional setting, which will exhibit a wealth of new phenomena compared with the fixed-dimensional setting. 







\acks{Qian Lin is supported in part by National Natural Science Foundation of China (Grant 92370122, Grant 11971257) and the Beijing Natural Science Foundation (Grant Z190001).}


\bigskip
\bigskip
\bigskip
\bigskip

\appendix


\section*{Appendix} 
In the Appendix, we provide the proof of Theorem \ref{main theorem} (Appendix \ref{section proofs main}), Theorem \ref{theorem inner s ge 1} \& \ref{theorem inner s le 1} (Appendix \ref{section proofs inner}) and Theorem \ref{theorem lower bound} (Appendix \ref{section proofs minimax}). Necessary auxiliary results can be found in Appendix \ref{section auxiliart}.

\section{Proof of Theorem \ref{main theorem}}\label{section proofs main}
The proof of Theorem \ref{main theorem} consists of the following steps: First, we introduce the bias-variance decomposition in Section \ref{section bias var decom}. Next, we derive the bounds of variance term in Section \ref{section variance term} and bias term in Section \ref{section bias term}. Finally, using the results in these sections, we formally prove Theorem \ref{main theorem} in Section \ref{section final proof of main theorem}.

\subsection{Bias-variance decomposition}\label{section bias var decom}

The proof of Theorem \ref{main theorem} is based on the traditional bias-variance decomposition. The contribution in this paper is that we refine the tools in \cite{Li2023KernelIG} and \cite{li2023asymptotic} to handle the large-dimensional case. Throughout the proof, we denote 
\begin{equation}\label{brief notation}
   T_{\lambda} = T + \lambda;~~ T_{\boldsymbol{X} \lambda} = T_{\boldsymbol{X}} + \lambda,
\end{equation}
where $\lambda$ is the regularization parameter. We use $\| \cdot \|_{\mathscr{B}(B_1,B_2)}$ to denote the operator norm of a bounded linear operator from a Banach space $B_1$ to $B_2$, i.e., $ \| A \|_{\mathscr{B}(B_1,B_2)} = \sup\limits_{\|f\|_{B_1}=1} \|A f\|_{B_2} $. Without bringing ambiguity, we will briefly denote the operator norm as $\| \cdot \|$. In addition, we use $\text{tr}A$ and $\| A \|_{1}$ to denote the trace and the trace norm of an operator. We use $\| A \|_{2}$ to denote the Hilbert-Schmidt norm. In addition, we denote $ L^{2}(\mathcal{X},\mu)$ as $ L^{2}$, $ L^{\infty}(\mathcal{X},\mu)$ as $ L^{\infty}$ for brevity throughout the proof.



We also need the following essential notations in our proof, which are frequently used in related literature. Denote the samples $\boldsymbol{Z} = \{ (\boldsymbol{x}_{i}, y_{i}) \}_{i=1}^{n}$. Define the sampling operator $ K_{\boldsymbol{x}}: \mathbb{R} \rightarrow \mathcal{H}, ~ y \mapsto y k(\boldsymbol{x},\cdot) $ and its adjoint operator $K_{\boldsymbol{x}}^{*}: \mathcal{H} \rightarrow \mathbb{R},~ f \mapsto f(\boldsymbol{x})$. Then we can define $T_{\boldsymbol{x}} = K_{\boldsymbol{x}} K_{\boldsymbol{x}}^{*}$. Further, we define the sample covariance operator $T_{\boldsymbol{X}}: \mathcal{H} \to \mathcal{H}$ as
\begin{equation}\label{def of TX}
    T_{\boldsymbol{X}}:=\frac{1}{n} \sum_{i=1}^n K_{\boldsymbol{x}_i} K_{\boldsymbol{x}_i}^*.
\end{equation}
Then we know that $ \| T_{\boldsymbol{X}} \| \le \left\| T_{\boldsymbol{X}} \right\|_{1} \le \kappa^{2}$ and $ T_{\boldsymbol{X}}$ is a trace class thus compact operator. Further, define the sample basis function
\begin{equation}
  g_{\boldsymbol{Z}}:=\frac{1}{n} \sum_{i=1}^n K_{\boldsymbol{x}_i} y_i \in \mathcal{H}.
\end{equation}
As shown in \citet{Caponnetto2007OptimalRF}, the operator form of the KRR estimator \eqref{krr estimator} writes
\begin{align}\label{eq:KRR}
  \hat{f}_\lambda = (T_{\boldsymbol{X}} + \lambda)^{-1} g_{\boldsymbol{Z}}.
\end{align}
In order to derive the bias term, we define 
\begin{equation}
    \tilde{g}_{\boldsymbol{Z}} := \mathbb{E}\left( g_{\boldsymbol{Z}} | \boldsymbol{X} \right) = \frac{1}{n} \sum_{i=1}^n K_{\boldsymbol{x}_i} f_{\rho}^{*}(\boldsymbol{x}_{i}) \in \mathcal{H};
\end{equation}
and
\begin{equation}\label{def tilde f lambda}
    \tilde{f}_{\lambda} := \mathbb{E}\left( \hat{f}_{\lambda} | \boldsymbol{X} \right) = \left(T_{\boldsymbol{X}} + \lambda\right)^{-1} \tilde{g}_{\boldsymbol{Z}} \in \mathcal{H}.
\end{equation}
We also need to define the expectation of $g_{\boldsymbol{Z}}$ as
\begin{equation}
    g = \mathbb{E} g_{\boldsymbol{Z}} = \int_{\mathcal{X}} k(\boldsymbol{x},\cdot) f_{\rho}^{*}(\boldsymbol{x}) d\mu(\boldsymbol{x}) = S_{k}^{*} f_{\rho}^* \in \mathcal{H},
\end{equation}
and
\begin{equation}\label{def f lambda}
    f_{\lambda} = \left( T+\lambda \right)^{-1} g = \left( T+\lambda \right)^{-1} S_{k}^{*} f_{\rho}^{*}.
\end{equation}

Then we have the decomposition
\begin{align}
    \hat{f}_{\lambda} - f_{\rho}^{*} &= \frac{1}{n} (T_{\boldsymbol{X}}+\lambda)^{-1} \sum_{i=1}^n K_{\boldsymbol{x}_i}y_i - f_{\rho}^{*} \notag \\
    &= \frac{1}{n} (T_{\boldsymbol{X}}+\lambda)^{-1} \sum_{i=1}^n K_{\boldsymbol{x}_i} (f^{*}_{\rho}(\boldsymbol{x}_i) + \epsilon_i) - f_{\rho}^{*} \notag \\
    &= (T_{\boldsymbol{X}}+\lambda)^{-1} \tilde{g}_{\boldsymbol{Z}} + \frac{1}{n}\sum_{i=1}^n (T_{\boldsymbol{X}}+\lambda)^{-1} K_{\boldsymbol{x}_i}\epsilon_i - f_{\rho}^{*} \notag \\
    &= \left( \tilde{f}_{\lambda} - f_{\rho}^{*} \right) + \frac{1}{n}\sum_{i=1}^n (T_{\boldsymbol{X}}+\lambda)^{-1} K_{\boldsymbol{x}_i}\epsilon_i.
\end{align}
Taking expectation over the noise $\epsilon$ conditioned on $X$ and noticing that $\epsilon | \boldsymbol{x}$ are independent noise with mean 0 and variance $\sigma^{2}$, we obtain the bias-variance decomposition:
\begin{align}\label{eq:BiasVarianceDecomp}
  \mathbb{E} \left( \left\|\hat{f}_\lambda - f^{*}_{\rho} \right\|^2_{L^2} \;\Big|\; \boldsymbol{X} \right)
  =  \mathbf{Bias}^2(\lambda) + \mathbf{Var}(\lambda),
\end{align}
where
\begin{equation}\label{eq:a_BiasVarFormula}
  \begin{aligned}
    & \mathbf{Bias}^2(\lambda) \coloneqq
    \left\|\tilde{f}_{\lambda} - f_{\rho}^{*}\right\|_{L^2}^2, \quad
    \mathbf{Var}(\lambda) \coloneqq
    \frac{\sigma^2}{n^2} \sum_{i=1}^n  \left\|(T_{\boldsymbol{X}}+\lambda)^{-1} k(\boldsymbol{x}_i,\cdot)\right\|^2_{L^2}.
  \end{aligned}
\end{equation}
Given the decomposition \eqref{eq:BiasVarianceDecomp}, we next derive the upper and lower bounds of $ \mathbf{Bias}^2(\lambda) $ and $ \mathbf{Var}(\lambda) $ in the following two subsections.

\subsection{Variance term}\label{section variance term}

In this subsection, our goal is to derive Theorem \ref{theorem variance approximation}, which shows the upper and lower bounds of variance under some approximation conditions. Before formally introduce Theorem \ref{theorem variance approximation}, we have a lot of preparatory work.

Following \citet{Li2023KernelIG}, we consider the sample subspace
  \begin{displaymath}
    \mathcal{H}_n = \mathrm{span} \left\{ k(\boldsymbol{x}_1,\cdot),\dots,k(\boldsymbol{x}_n,\cdot) \right\} \subset \mathcal{H}.
  \end{displaymath}
  Recall the notation $ \mathbb{K}(\boldsymbol{X},\boldsymbol{X}) = \left(k\left(\boldsymbol{x}_i, \boldsymbol{x}_j\right)\right)_{n \times n}$ and $ \mathbb{K}(\boldsymbol{X},\cdot) = \left\{ k(\boldsymbol{x}_1,\cdot),\dots,k(\boldsymbol{x}_n,\cdot) \right\}$. Define the normalized sample kernel matrix
  \begin{displaymath}
      \boldsymbol{K} = \frac{1}{n} \mathbb{K}(\boldsymbol{X},\boldsymbol{X}).
  \end{displaymath}
  Then, it is easy to verify that $\mathrm{Ran} (T_{\boldsymbol{X}}) = \mathcal{H}_{n}$ and $\boldsymbol{K}$ is the representation matrix of $T_{\boldsymbol{X}}$
  under the natural basis $\left\{ k(\boldsymbol{x}_1,\cdot),\dots,k(\boldsymbol{x}_n,\cdot) \right\}$.
  Consequently, for any continuous function $\varphi$ we have
  \begin{align}\label{eq:TXMatrixForm}
    \varphi(T_{\boldsymbol{X}}) \mathbb{K}(\boldsymbol{X},\cdot) = \varphi(\boldsymbol{K}) \mathbb{K}(X,\cdot),
  \end{align}
  where the left-hand side means applying the operator elementwise.
  Since the property of reproducing kernel Hilbert space implies $\langle k(\boldsymbol{x},\cdot),f \rangle_{\mathcal{H}} = f(\boldsymbol{x}), ~ \forall f \in \mathcal{H}$,
  taking inner product elementwise between \eqref{eq:TXMatrixForm} and $f$, we have
  \begin{align}
    \label{eq:TXActionF}
    (\varphi(T_{\boldsymbol{X}}) f)[\boldsymbol{X}] = \varphi(\boldsymbol{K})f[\boldsymbol{X}].
  \end{align}

  Moreover, for $f,g \in \mathcal{H}$, we define empirical semi-inner product
  \begin{align}
    \label{eq:SampleInnerProductL2}
    \langle f,g\rangle_{L^2,n} & = \frac{1}{n}\sum_{i=1}^n f(\boldsymbol{x}_i)g(\boldsymbol{x}_i) = \frac{1}{n}f[\boldsymbol{X}]^{\top} g[\boldsymbol{X}],
  \end{align}
  and denote by $\|\cdot\|_{L^2,n}$ the corresponding empirical semi-norm.
  We also denote by $P_n$ the empirical measure with respect to $\boldsymbol{X} = (\boldsymbol{x}_1,\dots,\boldsymbol{x}_n)$.

For simplicity of notations, we denote $k_{\boldsymbol{x}}(\cdot) = k(\boldsymbol{x},\cdot)$, $ \boldsymbol{x} \in \mathcal{X}$ in the rest of the proof. The following lemma rewrites the variance term \eqref{eq:a_BiasVarFormula} using the empirical semi-norm.
\begin{lemma}\label{lemma var trans}
    The variance term in \eqref{eq:a_BiasVarFormula} satisfies that
    \begin{align}
      \label{eq:VarAlterForm}
      \mathbf{Var}(\lambda) = \frac{\sigma^2}{n} \int_{\mathcal{X}} \left\|(T_{\boldsymbol{X}}+\lambda)^{-1}k_{\boldsymbol{x}}(\cdot)\right\|_{L^2,n}^2 \mathrm{d} \mu(\boldsymbol{x}).
    \end{align}
\end{lemma}
\begin{proof}
    First, we have
    \begin{align}\label{var trans 1}
        \mathbf{Var}(\lambda) &\coloneqq \frac{\sigma^2}{n^2} \sum_{i=1}^n  \left\|(T_{\boldsymbol{X}}+\lambda)^{-1} k(\boldsymbol{x}_i,\cdot)\right\|^2_{L^2} \notag \\
        &= \frac{\sigma^2}{n^2}\left\|(T_{\boldsymbol{X}}+\lambda)^{-1} \mathbb{K}(\boldsymbol{X},\cdot)\right\|^2_{L^2(\mathbb{R}^n)} \notag \\
        &\overset{\eqref{eq:TXMatrixForm}}{=} \frac{\sigma^2}{n^2}\left\|(\boldsymbol{K}+\lambda)^{-1} \mathbb{K}(\boldsymbol{X},\cdot)\right\|^2_{L^2(\mathbb{R}^n)} \notag \\
        &= \frac{\sigma^2}{n^2} \int_{\mathcal{X}} \mathbb{K}(\boldsymbol{x},\boldsymbol{X})(\boldsymbol{K}+\lambda)^{-2} \mathbb{K}(\boldsymbol{X},\boldsymbol{x}) ~ \mathrm{d} \mu(\boldsymbol{x}).
    \end{align}
    Next, using \eqref{eq:TXActionF} and the fact that $k_{\boldsymbol{x}}[\boldsymbol{X}] = \mathbb{K}(\boldsymbol{X},\boldsymbol{x})$, we have
    \begin{align*}
      \left( (T_{\boldsymbol{X}}+\lambda)^{-1}k_{\boldsymbol{x}}  \right)[\boldsymbol{X}] = (\boldsymbol{K}+\lambda)^{-1} k_{\boldsymbol{x}}[\boldsymbol{X}] = (\boldsymbol{K}+\lambda)^{-1}\mathbb{K}(\boldsymbol{X},\boldsymbol{x}),
    \end{align*}
   which implies
\begin{align}\label{var trans 2}
    \frac{1}{n} \mathbb{K}(\boldsymbol{x},\boldsymbol{X})(\boldsymbol{K}+\lambda)^{-2}\mathbb{K}(\boldsymbol{X},\boldsymbol{x})
    &= \frac{1}{n}\left\|(\boldsymbol{K}+\lambda)^{-1} \mathbb{K}(\boldsymbol{X},\boldsymbol{x})\right\|_{\mathbb{R}^n}^2 \notag \\
    &= \frac{1}{n}\left\|\left( (T_{\boldsymbol{X}}+\lambda)^{-1} k_{\boldsymbol{x}}  \right)[\boldsymbol{X}]\right\|_{\mathbb{R}^n}^2 \notag \\
    &= \left\|(T_{\boldsymbol{X}}+\lambda)^{-1}k_{\boldsymbol{x}}\right\|_{L^2,n}^2.
\end{align}
Therefore, plugging \eqref{var trans 2} into \eqref{var trans 1}, we get the desired results
\begin{displaymath}
      \mathbf{Var}(\lambda) = \frac{\sigma^2}{n} \int_{\mathcal{X}} \left\|(T_{\boldsymbol{X}}+\lambda)^{-1}k_{\boldsymbol{x}}(\cdot)\right\|_{L^2,n}^2 \mathrm{d} \mu(\boldsymbol{x}).
\end{displaymath}
\end{proof}

The operator form \eqref{eq:VarAlterForm} allows us to apply concentration inequalities and establish the following two-step approximation (recall the notations $T_{\boldsymbol{X} \lambda} $ and $ T_{\lambda} $ in \eqref{brief notation}). 
\begin{align}\label{eq:4_2Step}
    \int_{\mathcal{X}} \left\|{T_{\boldsymbol{X} \lambda}^{-1} k_{\boldsymbol{x}}}\right\|_{L^2, n}^{2} \mathrm{d}\mu(\boldsymbol{x})
    \stackrel{A}{\approx}
    \int_{\mathcal{X}} \left\|{T_{\lambda}^{-1} k_{\boldsymbol{x}}}\right\|_{L^2, n}^{2} \mathrm{d}\mu(\boldsymbol{x})
    \stackrel{B}{\approx}
    \int_{\mathcal{X}} \left\|{T_{\lambda}^{-1} k_{\boldsymbol{x}}}\right\|_{L^2}^{2} \mathrm{d}\mu(\boldsymbol{x}).
  \end{align}
Note that the above two-step approximation is an enhanced version of approximation (S24) in \cite{Li2023KernelIG}.\\

\textit{Approximation B.}  The following lemma characterizes the magnitude of Approximation B in high probability. Recall the definitions of $\mathcal{N}_{1}(\lambda)$ and $ \mathcal{N}_{2}(\lambda)$ in \eqref{n1 n2 m1 m2}.
\begin{lemma}[Approximation B]\label{lemma approximation B}
    Suppose that Assumption \ref{assumption kernel}, \ref{assumption noise} and \ref{assumption eigenfunction} hold. For any $\lambda = \lambda(d,n) \to 0$ and any fixed $ \delta \in (0,1)$, when $n$ is sufficiently large, with probability at least $1-\delta$, we have
\begin{align}\label{approximation B}
   \frac{1}{2}\int_{\mathcal{X}} \left\|{T_{\lambda}^{-1} k_{\boldsymbol{x}}}\right\|_{L^2}^{2} \mathrm{d}\mu(\boldsymbol{x}) - R_{2} \le \int_{\mathcal{X}} \left\|{T_{\lambda}^{-1} k_{\boldsymbol{x}}}\right\|_{L^2, n}^{2} \mathrm{d}\mu(\boldsymbol{x}) \le \frac{3}{2}\int_{\mathcal{X}} \left\|{T_{\lambda}^{-1} k_{\boldsymbol{x}}}\right\|_{L^2}^{2} \mathrm{d}\mu(\boldsymbol{x}) + R_{2},
\end{align}
where 
\begin{equation}
    R_{2} = \frac{5\mathcal{N}_{2}(\lambda)}{3n}\ln \frac{2}{\delta}.
\end{equation}
\end{lemma}
\begin{proof}
    Define a function
\begin{align}
     f(\boldsymbol{z}) &= \int_{\mathcal{X}} \left(T_{\lambda}^{-1}k_{\boldsymbol{x}}(\boldsymbol{z})\right)^{2} \mathrm{d}\mu(\boldsymbol{x}) \notag \\
     &= \int_{\mathcal{X}} \sum\limits_{i=1}^{\infty} \left( \frac{\lambda_{i}}{\lambda_{i} + \lambda} \right)^{2} e_{i}^{2}(\boldsymbol{x}) e_{i}^{2}(\boldsymbol{z}) \mathrm{d} \mu(\boldsymbol{x})\notag \\
     &= \sum\limits_{i=1}^{\infty} \left( \frac{\lambda_{i}}{\lambda_{i} + \lambda} \right)^{2}  e_{i}^{2}(\boldsymbol{z}).
\end{align}
Since Assumption \ref{assumption eigenfunction} holds, we have 
\begin{displaymath}
    \left\| f \right\|_{L^{\infty}}  \le \sum\limits_{i=1}^{\infty} \left( \frac{\lambda_{i}}{\lambda_{i} + \lambda} \right)^{2} = \mathcal{N}_{2}(\lambda); ~~ \left\| f \right\|_{L^{1}} = \sum\limits_{i=1}^{\infty} \left( \frac{\lambda_{i}}{\lambda_{i} + \lambda} \right)^{2} = \mathcal{N}_{2}(\lambda). 
\end{displaymath}
Applying Proposition \ref{prop:SampleNormEstimation} for $ \sqrt{f}$ and noticing that $\| \sqrt{f}\|_{L^{\infty}} = \sqrt{\| f \|_{L^{\infty}}} = \mathcal{N}_{2}(\lambda)^{\frac{1}{2}} $, we have
\begin{equation}\label{R2 proof 1}
    \frac{1}{2}\left\| \sqrt{f} \right\|_{L^2}^2 - \frac{5\mathcal{N}_{2}(\lambda)}{3n}\ln \frac{2}{\delta} \leq \left\|\sqrt{f}\right\|_{L^2,n}^2 \leq
      \frac{3}{2}\left\|\sqrt{f}\right\|_{L^2}^2 + \frac{5\mathcal{N}_{2}(\lambda)}{3n}\ln \frac{2}{\delta},
\end{equation}
with probability at least $ 1 - \delta$. 

On the one hand, we have 
\begin{align}
    \left\|\sqrt{f}\right\|_{L^2,n}^2  &= \int_{\mathcal{X}} f(\boldsymbol{y}) \mathrm{d}P_{n}(\boldsymbol{y}) = \int_{\mathcal{X}} \left[ \int_{\mathcal{X}} \left(T_{\lambda}^{-1}k_{\boldsymbol{x}}(\boldsymbol{y})\right)^{2} \mathrm{d}\mu(\boldsymbol{x}) \right] \mathrm{d}P_{n}(\boldsymbol{y}) \notag \\
    &= \int_{\mathcal{X}} \left[ \int_{\mathcal{X}} \left(T_{\lambda}^{-1}k_{\boldsymbol{x}}(\boldsymbol{y})\right)^{2} \mathrm{d}P_{n}(\boldsymbol{y})  \right] \mathrm{d}\mu(\boldsymbol{x})  \notag \\
    &= \int_{\mathcal{X}} \left\|{T_{\lambda}^{-1} k_{\boldsymbol{x}}}\right\|_{L^2, n}^{2} \mathrm{d}\mu(\boldsymbol{x}). \notag
\end{align}
On the other hand, we have
\begin{align}
    \left\|\sqrt{f}\right\|_{L^2}^2  &= \int_{\mathcal{X}} f(\boldsymbol{z}) \mathrm{d}\mu(\boldsymbol{z}) \notag \\
    &= \int_{\mathcal{X}} \left[ \int_{\mathcal{X}} \left(T_{\lambda}^{-1}k_{\boldsymbol{x}}(\boldsymbol{z})\right)^{2} \mathrm{d}\mu(\boldsymbol{x}) \right] \mathrm{d}\mu(\boldsymbol{z}) \notag \\
    &= \int_{\mathcal{X}} \left\|{T_{\lambda}^{-1} k_{\boldsymbol{x}}}\right\|_{L^2}^{2} \mathrm{d}\mu(\boldsymbol{x}). \notag
\end{align}
Therefore, \eqref{R2 proof 1} implies the desired results.
\end{proof}

\textit{Approximation A.}  The proof of Approximation A requires the following proposition, which is a simple but important observation by \citet{li2023_SaturationEffect}.
  \begin{proposition}\label{proposition observation}
    For any $f,g \in \mathcal{H}$, we have
    \begin{equation}
      \label{eq:SampleInnerProductsRelation}
      \langle f,g \rangle_{L^2,n} = \langle T_{\boldsymbol{X}} f,g\rangle_{\mathcal{H}} = \langle T_{\boldsymbol{X}}^{1/2} f, T_{\boldsymbol{X}}^{1/2} g\rangle_{\mathcal{H}}.
    \end{equation}
  \end{proposition}
  \begin{proof}
    Notice that $T_{\boldsymbol{X}} f = \frac{1}{n}\sum_{i=1}^n f(\boldsymbol{x}_i) k(\boldsymbol{x}_{i},\cdot)$, and thus
    \begin{align*}
      \langle T_{\boldsymbol{X}} f,g \rangle_{\mathcal{H}} = \frac{1}{n}\sum_{i=1}^n f(\boldsymbol{x}_i) \langle k(\boldsymbol{x}_{i},\cdot), g\rangle_{\mathcal{H}}
      = \frac{1}{n}\sum_{i=1}^n f(\boldsymbol{x}_i) g(\boldsymbol{x}_i) = \langle f,g \rangle_{L^2,n}.
    \end{align*}
    The second inequality comes from the definition of $T_{\boldsymbol{X}}^{1/2}$.
  \end{proof}
  
The following lemma characterizes the magnitude of Approximation A in high probability. 

\begin{lemma}[Approximation A]\label{lemma approximation A}
    Suppose that Assumption \ref{assumption kernel}, \ref{assumption noise} and \ref{assumption eigenfunction} hold. Define $R_{1} = R_{1}(\lambda, X)$ as a function of $\lambda$ and $X$:
    \begin{align}\label{approximation A}
    R_{1} \coloneqq \left| \int_{\mathcal{X}} \left\|{T_{\boldsymbol{X} \lambda}^{-1} k_{\boldsymbol{x}}}\right\|_{L^2, n}^{2} \mathrm{d}\mu(\boldsymbol{x}) - \int_{\mathcal{X}} \left\|{T_{\lambda}^{-1} k_{\boldsymbol{x}}}\right\|_{L^2, n}^{2} \mathrm{d}\mu(\boldsymbol{x})  \right|.
\end{align}
 Suppose that $ \lambda = \lambda(d,n)$ satisfies $\frac{\mathcal{N}_{1}(\lambda)}{n} \ln{n} = o(1)$. Then for any fixed $\delta \in (0,1)$, when $n$ is sufficiently large, with probability at least $1-\delta$, we have
    \begin{align}
        R_{1} &\le 36 n^{-1} \mathcal{N}_{1}(\lambda)^{2} \ln{n} + 12 n^{-\frac{1}{2}} \mathcal{N}_{1}(\lambda) \mathcal{N}_{2}(\lambda)^{\frac{1}{2}} \left( \ln{n} \right)^{\frac{1}{2}}.
    \end{align}
\end{lemma}
\begin{proof}
    First, we rewrite $R_{1}$ as
\begin{align}\label{proof appr A 1}
    R_{1} &= \left| \int_{\mathcal{X}} \left\|{T_{\boldsymbol{X} \lambda}^{-1} k_{\boldsymbol{x}}}\right\|_{L^2, n}^{2} \mathrm{d}\mu(\boldsymbol{x}) - \int_{\mathcal{X}} \left\|{T_{\lambda}^{-1} k_{\boldsymbol{x}}}\right\|_{L^2, n}^{2} \mathrm{d}\mu(\boldsymbol{x})  \right| \notag \\
    &\le  \int_{\mathcal{X}} \left| \left\|{T_{\boldsymbol{X}}^{\frac{1}{2}} T_{\boldsymbol{X} \lambda}^{-1} k_{\boldsymbol{x}}}\right\|_{\mathcal{H}}^{2}  - \left\|{ T_{\boldsymbol{X}}^{\frac{1}{2}} T_{\lambda}^{-1} k_{\boldsymbol{x}}}\right\|_{\mathcal{H}}^{2}  \right| \mathrm{d}\mu(\boldsymbol{x}) \notag \\
    &= \int_{\mathcal{X}} \left| \left\|{T_{\boldsymbol{X}}^{\frac{1}{2}} T_{\boldsymbol{X} \lambda}^{-1} k_{\boldsymbol{x}}}\right\|_{\mathcal{H}} -  \left\|{ T_{\boldsymbol{X}}^{\frac{1}{2}} T_{\lambda}^{-1} k_{\boldsymbol{x}}}\right\|_{\mathcal{H}} \right| \cdot \left| \left\|{T_{\boldsymbol{X}}^{\frac{1}{2}} T_{\boldsymbol{X} \lambda}^{-1} k_{\boldsymbol{x}}}\right\|_{\mathcal{H}} + \left\|{ T_{\boldsymbol{X}}^{\frac{1}{2}} T_{\lambda}^{-1} k_{\boldsymbol{x}}}\right\|_{\mathcal{H}} \right| \mathrm{d}\mu(\boldsymbol{x}). \notag \\
    &:= \int_{\mathcal{X}} \left| X_{1} - X_{2} \right| \cdot \left| X_{1} + X_{2} \right| \mathrm{d}\mu(\boldsymbol{x}).
\end{align}
where for the second line, we use Proposition \ref{proposition observation} to transfer $ \|\|_{L^{2},n}$ norms into $\| \|_{\mathcal{H}}$. 

$\left( \text{\uppercase\expandafter{\romannumeral1}} \right): $ Given the notations of $X_{1}, X_{2}$ in \eqref{proof appr A 1} (both are functions of $\boldsymbol{x}, \boldsymbol{X}$ and $\lambda$), we begin to handle $ \left| X_{1} - X_{2} \right| $: 
\begin{align}\label{x1-x2 bound}
    \left| X_{1} - X_{2} \right| &\le \left\|{T_{\boldsymbol{X}}^{1/2}T_{\boldsymbol{X} \lambda}^{-1} \left( T - T_{\boldsymbol{X}} \right) T_{\lambda}^{-1}k_{\boldsymbol{x}}}\right\|_{\mathcal{H}} \notag \\
    &\le \left\|{T_{\boldsymbol{X}}^{1/2}T_{\boldsymbol{X} \lambda}^{-1/2}}\right\| \cdot
      \left\|{T_{\boldsymbol{X} \lambda}^{-1/2} T_{\lambda}^{1/2}}\right\|
      \cdot \left\|{T_{\lambda}^{-1/2} \left( T - T_{\boldsymbol{X}} \right)T_{\lambda}^{-1/2} }\right\|
      \cdot \left\|{T_{\lambda}^{-1/2}k_{\boldsymbol{x}}}\right\|_{\mathcal{H}}
\end{align}
$\left( \text{\lowercase\expandafter{\romannumeral1}} \right)~ $ Now we bound the third term in \eqref{x1-x2 bound}. Denote $A_{i} = T_\lambda^{-\frac{1}{2}} (T - T_{\boldsymbol{x}_{i}}) T_\lambda^{-\frac{1}{2}} $, using Lemma \ref{emb norm}, we have 
  \begin{displaymath}
      \| A_{i} \| = \| T_\lambda^{-\frac{1}{2}} T T_\lambda^{-\frac{1}{2}} \| + \| T_\lambda^{-\frac{1}{2}} T_{\boldsymbol{x}_{i}} T_\lambda^{-\frac{1}{2}} \| \le 2 \mathcal{N}_{1}(\lambda), \quad \mu \text {-a.e. } \boldsymbol{x} \in \mathcal{X}.
  \end{displaymath}
  We use $ A \preceq B$ to denote that $ A-B$ is a positive semi-definite operator. Using the fact that $\mathbb{E}(B-\mathbb{E} B)^2 \preceq \mathbb{E} B^2$ for a self-adjoint operator $B$, we have
  \begin{displaymath}
    \mathbb{E} A_{i}^{2} \preceq \mathbb{E}\left[T_\lambda^{-\frac{1}{2}} T_{\boldsymbol{x}_{i}} T_\lambda^{-\frac{1}{2}}\right]^2.
    \end{displaymath}
  In addition, Lemma \ref{emb norm} shows that $0 \preceq T_\lambda^{-\frac{1}{2}} T_{\boldsymbol{x}_{i}} T_\lambda^{-\frac{1}{2}} \preceq \mathcal{N}_{1}(\lambda), \mu \text {-a.e. } \boldsymbol{x} \in \mathcal{X}$. So we have
  \begin{displaymath}
     \mathbb{E} A_{i}^{2}  \preceq \mathbb{E} \left[T_\lambda^{-\frac{1}{2}} T_{\boldsymbol{x}_{i}} T_\lambda^{-\frac{1}{2}}\right]^{2} \preceq \mathbb{E}\left[ \mathcal{N}_{1}(\lambda) \cdot T_\lambda^{-\frac{1}{2}} T_{\boldsymbol{x}_{i}} T_\lambda^{-\frac{1}{2}}\right] = \mathcal{N}_{1}(\lambda) T_{\lambda}^{-1} T,
  \end{displaymath}
  Define an operator $V := \mathcal{N}_{1}(\lambda) T_{\lambda}^{-1} T$, we have  
  \begin{align}
      \| V \| &= \mathcal{N}_{1}(\lambda) \frac{\lambda_{1}}{\lambda_{1} + \lambda} = \mathcal{N}_{1}(\lambda) \frac{\|T\|}{\|T\| + \lambda} \le \mathcal{N}_{1}(\lambda); \notag \\
      \text{tr}V &= \mathcal{N}_{1}(\lambda)^{2}; \notag \\
      \frac{\text{tr}V}{ \| V \|} &= \frac{\mathcal{N}_{1}(\lambda) (\|T\| + \lambda)}{\|T\|}. \notag
  \end{align}
  Using Lemma \ref{lemma concentration of operator} to $A_{i}$, $V$, for any fixed $\delta \in (0,1)$, with probability at least $1-\delta$, we have
    \begin{displaymath}
        \Vert T_\lambda^{-\frac{1}{2}} (T - T_{\boldsymbol{X}}) T_\lambda^{-\frac{1}{2}} \Vert
        \le \frac{4 \mathcal{N}_{1}(\lambda)}{3n} \beta + \sqrt {\frac{2 \mathcal{N}_{1}(\lambda)}{n} \beta},
   \end{displaymath}
   where 
   \begin{displaymath}
       \beta = \ln{\frac{4 \mathcal{N}_{1}(\lambda) (\|T\| + \lambda) }{\delta \|T\|}}.
   \end{displaymath}
Further recall that the condition $ \mathcal{N}_{1}(\lambda) \ln{n}/n = o(1)$ implies that $\mathcal{N}_{1}(\lambda) = O(n)$ and thus $\beta = O(\ln n)$, so when $n$ is sufficiently large, we can conclude that
\begin{equation}\label{x1-x2-1}
    \Vert T_\lambda^{-\frac{1}{2}} (T - T_{\boldsymbol{X}}) T_\lambda^{-\frac{1}{2}} \Vert \le \sqrt {\frac{2 \mathcal{N}_{1}(\lambda)}{n} \beta}
        \le n^{-\frac{1}{2}} \mathcal{N}_{1}(\lambda)^{\frac{1}{2}} \left( \ln{n} \right)^{\frac{1}{2}}.
\end{equation}
$\left( \text{\lowercase\expandafter{\romannumeral2}} \right)~ $Next we bound the forth term in \eqref{x1-x2 bound}. Using Lemma \ref{due embedding bound}, we have
\begin{equation}\label{x1-x2-2}
    \|T_{\lambda}^{-\frac{1}{2}} k(\boldsymbol{x},\cdot) \|_{\mathcal{H}} \le \mathcal{N}_{1}(\lambda)^{\frac{1}{2}},~~ \mu \text {-a.e. } \boldsymbol{x} \in \mathcal{X}.
\end{equation}
$\left( \text{\lowercase\expandafter{\romannumeral3}} \right)~ $Finally, we bound the first two terms in \eqref{x1-x2 bound}. Since we have assumed $ \mathcal{N}_{1}(\lambda) \ln{n} / n = o(1)$, \eqref{x1-x2-1} implies that when $n$ is sufficiently large, we have
\begin{displaymath}
    a := \Vert T_\lambda^{-\frac{1}{2}} (T - T_{\boldsymbol{X}}) T_\lambda^{-\frac{1}{2}} \Vert \le \frac{2}{3}.
\end{displaymath}
Therefore, we have
   \begin{align}\label{x1-x2 3}
   \left\|T_\lambda^{-\frac{1}{2}} T_{\boldsymbol{X} \lambda}^{\frac{1}{2}}\right\|^2 & =\left\|T_\lambda^{-\frac{1}{2}} T_{\boldsymbol{X} \lambda} T_\lambda^{-\frac{1}{2}}\right\|=\left\|T_\lambda^{-\frac{1}{2}}\left(T_{\boldsymbol{X}}+\lambda\right) T_\lambda^{-\frac{1}{2}}\right\| \notag \\ & =\left\|T_\lambda^{-\frac{1}{2}}\left(T_{\boldsymbol{X}}-T+T+\lambda\right) T_\lambda^{-\frac{1}{2}}\right\| \notag \\ 
   & =\left\|T_\lambda^{-\frac{1}{2}}\left(T_{\boldsymbol{X}}-T\right) T_\lambda^{-\frac{1}{2}}+I\right\| \notag\\ 
   & \leq a+1 \leq 2; 
   \end{align}
   and
   \begin{align}\label{x1-x2 4}
    \left\|T_\lambda^{\frac{1}{2}} T_{\boldsymbol{X} \lambda}^{-\frac{1}{2}}\right\|^2 & =\left\|T_\lambda^{\frac{1}{2}} T_{\boldsymbol{X} \lambda}^{-1} T_\lambda^{\frac{1}{2}}\right\|=\left\|\left(T_\lambda^{-\frac{1}{2}} T_{\boldsymbol{X} \lambda} T_\lambda^{-\frac{1}{2}}\right)^{-1}\right\| \notag \\
    &= \left\|\left(I - T_{\lambda}^{-\frac{1}{2}} (T_{\boldsymbol{X}}-T) T_{\lambda}^{-\frac{1}{2}} \right)^{-1}\right\| \notag \\
    &\le \sum\limits_{k=0}^{\infty} \left\| T_{\lambda}^{-\frac{1}{2}} (T_{\boldsymbol{X}}-T) T_{\lambda}^{-\frac{1}{2}} \right\|^{k}  \notag \\
    &\le \sum\limits_{k=0}^{\infty} \left(\frac{2}{3}\right)^{k} \le 3. 
    \end{align}
Plugging \eqref{x1-x2-1}, \eqref{x1-x2-2}, \eqref{x1-x2 3} and \eqref{x1-x2 4}, into \eqref{x1-x2 bound}, when $n$ is sufficiently large, with probability at least $1-\delta$, we have
\begin{equation}\label{x1-x2 bound final}
    | X_{1} -X_{2} | \le 6 n^{-\frac{1}{2}} \mathcal{N}_{1}(\lambda) \left( \ln{n} \right)^{\frac{1}{2}},~~ \mu \text {-a.e. } \boldsymbol{x} \in \mathcal{X}.
\end{equation}
$\left( \text{\uppercase\expandafter{\romannumeral2}} \right): $ Further, when $n$ is sufficiently large, with probability at least $1-\delta$, we also have
\begin{align}\label{x2 int}
    \int_{\mathcal{X}} X_{2} \mathrm{d} \mu(\boldsymbol{x}) &= \int_{\mathcal{X}} \left\|{T_{\lambda}^{-1} k_{\boldsymbol{x}}}\right\|_{L^2, n} \mathrm{d} \mu(\boldsymbol{x}) \notag \\
    &\le \left[  \int_{\mathcal{X}} \left\|{T_{\lambda}^{-1} k_{\boldsymbol{x}}}\right\|_{L^2, n}^{2}\mathrm{d} \mu(\boldsymbol{x}) \right]^{\frac{1}{2}} \notag \\
    &\le \left(\frac{3}{2}\mathcal{N}_{2}(\lambda) + R_{2}\right)^{\frac{1}{2}} \notag \\
    &\le \left(2 \mathcal{N}_{2}(\lambda)\right)^{\frac{1}{2}},
\end{align}
where the third line follows from Lemma \ref{lemma approximation B}.

Now we are ready to derive the upper bound of $R_{1}$. Combining the bound \eqref{x1-x2 bound final} and \eqref{x2 int}, when $n$ is sufficiently large, with probability at least $1-2\delta$, we have
\begin{align}\label{R1 bound proof final}
    R_{1} &\le \int_{\mathcal{X}} \left| X_{1} - X_{2} \right| \cdot \left| X_{1} + X_{2} \right| \mathrm{d}\mu(\boldsymbol{x}) \notag \\
    &= \int_{\mathcal{X}} \left| X_{1} - X_{2} \right| \cdot \left| X_{1} - X_{2} + 2 X_{2} \right| \mathrm{d}\mu(\boldsymbol{x}) \notag \\
    &\le \int_{\mathcal{X}} \left| X_{1} - X_{2} \right|^{2}\mathrm{d}\mu(\boldsymbol{x}) + \int_{\mathcal{X}} \left| X_{1} - X_{2} \right| \cdot 2 X_{2} ~ \mathrm{d}\mu(\boldsymbol{x}) \notag \\
    &\le 36 n^{-1} \mathcal{N}_{1}(\lambda)^{2}\ln{n} + 24 n^{-\frac{1}{2}} \mathcal{N}_{1}(\lambda) \mathcal{N}_{2}(\lambda)^{\frac{1}{2}} \left( \ln{n} \right)^{\frac{1}{2}}.
\end{align}
Without loss of generality, we can assume \eqref{R1 bound proof final} holds with probability at least $ 1- \delta$ and we finish the proof.
\end{proof}

\textit{Final proof of the variance term.} Now we are ready to state the theorem about the variance term. 

\begin{theorem}[Variance term]\label{theorem variance approximation}
    Suppose that Assumption \ref{assumption kernel}, \ref{assumption noise} and \ref{assumption eigenfunction} hold. If the following approximation conditions hold for some $\lambda = \lambda(d,n) \to 0$:
    \begin{align}\label{var conditions}
    \frac{\mathcal{N}_{1}(\lambda)}{n} \ln{n} = o(1); ~~n^{-1} \mathcal{N}_{1}(\lambda)^{2} \ln{n} = o(\mathcal{N}_{2}(\lambda)),
\end{align}
then we have
\begin{equation}
    \mathbf{Var}(\lambda) = \Theta_{\mathbb{P}}\left(\frac{\sigma^{2} \mathcal{N}_{2}(\lambda)}{n}\right).
\end{equation}
\end{theorem}
\begin{proof}
    Lemma \ref{lemma var trans} has shown that
    \begin{displaymath}
      \mathbf{Var}(\lambda) = \frac{\sigma^2}{n} \int_{\mathcal{X}} \left\|(T_{\boldsymbol{X}}+\lambda)^{-1}k_{\boldsymbol{x}}(\cdot)\right\|_{L^2,n}^2 \mathrm{d} \mu(\boldsymbol{x}).
    \end{displaymath}
Denote $R_{1}$ as in Lemma \ref{lemma approximation A}, then conditions \eqref{var conditions} and Lemma \ref{lemma approximation A} imply that 
\begin{displaymath}
    R_{1} = o_{\mathbb{P}}\left(\mathcal{N}_{2}(\lambda)\right).
\end{displaymath}
Further recall that in Lemma \ref{lemma approximation B}, we have defined
\begin{displaymath}
     R_{2} = \frac{5\mathcal{N}_{2}(\lambda)}{3n}\ln \frac{2}{\delta}.
\end{displaymath}
Then on the one hand, Lemma \ref{lemma approximation B} shows that, for any $\delta \in (0,1)$, when $n$ is sufficiently large, with probability at least $1-\delta$, we have
\begin{align}
    n \mathrm{\textbf{Var}}(\lambda) / \sigma^{2} = \int_{\mathcal{X}} \left\|{T_{\boldsymbol{X} \lambda}^{-1} k_{\boldsymbol{x}}}\right\|_{L^2, n}^{2} \mathrm{d}\mu(\boldsymbol{x}) &\le  \int_{\mathcal{X}} \left\|{T_{\lambda}^{-1} k_{\boldsymbol{x}}}\right\|_{L^2, n}^{2} \mathrm{d}\mu(\boldsymbol{x}) + R_{1} \notag \\
    &\le \frac{3}{2} \int_{\mathcal{X}} \left\|{T_{\lambda}^{-1} k_{\boldsymbol{x}}}\right\|_{L^2}^{2} \mathrm{d}\mu(\boldsymbol{x}) + R_{1} + R_{2} \notag \\
    &= \frac{3}{2} \mathcal{N}_{2}(\lambda) + R_{1} + R_{2}, \notag
\end{align}
which further implies
\begin{equation}\label{var appr upper}
    n \mathrm{\textbf{Var}}(\lambda) / \sigma^{2} = O_{\mathbb{P}}\left(\mathcal{N}_{2}(\lambda)\right).
\end{equation}
On the other hand, we also have
\begin{align}
    n \mathrm{\textbf{Var}}(\lambda) / \sigma^{2} = \int_{\mathcal{X}} \left\|{T_{\boldsymbol{X} \lambda}^{-1} k_{\boldsymbol{x}}}\right\|_{L^2, n}^{2} \mathrm{d}\mu(\boldsymbol{x}) &\ge  \int_{\mathcal{X}} \left\|{T_{\lambda}^{-1} k_{\boldsymbol{x}}}\right\|_{L^2, n}^{2} \mathrm{d}\mu(\boldsymbol{x}) - R_{1} \notag \\
    &\ge \frac{1}{2} \int_{\mathcal{X}} \left\|{T_{\lambda}^{-1} k_{\boldsymbol{x}}}\right\|_{L^2}^{2} \mathrm{d}\mu(\boldsymbol{x}) - R_{1} - R_{2} \notag \\
    &= \frac{1}{2} \mathcal{N}_{2}(\lambda) - R_{1} - R_{2}, \notag
\end{align}
which further implies
\begin{equation}\label{var appr lower}
    n \mathrm{\textbf{Var}}(\lambda) / \sigma^{2} = \Omega_{\mathbb{P}}\left(\mathcal{N}_{2}(\lambda)\right).
\end{equation}
Combining \eqref{var appr upper} and \eqref{var appr lower}, we finish the proof.
\end{proof}

\subsection{Bias term}\label{section bias term}
In this subsection, our goal is to derive Theorem \ref{theorem bias approximation}, which shows the upper and lower bounds of bias under some approximation conditions. 

The triangle inequality implies that 
\begin{equation}\label{proof bias thm-1}
    \mathrm{\textbf{Bias}}(\lambda) = \left\| \tilde{f}_{\lambda} - f_{\rho}^{*}\right\|_{L^{2}} \ge \left\| f_{\lambda} - f_{\rho}^{*}\right\|_{L^{2}} - \left\| \tilde{f}_{\lambda} - f_{\lambda}\right\|_{L^{2}},
\end{equation}
where $ \tilde{f}_{\lambda}, f_{\lambda}$ are defined as \eqref{def tilde f lambda} and \eqref{def f lambda}.\\

The following lemma characterizes the dominant term of $\mathrm{\textbf{Bias}}(\lambda) $.
\begin{lemma}\label{lemma bias main term}
    Suppose that Assumption \ref{assumption kernel} and \ref{assumption noise} hold. Then for any $ \lambda = \lambda(d,n) \to 0$, we have
    \begin{equation}\label{goal bias bound}
    \left\| f_{\lambda} - f_{\rho}^{*}\right\|_{L^{2}} = \mathcal{M}_{2}(\lambda)^{\frac{1}{2}}.
\end{equation}
\end{lemma}
\begin{proof}
    Recall that we have defined $ f_{\rho}^{*} = \sum\limits_{i=1}^{\infty} f_{i} e_{i}(\boldsymbol{x}) \in L^{2}(\mathcal{X},\mu)$ and $ f_{\lambda} = \left( T + \lambda \right)^{-1} S_{k}^{*} f_{\rho}^{*}$. Therefore, we have
    \begin{align}
        \left\| f_{\lambda} - f_{\rho}^{*}\right\|_{L^{2}}^{2} &= \left\|\sum\limits_{i=1}^{\infty} f_{i} e_{i}(\boldsymbol{x}) - \sum\limits_{i=1}^{\infty} \frac{\lambda_{i}}{\lambda_{i} +  \lambda} f_{i} e_{i}(\boldsymbol{x}) \right\|_{L^{2}}^{2} \notag \\
        &= \left\|\sum\limits_{i=1}^{\infty} \frac{\lambda}{\lambda_{i} +  \lambda} f_{i} e_{i}(\boldsymbol{x}) \right\|_{L^{2}}^{2} \notag \\
        &= \sum\limits_{i=1}^{\infty} \left(\frac{\lambda}{\lambda_{i} +  \lambda} f_{i}\right)^{2} \notag \\
        &= \mathcal{M}_{2}(\lambda). \notag
    \end{align}
\end{proof}

Our next goal is to prove that second term in \eqref{proof bias thm-1} is higher order infinitesimal, i.e., $ \left\| \tilde{f}_{\lambda} - f_{\lambda}\right\|_{L^{2}} = o_{\mathbb{P}}(\mathcal{M}_{2}(\lambda)^{\frac{1}{2}})$.

\begin{lemma}\label{lemma bias appr term}
Suppose that Assumption \ref{assumption kernel}, \ref{assumption noise} and \ref{assumption eigenfunction} hold. If the following approximation conditions hold for some $\lambda = \lambda(d,n) \to 0$:
    \begin{align}\label{bias conditions}
    \frac{\mathcal{N}_{1}(\lambda)}{n} \ln{n} = o(1);~~ n^{-1} \mathcal{N}_{1}(\lambda)^{\frac{1}{2}} \mathcal{M}_{1}(\lambda) = o\left(\mathcal{M}_{2}(\lambda)^{\frac{1}{2}}\right),
\end{align}
then we have
    \begin{equation}
        \left\| \tilde{f}_{\lambda} - f_{\lambda}\right\|_{L^{2}} = o_{\mathbb{P}}(\mathcal{M}_{2}(\lambda)^{\frac{1}{2}}).
    \end{equation}
\end{lemma}
\begin{proof}
To begin with, be definition, we rewrite \eqref{goal bias bound} as follows
\begin{align}\label{appr proof-0}
    \left\|\tilde{f}_\lambda - f_{\lambda}\right\|_{L_{2}} & =\left\| S_{k} \left(\tilde{f}_\lambda-f_\lambda\right)\right\|_{L^{2}} \notag \\
    & =\left\| S_{k} T_\lambda^{-\frac{1}{2}} \cdot T_\lambda^{\frac{1}{2}} T_{\boldsymbol{X} \lambda}^{-1} T_\lambda^{\frac{1}{2}} \cdot T_\lambda^{-\frac{1}{2}} T_{\boldsymbol{X} \lambda} \left(\tilde{f}_\lambda-f_\lambda\right)\right\|_{L^{2}} \notag \\
    &\leq \left\| S_{k} T_\lambda^{-\frac{1}{2}}\right\|_{\mathscr{B}\left(\mathcal{H},L^{2}\right)} \cdot \left\| T_{\lambda}^{\frac{1}{2}} T_{\boldsymbol{X} \lambda}^{-1} T_{\lambda}^{\frac{1}{2}} \right\|_{\mathscr{B}\left(\mathcal{H},\mathcal{H}\right)} \cdot \left\| T_{\lambda}^{-\frac{1}{2}} \left( \tilde{g}_{\boldsymbol{Z}} - T_{\boldsymbol{X} \lambda} f_{\lambda} \right) \right\|_{\mathcal{H}}.
\end{align}
$\left( \text{\lowercase\expandafter{\romannumeral1}} \right)~$ For any $f \in \mathcal{H}$ and $\|f\|_{\mathcal{H}}=1$, suppose that $f = \sum\limits_{i=1}^{\infty} a_{i} \lambda_{i}^{1/2} e_{i}$ satisfying that $ \sum\limits_{i=1}^{\infty} a_{i}^{2} = 1$. So for the first term in \eqref{appr proof-0}, we have
\begin{align}\label{bias appr plug 1}
    \left\| S_{k} T_\nu^{-\frac{1}{2}}\right\|_{\mathscr{B}\left(\mathcal{H},L^{2}\right)} &= \sup_{\|f\|_{\mathcal{H}}=1} \left\|S_{k} T_\nu^{-\frac{1}{2}} f\right\|_{L^{2}} \notag \\
    &\le \sup_{\|f\|_{\mathcal{H}}=1} \left\|\sum\limits_{i=1}^{\infty} \frac{\lambda_{i}^{\frac{1}{2}}}{(\lambda_{i}+\lambda)^{\frac{1}{2}}} a_{i} e_{i} \right\|_{L^{2}} \notag \\
    &\le  \sup_{i \ge 1} \frac{\lambda_{i}^{\frac{1}{2}}}{(\lambda_{i}+\lambda)^{\frac{1}{2}}} \cdot \sup_{\|f\|_{\mathcal{H}}=1} \left\|\sum\limits_{i=1}^{\infty} a_{i} e_{i} \right\|_{L^{2}} \notag \\
    &\le 1. 
\end{align}
$\left( \text{\lowercase\expandafter{\romannumeral2}} \right)~$For the second term, since we have assumed $\mathcal{N}_{1}(\lambda) \ln{n} / n = o(1)$, for any fixed $\delta \in (0,1)$, when $n$ is sufficiently large, we have proved in \eqref{x1-x2 3} and \eqref{x1-x2 4} that, with probability at least $ 1-\delta$
\begin{equation}\label{bias appr plug 2}
    \left\| T_{\lambda}^{\frac{1}{2}} T_{\boldsymbol{X} \lambda}^{-1} T_{\lambda}^{\frac{1}{2}} \right\| \le \left\| T_{\lambda}^{\frac{1}{2}} T_{\boldsymbol{X} \lambda}^{-\frac{1}{2}}  \right\| \cdot \left\| T_{\boldsymbol{X} \lambda}^{-\frac{1}{2}}  T_{\lambda}^{\frac{1}{2}} \right\| \le 3.
\end{equation}
$\left( \text{\lowercase\expandafter{\romannumeral3}} \right)~$For the third term in \eqref{appr proof-0}, it can be rewritten as 
\begin{align}\label{appr proof-1}
    \left\| T_{\lambda}^{-\frac{1}{2}} \left( \tilde{g}_{\boldsymbol{Z}} - T_{\boldsymbol{X} \lambda} f_{\lambda} \right) \right\|_{\mathcal{H}}
    &=\left\|T_\lambda^{-\frac{1}{2}}\left[\left(\tilde{g}_{\boldsymbol{Z}} - \left(T_{\boldsymbol{X}} + \lambda + T - T \right) f_\lambda\right)\right]\right\|_{\mathcal{H}} \notag \\
    &=\left\|T_\lambda^{-\frac{1}{2}}\left[\left(\tilde{g}_{\boldsymbol{Z}} - T_{\boldsymbol{X}} f_\lambda\right) - \left(T + \lambda \right) f_\lambda + T f_\lambda \right]\right\|_{\mathcal{H}} \notag \\
    &= \left\|T_\lambda^{-\frac{1}{2}}\left[\left(\tilde{g}_{\boldsymbol{Z}}-T_{\boldsymbol{X}} f_\lambda\right)-\left(g-T f_\lambda\right)\right]\right\|_{\mathcal{H}}.
\end{align}
Denote $\xi_{i} = \xi(\boldsymbol{x}_{i}) =  T_{\lambda}^{-\frac{1}{2}}(K_{\boldsymbol{x}_{i}} f_{\rho}^{*}(\boldsymbol{x}_{i}) - T_{\boldsymbol{x}_{i}} f_{\lambda}) $. To use Bernstein inequality, we need to bound the $m$-th moment of $\xi(\boldsymbol{x})$:
\begin{align}\label{proof of 4.9-1}
       \mathbb{E} \left\| \xi(\boldsymbol{x}) \right\|_{\mathcal{H}}^{m} &= \mathbb{E} \left\| T_{\lambda}^{-\frac{1}{2}} K_{\boldsymbol{x}}(f_{\rho}^{*} - f_{\lambda}(\boldsymbol{x})) \right\|_{\mathcal{H}}^{m} \notag \\
       &\le \mathbb{E} \Big( \left\| T_{\lambda}^{-\frac{1}{2}} k(\boldsymbol{x},\cdot)\right\|_{\mathcal{H}}^{m}  \mathbb{E} \big( \left|(f_{\rho}^{*} - f_{\lambda}(\boldsymbol{x})) \right|^{m} ~\big|~ \boldsymbol{x} \big) \Big).
\end{align}
Note that Lemma \ref{due embedding bound} shows that
\begin{displaymath}
  \left\| T_{\lambda}^{-\frac{1}{2}} k(\boldsymbol{x},\cdot)\right\|_{\mathcal{H}} \le \mathcal{N}_{1}(\lambda)^{\frac{1}{2}},~~ \mu \text {-a.e. } \boldsymbol{x} \in \mathcal{X};
\end{displaymath}
By definition of $\mathcal{M}_{1}(\lambda)$, we also have
\begin{equation}\label{m1 occurs}
    \left\| f_{\lambda} - f_{\rho}^{*} \right\|_{L^{\infty}} \le \operatorname*{ess~sup}_{\boldsymbol{x} \in \mathcal{X}} \left| \sum\limits_{i=1}^{\infty} \frac{\lambda}{\lambda_{i} + \lambda} f_{i} e_{i}(\boldsymbol{x}) \right| = \mathcal{M}_{1}(\lambda).
\end{equation}
In addition, we have proved in Lemma \ref{lemma bias main term} that
\begin{displaymath}
    \mathbb{E} | (f_{\lambda}(\boldsymbol{x}) - f_{\rho}^{*}(\boldsymbol{x}))  |^{2} = \mathcal{M}_{2}(\lambda).
\end{displaymath}
So we get the upper bound of (\ref{proof of 4.9-1}), i.e.,
\begin{align}
    (\ref{proof of 4.9-1}) &\le \mathcal{N}_{1}(\lambda)^{\frac{m}{2}} \cdot  \| f_{\lambda} - f_{\rho}^{*}  \|_{L^{\infty}}^{m-2} \cdot \mathbb{E} | (f_{\lambda}(\boldsymbol{x}) - f_{\rho}^{*}(\boldsymbol{x}))  |^{2} \notag \\
    &\le \mathcal{N}_{1}(\lambda)^{\frac{m}{2}} \mathcal{M}_{1}(\lambda)^{m-2} \mathcal{M}_{2}(\lambda) \notag \\
    &\le \left( \mathcal{N}_{1}(\lambda)^{\frac{1}{2}} \mathcal{M}_{1}(\lambda) \right)^{m-2} \left( \mathcal{N}_{1}(\lambda)^{\frac{1}{2}} \mathcal{M}_{2}(\lambda)^{\frac{1}{2}} \right)^{2}. \notag
\end{align}
Using Lemma \ref{bernstein} with therein notations: $L =  \mathcal{N}_{1}(\lambda)^{\frac{1}{2}} \mathcal{M}_{1}(\lambda)$ and $\sigma  = \mathcal{N}_{1}(\lambda)^{\frac{1}{2}} \mathcal{M}_{2}(\lambda)^{\frac{1}{2}} $, for any fixed $\delta \in (0,1)$, with probability at least $1-\delta$, we have
\begin{equation}\label{proof bias appr 3}
    \eqref{appr proof-1} \le 4\sqrt{2} \log \frac{2}{\delta} \left( \frac{ \mathcal{N}_{1}(\lambda)^{\frac{1}{2}} \mathcal{M}_{1}(\lambda)}{n} + \frac{\mathcal{N}_{1}(\lambda)^{\frac{1}{2}} \mathcal{M}_{2}(\lambda)^{\frac{1}{2}}}{\sqrt{n}} \right).
\end{equation}
Since we have assumed $n^{-1} \mathcal{N}_{1}(\lambda)^{\frac{1}{2}} \mathcal{M}_{1}(\lambda) = o\left(\mathcal{M}_{2}(\lambda)^{\frac{1}{2}}\right)$ and $\mathcal{N}_{1}(\lambda) \ln{n} / n = o(1)$,  \eqref{proof bias appr 3} further implies
\begin{equation}\label{bias appr plug 3}
    \left\| T_{\lambda}^{-\frac{1}{2}} \left( \tilde{g}_{\boldsymbol{Z}} - T_{\boldsymbol{X} \lambda} f_{\lambda} \right) \right\|_{\mathcal{H}} = o_{\mathbb{P}}\left(\mathcal{M}_{2}(\lambda)^{\frac{1}{2}}\right).
\end{equation}
Plugging \eqref{bias appr plug 1}, \eqref{bias appr plug 2} and \eqref{bias appr plug 3} into \eqref{appr proof-0}, we finish the proof.

\end{proof}

\textit{Final proof of the bias term.} Now we are ready to state the theorem about the bias term.
\begin{theorem}\label{theorem bias approximation}
    Suppose that Assumption \ref{assumption kernel}, \ref{assumption noise} and \ref{assumption eigenfunction} hold. If the following approximation conditions hold for some $\lambda = \lambda(d,n) \to 0$:
    \begin{align}\label{bias theorem proof condition}
    \frac{\mathcal{N}_{1}(\lambda)}{n} \ln{n} = o(1); ~~\text{and}~~ n^{-1} \mathcal{N}_{1}(\lambda)^{\frac{1}{2}} \mathcal{M}_{1}(\lambda) = o\left(\mathcal{M}_{2}(\lambda)^{\frac{1}{2}}\right),
\end{align}
then we have
\begin{equation}\label{bias theorem proof condition 2}
    \mathbf{Bias}^{2}(\lambda) = \Theta_{\mathbb{P}}\left(\mathcal{M}_{2}(\lambda) \right).
\end{equation}
\end{theorem}
\begin{proof}
The triangle inequality implies that 
\begin{displaymath}
    \mathrm{\textbf{Bias}}(\lambda) = \left\| \tilde{f}_{\lambda} - f_{\rho}^{*}\right\|_{L^{2}} \ge \left\| f_{\lambda} - f_{\rho}^{*}\right\|_{L^{2}} - \left\| \tilde{f}_{\lambda} - f_{\lambda}\right\|_{L^{2}},
\end{displaymath}
When $\lambda = \lambda(d,n)$ satisfies \eqref{bias theorem proof condition}, Lemma \ref{lemma bias main term} and Lemma \ref{lemma bias appr term} prove that 
\begin{displaymath}
    \left\| f_{\lambda} - f_{\rho}^{*}\right\|_{L^{2}} = \mathcal{M}_{2}(\lambda)^{\frac{1}{2}};~~ \left\| \tilde{f}_{\lambda} - f_{\lambda}\right\|_{L^{2}} = o_{\mathbb{P}}(\mathcal{M}_{2}(\lambda)^{\frac{1}{2}}),
\end{displaymath}
which directly prove \eqref{bias theorem proof condition 2}.
\end{proof}

\subsection{Final proof of Theorem \ref{main theorem}}\label{section final proof of main theorem}
Now we are ready to prove Theorem \ref{main theorem}. Note that we have assumed in Theorem \ref{main theorem} that $ \lambda = \lambda(d,n) \to 0$ satisfies all the conditions required in Theorem \ref{theorem variance approximation} and Theorem \ref{theorem bias approximation}. Therefore, Theorem \ref{theorem variance approximation} and Theorem \ref{theorem bias approximation} show that
\begin{displaymath}
    \mathrm{\textbf{Var}}(\lambda) = \Theta_{\mathbb{P}}\left(\frac{\sigma^{2} \mathcal{N}_{2}(\lambda)}{n}\right);~~ \mathbf{Bias}^{2}(\lambda) = \Theta_{\mathbb{P}}\left(\mathcal{M}_{2}(\lambda) \right).
\end{displaymath}
Recalling the bias-variance decomposition \eqref{eq:BiasVarianceDecomp}, we finish the proof.\\

$\hfill\blacksquare$ 

\section{Proof of inner product kernel}\label{section proofs inner}
In this section, we aim to apply Theorem \ref{main theorem} to prove the results in Section \ref{section app inner}. We will see that the application is nontrivial and is an important contribution of this paper.

We first introduce more necessary preliminaries in Appendix \ref{section prelimi about inner}, which is an preparation for subsequent calculations. Next, in order to apply Theorem \ref{main theorem} to get specific convergence rates, we calculate the exact convergence rates of the key quantities therein in Appendix \ref{section calcu of key}. Finally, we state the proof of Theorem \ref{theorem inner s ge 1} and  Theorem \ref{theorem inner s le 1} in turn in Appendix \ref{section proof of inner ge 1} and \ref{section proof of inner le 1}. We will see that there are essential differences in the proof of these two theorems.

\subsection{More preliminaries about inner product kernel on the sphere}\label{section prelimi about inner}
Suppose that $ \mathcal{X} = \mathbb{S}^{d}$ and $\mu$ is the uniform distribution on $ \mathbb{S}^{d} $. Recall that in Section \ref{section app inner}, we consider the inner product kernel, i.e., there exists a function $ \Phi(t): [-1,1] \to \mathbb{R}$ such that $ k(\boldsymbol{x}, \boldsymbol{x}^{\prime}) = \Phi\left( \langle \boldsymbol{x}, \boldsymbol{x}^{\prime} \rangle\right), \forall \boldsymbol{x}, \boldsymbol{x}^{\prime} \in \mathbb{S}^{d}$. Then Mercer's decomposition for the inner product kernel is given in the basis of spherical harmonics:
\begin{equation}\label{mercer of inner}
  k(\boldsymbol{x},\boldsymbol{y}) = \sum_{k=0}^{\infty} \mu_k \sum_{l=1}^{N(d,k)} Y_{k,l}(\boldsymbol{x})Y_{k,l}(\boldsymbol{y}),
\end{equation}
where $\{Y_{k,l}\}_{l=1}^{N(d,k)}$ are spherical harmonic polynomials of degree $k$; $ \mu_{k}$ are the eigenvalues with multiplicity $N(d,0)=1$; $N(d, k) = \frac{2k+d-1}{k} \cdot \frac{(k+d-2)!}{(d-1)!(k-1)!}, k =1,2,\cdots$.

By known results on spherical harmonics, the eigenvalues $\mu_{k}$'s have the following explicit expression \citep{Bietti2019OnTI}:
\begin{equation}\label{eqn:explicit_eigenvalues_inner_product_kernel}
\begin{aligned}
\mu_k=\frac{\omega_{d-1}}{\omega_{d}} \int_{-1}^1 \Phi(t) P_k(t)\left(1-t^2\right)^{(d-2) / 2} ~\mathrm{d} t,
\end{aligned}
\end{equation}
where $P_k$ is the $k$-th Legendre polynomial in dimension $d+1$, $\omega_{d}
$ denotes the surface of the sphere $\mathbb{S}^{d}$.
$ $\\

Although the above expression of $\mu_{k}, N(d,k)$ are complicated, Lemma \ref{lemma inner eigen} $\sim$ \ref{lemma Ndk} (mainly cited from \citealt{lu2023optimal}) give concise characterizations of $\mu_{k}$ and $N(d,k)$, which is sufficient for the analysis in this paper.

\begin{lemma}\label{lemma inner eigen}
    Suppose that $k = \{ k_{d}\}_{d=1}^{\infty}$ is a sequence of inner product kernels on the sphere satisfying Assumption \ref{assumption kernel} and \ref{assumption inner product kernel}. For any fixed integer $p \ge 0$, there exist constants $\mathfrak{C}, \mathfrak{C}_1$ and $\mathfrak{C}_2$ only depending on $p$ and $\{a_j\}_{j \leq p+1}$, such that for any $d \geq \mathfrak{C}$, we have
\begin{equation}
\begin{aligned}
{\mathfrak{C}_1}{d^{-k}} &\leq \mu_{k} \leq {\mathfrak{C}_2}{d^{-k}}, ~~ k=0,1,\cdots, p+1.
\end{aligned}
\end{equation}
\end{lemma}
\begin{proof}
    From equation (22) in \cite{Ghorbani2019LinearizedTN}, for any integer $p \ge 0$, there exist constants $\mathfrak{C}$ only depending on $p$ and $\{a_j\}_{j \leq p+1}$, such that for any $d \ge \mathfrak{C}$, we have
\begin{equation}
\begin{aligned}
\frac{ \Phi^{(k)}(0)}{d^{k}} &\leq \mu_{k} \leq \frac{2 \Phi^{(k)}(0)}{d^{k}}, \quad k =0,1,\cdots, p+1.
\end{aligned}
\end{equation}
Note that for any $k \geq 0$, we have $a_k=\Phi^{(k)}(0)$.
Therefore, letting $\mathfrak{C}_1 := \min_{k \leq p+1}\{a_k\}>0$ and $\mathfrak{C}_2 := 2 \max_{k \leq p+1}\{a_k\}<\infty$, then we finish the proof.
\end{proof}

The following property of the eigenvalues $\{\mu_k\}_{k\geq 0}$ indicates that when we consider $\mu_{p}$ with any fixed $ p \ge 0$, the subsequent eigenvalues $ \mu_{k}$'s ($ k\ge p+1$) are much smaller than $\mu_{p}$. The following lemma is the same as Lemma 3.3 in \citet{lu2023optimal}.

\begin{lemma}\label{lemma:monotone_of_eigenvalues_of_inner_product_kernels}
    Suppose that $k = \{ k_{d}\}_{d=1}^{\infty}$ is a sequence of inner product kernels on the sphere satisfying Assumption \ref{assumption kernel} and \ref{assumption inner product kernel}. For any fixed integer $p \ge 0$, there exist constants $\mathfrak{C}$ only depending on $p$ and $\{a_j\}_{j \leq p+1}$, such that for any $d \geq \mathfrak{C}$, we have
    \begin{equation*}
        \mu_k \leq \frac{\mathfrak{C}_2}{\mathfrak{C}_1} d^{-1} \mu_{p}, \quad k=p+1, p+2, \cdots
    \end{equation*}
    where $\mathfrak{C}_1$ and $\mathfrak{C}_2$ are constants given in Lemma \ref{lemma inner eigen}.
\end{lemma}

\begin{lemma}\label{lemma Ndk}
    For an integer $ k \ge 0$, denote $N(d,k)$ as the multiplicity of the eigenspace corresponding to $\mu_{k}$ in the Mercer's decomposition. For any fixed integer $p \ge 0$, there exist constants $\mathfrak{C}_3, \mathfrak{C}_4$ and $\mathfrak{C}$ only depending on $p$, such that for any $d \ge \mathfrak{C}$, we have
    \begin{equation}\label{Ndk rate}
        \mathfrak{C}_3 d^k \le N(d, k)  \le \mathfrak{C}_4 d^k, \quad k = 0, 1, \cdots, p+1.
    \end{equation}
\end{lemma}
\begin{proof}
    When $k=0$, we have $N(d,0)=1$, which satisfies \eqref{Ndk rate}. When $k \ge 1 $, Section 1.6 in \cite{Gallier2009SphericalHA} shows that 
    \begin{displaymath}
        N(d, k) = \frac{2k+d-1}{k} \cdot \frac{(k+d-2)!}{(d-1)!(k-1)!}.
    \end{displaymath}
    Note that $p$ is fixed and we consider those $ k \le p+1$, \eqref{Ndk rate} follows from detailed calculations using Stirling's approximation. We refer to Lemma B.1 and Lemma D.4 in \cite{lu2023optimal} for more details.
\end{proof}

The following lemma verifies that if we consider the inner product kernel on the sphere, then Assumption \ref{assumption eigenfunction} naturally holds. 
\begin{lemma}\label{lemma inner assumption holds}
    Suppose that $ \mathcal{X} = \mathbb{S}^{d}$ and $\mu$ is the uniform distribution on $ \mathbb{S}^{d} $. Suppose that $ k$ is an inner product kernel, then Assumption \ref{assumption eigenfunction} holds.
\end{lemma}
\begin{proof}
    Recall that we have the mercer decomposition \eqref{mercer of inner}. Define the sum
\begin{displaymath}
    Z_{k,d}(\boldsymbol{x},\boldsymbol{y}) = \sum_{l=1}^{N(d,k)} Y_{k,l}(\boldsymbol{x})Y_{k,l}(\boldsymbol{y}).
\end{displaymath}
Then \citet[Corollary 1.2.7]{dai2013_ApproximationTheory} shows that $Z_{k,d}(\boldsymbol{x},\boldsymbol{y})$ depends only on $\langle x,y \rangle$ and satisfies~
\begin{displaymath}
    \left|Z_{k,d}(\boldsymbol{x},\boldsymbol{y})\right| \leq Z_{k,d}(\boldsymbol{x},\boldsymbol{x}) = N(d,k), \quad \forall x,y \in \mathbb{S}^d.
\end{displaymath}
Therefore, we have
\begin{align}
    \sup\limits_{\boldsymbol{x} \in \mathcal{X}} \sum\limits_{i=1}^{\infty} \left( \frac{\lambda_{i}}{\lambda_{i} + \lambda} \right)^{2}e_{i}^{2}(\boldsymbol{x}) &= \sup\limits_{\boldsymbol{x} \in \mathcal{X}} \sum\limits_{k=1}^{\infty} \sum\limits_{l=1}^{N(d,k)} \left( \frac{\mu_{k}}{\mu_{k} + \lambda} \right)^{2} Y_{k,l}^{2}(\boldsymbol{x}) = \sum\limits_{k=1}^{\infty } \left( \frac{\mu_{k}}{\mu_{k} + \lambda} \right)^{2} \sup\limits_{\boldsymbol{x} \in \mathcal{X}}  Z_{k,d}(\boldsymbol{x},\boldsymbol{x}) \notag \\
    &= \sum\limits_{k=1}^{\infty } \left( \frac{\mu_{k}}{\mu_{k} + \lambda} \right)^{2} N(d,k) = \sum\limits_{k=1}^{\infty} \sum\limits_{l=1}^{N(d,k)} \left( \frac{\mu_{k}}{\mu_{k} + \lambda} \right)^{2} \notag \\
    &= \mathcal{N}_{2}(\lambda). \notag
  \end{align}
  The other equation in Assumption \ref{assumption eigenfunction} can be proved similarly. 
\end{proof}

\subsection{Calculations of some key quantities}\label{section calcu of key}
Based on the information of the eigenvalues in the last subsection, this subsection determines the exact convergence rates of the quantities appeared in Theorem \ref{main theorem}. These rates will finally determine the convergence rates in Theorem \ref{theorem inner s ge 1} and  Theorem \ref{theorem inner s le 1}. Note that we assume $d$ diverges to infinite with $n$ in Theorem \ref{theorem inner s ge 1} and  Theorem \ref{theorem inner s le 1} .

\begin{lemma}\label{lemma calculation n1 n2}
   Consider $\mathcal{X} = \mathbb{S}^{d}$ and the marginal distribution $\mu$ to be the uniform distribution. Let $k = \{ k_{d}\}_{d=1}^{\infty}$ be a sequence of inner product kernels on the sphere satisfying Assumption \ref{assumption kernel} and \ref{assumption inner product kernel}. By choosing $\lambda = d^{-l}$ for some $l > 0$, we have:
    \begin{equation}\label{inner product N1}
        \mathcal{N}_{1}(\lambda) = \Theta\left(\lambda^{-1}\right);
    \end{equation}
    If $ p \le l \le p+1$ for some $ p \in \{0,1,2\cdots\}$, we have
    \begin{equation}\label{inner product N2}
        \mathcal{N}_{2}(\lambda) = \Theta\left(d^{p} + \lambda^{-2} d^{-(p+1)}\right).
    \end{equation}
    The notation $\Theta$ involves constants only depending on $\kappa$ and $p$.
\end{lemma}
\begin{proof}
If $ p \le l \le p+1$ for some $ p \in \{0,1,2\cdots\}$, Lemma \ref{lemma inner eigen} and Lemma \ref{lemma Ndk} show that there exist constants $\mathfrak{C}, \mathfrak{C}_1, \mathfrak{C}_2, \mathfrak{C}_{3}$ and $ \mathfrak{C}_{4}$ only depending on $p$ (recall that we ignore the dependence on $\{a_{j}\}_{j=0}^{\infty}$), such that for any $d \geq \mathfrak{C}$, we have
    \begin{equation}
        \mathfrak{C}_{2}^{-1} \mu_{p+1} \le \lambda \le \mathfrak{C}_{1}^{-1} \mu_{p};
    \end{equation}
    and for $k = 0, 1, \cdots, p+1  $,
    \begin{equation}
        {\mathfrak{C}_1}{d^{-k}} \leq \mu_{k} \leq {\mathfrak{C}_2}{d^{-k}}; ~~ \mathfrak{C}_3 d^k \le N(d, k)  \le \mathfrak{C}_4 d^k.
    \end{equation}

We first prove \eqref{inner product N1}. On the one hand, for any $d \geq \mathfrak{C}$, we have
\begin{align}\label{proof n1 1}
    \mathcal{N}_{1}(\lambda) &= \sum\limits_{k=0}^{p} \frac{\mu_{k}}{\mu_{k} + \lambda} N(d,k) + \sum\limits_{k=p+1}^{\infty} \frac{\mu_{k}}{\mu_{k} + \lambda} N(d,k) \notag \\
    &\le \sum\limits_{k=0}^{p} \frac{\mu_{k}}{\mu_{k}} N(d,k) + \sum\limits_{k=p+1}^{\infty} \frac{\mu_{k}}{ \lambda} N(d,k) \notag \\
    &\le p \mathfrak{C}_4 d^{p} + \lambda^{-1} \sum\limits_{k=p+1}^{\infty} \mu_{k} N(d,k) \notag \\
    &\le p \mathfrak{C}_4 d^{p} + \lambda^{-1} \kappa^{2} \notag \\
    &\lesssim \lambda^{-1},
\end{align} 
where we use the fact that $ \sum\limits_{k=p+1}^{\infty} \mu_{k} N(d,k) \le \sum\limits_{i=1}^{\infty} \lambda_{i} \le \sup\limits_{\boldsymbol{x} \in \mathcal{X}} k(x,x) \le \kappa^{2}$ for the third inequality.\\
On the other hand, for any $d \geq \mathfrak{C}$, we have
\begin{align}\label{proof n1 2}
    \mathcal{N}_{1}(\lambda) &= \sum\limits_{k=0}^{p} \frac{\mu_{k}}{\mu_{k} + \lambda} N(d,k) + \sum\limits_{k=p+1}^{\infty} \frac{\mu_{k}}{\mu_{k} + \lambda} N(d,k) \notag \\
    &\ge \frac{\mu_{p}}{\mu_{p} + \lambda} N(d,p) + \frac{\mu_{p+1}}{ \mu_{p+1}+\lambda} N(d,p+1) \notag \\
    &\ge \frac{\mu_{p}}{\mu_{p} + \mathfrak{C}_{1}^{-1} \mu_{p} } N(d,p) + \frac{\mu_{p+1}}{\mathfrak{C}_{2} \lambda + \lambda} N(d,p+1) \notag \\
    &\ge \frac{\mathfrak{C}_{3}}{1+\mathfrak{C}_{1}^{-1}} d^{p} + \frac{\mathfrak{C}_{1} \mathfrak{C}_{3}}{\mathfrak{C}_{2}+1} \lambda^{-1} \notag \\
    &\gtrsim \lambda^{-1}.
\end{align} 
Combining \eqref{proof n1 1} and \eqref{proof n1 2}, we finish the proof of \eqref{inner product N1}.
$ $\\

Now we begin to prove \eqref{inner product N2}. On the one hand, for any $d \geq \mathfrak{C}$ we have
\begin{align}\label{proof n2 1}
    \mathcal{N}_{2}(\lambda) &= \sum\limits_{k=0}^{p} \left(\frac{\mu_{k}}{\mu_{k} + \lambda}\right)^{2} N(d,k) + \sum\limits_{k=p+1}^{\infty} \left(\frac{\mu_{k}}{\mu_{k} + \lambda}\right)^{2} N(d,k) \notag \\
    &\le \sum\limits_{k=0}^{p} \left(\frac{\mu_{k}}{\mu_{k}}\right)^{2} N(d,k) + \sum\limits_{k=p+1}^{\infty} \left(\frac{\mu_{k}}{ \lambda}\right)^{2} N(d,k) \notag \\
    &\le p \mathfrak{C}_{4} d^{p} + \lambda^{-2} \sum\limits_{k=p+1}^{\infty} \mu_{k}^{2} N(d,k) \notag \\
    &\le p \mathfrak{C}_{4} d^{p} + \lambda^{-2} \mu_{p+1} \sum\limits_{k=p+1}^{\infty} \mu_{k} N(d,k) \notag \\
    &\le p \mathfrak{C}_{4} d^{p} + \lambda^{-2} \mathfrak{C}_{2} d^{-(p+1)} \kappa^{2}.
\end{align}
Note that for the forth equation, we use the fact that $ \mu_{k} \le \mu_{p+1}, \forall k \ge p+1$ (when d is sufficiently large), which can be proved by Lemma \ref{lemma:monotone_of_eigenvalues_of_inner_product_kernels}.\\
On the other hand, for any $d \geq \mathfrak{C}$, we have 
\begin{align}\label{proof n2 2}
    \mathcal{N}_{2}(\lambda) &= \sum\limits_{k=0}^{p} \left(\frac{\mu_{k}}{\mu_{k} + \lambda}\right)^{2} N(d,k) + \sum\limits_{k=p+1}^{\infty} \left(\frac{\mu_{k}}{\mu_{k} + \lambda}\right)^{2} N(d,k) \notag \\
    &\ge \left(\frac{\mu_{k}}{\mu_{p} + \mathfrak{C}_{1}^{-1} \mu_p}\right)^{2} N(d,p) +\left(\frac{\mu_{p+1}}{\mathfrak{C}_{2} \lambda + \lambda}\right)^{2} N(d,p+1) \notag \\
    &\ge \frac{\mathfrak{C}_{3}}{\left(1 + \mathfrak{C}_{1}^{-1}\right)^{2}} d^{p} + \lambda^{-2} \left(\frac{\mathfrak{C}_{1} }{ \mathfrak{C}_{2} +1}\right)^{2} \mathfrak{C}_{3}  d^{-(p+1)}.
\end{align}
Combining \eqref{proof n2 1} and \eqref{proof n2 2}, we prove that $ \mathcal{N}_{2}(\lambda) = \Theta\left(d^{p} + \lambda^{-2} d^{-(p+1)}\right)$.

\end{proof}

Before stating lemmas about $\mathcal{M}_{1}(\lambda), \mathcal{M}_{2}(\lambda)$, we first introduce the following useful lemma.
\begin{lemma}\label{lemma useful m}
    Consider $\mathcal{X} = \mathbb{S}^{d}$ and the marginal distribution $\mu$ to be the uniform distribution. Let $k = \{ k_{d}\}_{d=1}^{\infty}$ be a sequence of inner product kernels on the sphere satisfying Assumption \ref{assumption kernel} and \ref{assumption inner product kernel}. Further suppose that Assumption \ref{assumption source condition} holds for some $ s > 0$. Suppose that $\alpha, \beta $ are two real numbers such that 
         \begin{equation}\label{useful condition 1}
             s + \alpha - \beta \le 0; ~~ s + \alpha \ge 0.
         \end{equation}
     Then by choosing $\lambda = d^{-l}$ for some $0 < l < \gamma$, if $ p \le l \le p+1$ for some $ p \in \{0,1,2\cdots\}$, we have
    \begin{equation}
        \mathcal{M}_{\alpha,\beta}(\lambda) := \sum\limits_{i=1}^{\infty} \frac{\lambda_{i}^{\alpha}}{\left(\lambda_{i} + \lambda \right)^{\beta}} f_{i}^{2} = \Theta \left( d^{-(s+\alpha-\beta)p} + \lambda^{-\beta} d^{-(p+1)(s+\alpha)}\right).
    \end{equation}
     The notation $\Theta$ involves constants only depending on $s, p, c_{0}$ and $ R_{\gamma}$, where $ c_{0}$ and $R_{\gamma} $ are the constants from Assumption \ref{assumption source condition}.

\end{lemma}
\begin{proof}
    Similar as the proof of Lemma \ref{lemma calculation n1 n2}, if $ p \le l \le p+1$ for some $ p \in \{0,1,2\cdots\}$, there exist constants $\mathfrak{C}, \mathfrak{C}_1, \mathfrak{C}_2, \mathfrak{C}_{3}$ and $ \mathfrak{C}_{4}$ only depending on $p$, such that for any $d \geq \mathfrak{C}$, we have
    \begin{equation}
        \mathfrak{C}_{2}^{-1} \mu_{p+1} \le \lambda \le \mathfrak{C}_{1}^{-1} \mu_{p};
    \end{equation}
    and for $k = 0, 1, \cdots, p+1  $,
    \begin{equation}
        {\mathfrak{C}_1}{d^{-k}} \leq \mu_{k} \leq {\mathfrak{C}_2}{d^{-k}}; ~~ \mathfrak{C}_3 d^k \le N(d, k)  \le \mathfrak{C}_4 d^k.
    \end{equation}
    
On the one hand, since \eqref{useful condition 1} holds, for any $d \geq \mathfrak{C}$, we have
    \begin{align}\label{proof useful 1}
        \mathcal{M}_{\alpha,\beta}(\lambda) &= \sum\limits_{k=0}^{\infty} \left( \frac{\mu_{k}^{s+\alpha}}{\left( \mu_{k} + \lambda \right)^{\beta}} \sum\limits_{i \in \mathcal{I}_{d,k}} \mu_{k}^{-s} f_{i}^{2} \right) \notag \\
        &\le   \sum\limits_{k=0}^{p} \frac{\mu_{k}^{s+\alpha}}{\left( \mu_{k} + \lambda \right)^{\beta}} R_{\gamma}^{2}  + \sum\limits_{k=p+1}^{\infty} \frac{\mu_{k}^{s+\alpha}}{\left( \mu_{k}  + \lambda \right)^{\beta}} \sum\limits_{i \in \mathcal{I}_{d,k}} \mu_{k}^{-s} f_{i}^{2}  \notag \\
        &\le  \sum\limits_{k=0}^{p} \mu_{k}^{s+\alpha-\beta}  R_{\gamma}^{2} + \sum\limits_{k=p+1}^{\infty} \frac{\mu_{k}^{s+\alpha}}{\lambda ^{\beta}} \sum\limits_{i \in \mathcal{I}_{d,k}} \mu_{k}^{-s} f_{i}^{2}  \notag \\
        &\le  p \mu_{p}^{s+\alpha-\beta} R_{\gamma}^{2} + \lambda^{-\beta} \mu_{p+1}^{s+\alpha} \sum\limits_{k=p+1}^{\infty} \sum\limits_{i \in \mathcal{I}_{d,k}} \mu_{k}^{-s} f_{i}^{2}  \notag \\
        &\le p R_{\gamma}^{2} \mathfrak{C}_{2}^{(s + \alpha - \beta)} d^{-(s + \alpha - \beta)p} + \lambda^{-\beta} \mathfrak{C}_{2}^{s+\alpha} d^{-(p+1)(s+\alpha)} R_{\gamma}^{2} \notag \\
        &\lesssim d^{-(s+\alpha-\beta)p} + \lambda^{-\beta} d^{-(p+1)(s+\alpha)}.
    \end{align}
    Note that Assumption \ref{assumption source condition} (a) implies $\sum\limits_{k=0}^{\infty} \sum\limits_{i \in \mathcal{I}_{d,k}} \mu_{k}^{-s} f_{i}^{2} = \sum\limits_{i=1}^{\infty} \lambda_{i}^{-s} f_{i}^{2} \le R_{\gamma}^{2}$; We also use the fact that $ \mu_{k} \le \mu_{p+1}, \forall k \ge p+1$, which can be proved by Lemma \ref{lemma:monotone_of_eigenvalues_of_inner_product_kernels}.\\
    On the other hand, for any $d \geq \mathfrak{C}$, we have
    \begin{align}\label{proof useful 2}
        \mathcal{M}_{\alpha,\beta}(\lambda) &= \sum\limits_{k=0}^{\infty} \left( \frac{\mu_{k}^{s+\alpha}}{\left( \mu_{k} + \lambda \right)^{\beta}} \sum\limits_{i \in \mathcal{I}_{d,k}} \mu_{k}^{-s} f_{i}^{2} \right) \notag \\
        &\ge \frac{\mu_{p}^{s+\alpha}}{\left( \mu_{p} + \lambda \right)^{\beta}} \sum\limits_{i \in \mathcal{I}_{d,p}} \mu_{p}^{-s} f_{i}^{2} + \frac{\mu_{p+1}^{s+\alpha}}{\left( \mu_{p+1} + \lambda \right)^{\beta}} \sum\limits_{i \in \mathcal{I}_{d,p+1}} \mu_{p+1}^{-s} f_{i}^{2} \notag \\
        &\ge  \frac{\mu_{p}^{s+\alpha}}{\left( \mu_{p} + \mathfrak{C}_{1}^{-1} \mu_{p} \right)^{\beta}} \sum\limits_{i \in \mathcal{I}_{d,p}} \mu_{p}^{-s} f_{i}^{2} + \frac{\mu_{p+1}^{s+\alpha}}{\left( \mathfrak{C}_{2} \lambda + \lambda \right)^{\beta}} \sum\limits_{i \in \mathcal{I}_{d,p+1}} \mu_{p+1}^{-s} f_{i}^{2} \notag \\
        &\ge \frac{\mathfrak{C}_{3}^{s+\alpha-\beta}}{\left( 1+ \mathfrak{C}_{1}^{-1}\right)^{\beta}} d^{-(s+\alpha-\beta)p} c_{0} + \lambda^{-\beta} \frac{\mathfrak{C}_{1}^{s+\alpha}}{\left( \mathfrak{C}_{2} + 1 \right)^{\beta}} d^{-(p+1)(s+\alpha)} c_{0} \notag \\
        &\gtrsim d^{-(s+\alpha-\beta)p} + \lambda^{-\beta} d^{-(p+1)(s+\alpha)}.
    \end{align}
    We use Assumption \ref{assumption source condition} (b), i.e., $ \sum\limits_{i \in \mathcal{I}_{d,p}} \mu_{p}^{-s} f_{i}^{2} \ge c_{0}$ and  $ \sum\limits_{i \in \mathcal{I}_{d,p+1}} \mu_{p+1}^{-s} f_{i}^{2} \ge c_{0}$, to obtain the lower bound. Combining \eqref{proof useful 1} and \eqref{proof useful 2}, we finish the proof.

\end{proof}

\begin{lemma}\label{lemma calculation m2}
    Consider $\mathcal{X} = \mathbb{S}^{d}$ and the marginal distribution $\mu$ to be the uniform distribution. Let $k = \{ k_{d}\}_{d=1}^{\infty}$ be a sequence of inner product kernels on the sphere satisfying Assumption \ref{assumption kernel} and \ref{assumption inner product kernel}. Further suppose that Assumption \ref{assumption source condition} holds for some $ s > 0$. Define $ \tilde{s} = \min\{s,2\} $. By choosing $\lambda = d^{-l}$ for some $0 < l < \gamma$, we have:
    if $ p \le l \le p+1$ for some $ p \in \{0,1,2, \cdots\}$, we have
    \begin{equation}\label{inner product M2}
        \mathcal{M}_{2}(\lambda) = \Theta\left( \lambda^{2} d^{(2-\tilde{s})p} + d^{-(p+1)\tilde{s}}\right).
    \end{equation}
    The notation $\Theta$ involves constants only depending on $s, p, c_{0}$ and $ R_{\gamma}$, where $ c_{0}$ and $R_{\gamma} $ are the constants from Assumption \ref{assumption source condition}.
\end{lemma}

\begin{proof}
    When $ 0 < s \le 2$, $  \mathcal{M}_{2}(\lambda) $ can be viewed as $ \lambda^{2} \mathcal{M}_{\alpha, \beta}(\lambda) $ in Lemma \ref{lemma useful m} with $\alpha = 0, \beta = 2$. The conditions \eqref{useful condition 1} are satisfied and Lemma \ref{lemma useful m} shows that 
    \begin{equation}\label{proof m2 1}
        \mathcal{M}_{2}(\lambda) = \Theta\left( \lambda^{2} d^{(2-s)p} + d^{-(p+1)s}\right).
    \end{equation}

    When $ s > 2$, without loss of generality, we can assume $ \lambda_{i} \le 1, \forall i $ and $ \lambda \le 1$. Recall we have proved in Lemma \ref{lemma inner eigen} that $ \mu_{0}$ remains as a constant as $d \to \infty$. On the one hand, Assumption \ref{assumption source condition} (b) also implies $f_{1}^{2} \ge c_{0} $ without loss of generality, which further implies
    \begin{align}\label{proof m2 2}
        \mathcal{M}_{2}(\lambda) &= \lambda^{2} \sum\limits_{i=1}^{\infty} \frac{f_{i}^{2}}{\left(\lambda_{i} + \lambda\right)^{2}} \ge \frac{1}{4} \lambda^{2} \sum\limits_{i=1}^{\infty} f_{i}^{2} \ge \frac{1}{4} \lambda^{2} c_{0}.
    \end{align}
    On the other hand, since $s>2$, we have 
    \begin{align}\label{proof m2 3}
        \mathcal{M}_{2}(\lambda) &= \lambda^{2} \sum\limits_{k=0}^{\infty} \left( \frac{\mu_{k}^{s}}{\left( \mu_{k} + \lambda \right)^{2}} \sum\limits_{i \in \mathcal{I}_{d,k}} \mu_{k}^{-s} f_{i}^{2} \right) \notag \\
        &\le \lambda^{2} \left( \sup\limits_{k \ge 0} \frac{\mu_{k}^{s}}{\left( \mu_{k} + \lambda\right)^{2}} \cdot \sum\limits_{k=0}^{\infty} \sum\limits_{i \in \mathcal{I}_{d,k}} \mu_{k}^{-s} f_{i}^{2}  \right) \notag \\
        &\le \lambda^{2} \cdot \sup\limits_{k \ge 0} \frac{\mu_{k}^{s}}{\left( \mu_{k} + \lambda\right)^{2}} \cdot R_{\gamma}^{2} \notag \\
        &\le \lambda^{2} R_{\gamma}^{2}.
    \end{align}
    Further note that since $ s > 2$ and $ p \le l \le p+1$, \eqref{proof m2 2} and \eqref{proof m2 3} implies
    \begin{displaymath}
        \mathcal{M}_{2}(\lambda) = \Theta \left( \lambda^{2} \right) = \Theta \left( \lambda^{2} d^{(2-\tilde{s})p} + d^{-(p+1)\tilde{s}} \right).
    \end{displaymath}
    We finish the proof.

\end{proof}

The following lemma applies for those $ s \ge 1$, which gives an upper bound of $\mathcal{M}_{1}(\lambda)$.
\begin{lemma}\label{lemma calculation m1 s ge 1}
    Consider $\mathcal{X} = \mathbb{S}^{d}$ and the marginal distribution $\mu$ to be the uniform distribution. Let $k = \{ k_{d}\}_{d=1}^{\infty}$ be a sequence of inner product kernels on the sphere satisfying Assumption \ref{assumption kernel} and \ref{assumption inner product kernel}. Further suppose that Assumption \ref{assumption source condition} holds for some $ s \ge 1$. Define $ \tilde{s} = \min\{s,2\} $. By choosing $\lambda = d^{-l}$ for some $0<l<\gamma$, we have:
    if $ p \le l \le p+1$ for some $ p \in \{0,1,2\cdots\}$, we have
    \begin{equation}\label{inner product M1 s ge 1}
        \mathcal{M}_{1}(\lambda) = O\left( \lambda^{\frac{1}{2}} d^{\frac{(2-\tilde{s})p}{2}} + d^{-\frac{(\tilde{s}-1)(p+1)}{2}}\right).
    \end{equation}
    The notation $O$ involves constants only depending on $s, \kappa, p$ and $ R_{\gamma}$, where $R_{\gamma} $ is the constant from Assumption \ref{assumption source condition}.
\end{lemma}
\begin{proof}
    First, Cauchy-Schwarz inequality shows that
\begin{align}\label{m1 to q1 n1}
    \mathcal{M}_{1}(\lambda)^{2} &= \operatorname*{ess~sup}_{\boldsymbol{x} \in \mathcal{X}} \left|\sum\limits_{i=1}^{\infty} \left( \frac{\lambda}{\lambda_{i} + \lambda} f_{i} e_{i}(\boldsymbol{x}) \right) \right|^{2} \notag \\
    &\le \left( \sum\limits_{i=1}^{\infty}  \frac{\lambda^{2} \lambda_{i}^{-1}}{\lambda_{i} + \lambda} f_{i}^{2} \right) \cdot \operatorname*{ess~sup}_{\boldsymbol{x} \in \mathcal{X}} \sum\limits_{i=1}^{\infty} \left( \frac{\lambda_{i}}{\lambda_{i} + \lambda} e_{i}(\boldsymbol{x})^{2} \right) \notag \\
    &\le \left( \sum\limits_{i=1}^{\infty}  \frac{\lambda^{2} \lambda_{i}^{-1}}{\lambda_{i} + \lambda} f_{i}^{2} \right) \cdot \sum\limits_{i=1}^{\infty} \left( \frac{\lambda_{i}}{\lambda_{i} + \lambda} \right) \notag \\
    &:= \mathcal{Q}_{1}(\lambda) \cdot \mathcal{N}_{1}(\lambda),
\end{align}
where we use \eqref{assumption eigen - n1} in Assumption \ref{assumption eigenfunction} for the second inequality.

When $ 1 \le s \le 2$, $\mathcal{Q}_{1}(\lambda)$ defined above can be viewed as $\lambda^{2} \mathcal{M}_{\alpha, \beta}(\lambda) $ in Lemma \ref{lemma useful m} with $\alpha = -1, \beta = 1$. In addition, the conditions \eqref{useful condition 1} are satisfied, thus Lemma \ref{lemma useful m} shows that 
\begin{equation}\label{proof q1 1}
    \mathcal{Q}_{1}(\lambda) = \Theta\left( \lambda^{2} d^{(2-s)p} + \lambda d^{-(s-1)(p+1)}\right).
\end{equation}

Since we assume the kernel to be bounded in Assumption \ref{assumption kernel}, we can assume $ \lambda_{i} \le 1, \forall i $ and $ \lambda \le 1$ without loss of generality. When $ s > 2$, on the one hand, Assumption \ref{assumption source condition} (b) also implies
    \begin{align}\label{proof q1 2}
        \mathcal{Q}_{1}(\lambda) &= \lambda^{2} \sum\limits_{i=1}^{\infty} \frac{f_{i}^{2}}{\left(\lambda_{i} + \lambda\right)\lambda_{i}} \ge \frac{1}{2} \lambda^{2} \sum\limits_{i=1}^{\infty} f_{i}^{2} \ge \frac{1}{2} \lambda^{2} c_{0}.
    \end{align}
    On the other hand, since $s>2$, we have 
    \begin{align}\label{proof q1 3}
        \mathcal{Q}_{1}(\lambda) &= \lambda^{2} \sum\limits_{k=0}^{\infty} \left( \frac{\mu_{k}^{s-1}}{ \mu_{k} + \lambda } \sum\limits_{i \in \mathcal{I}_{d,k}} \mu_{k}^{-s} f_{i}^{2} \right) \notag \\
        &\le \lambda^{2} \left( \sup\limits_{k \ge 0} \frac{\mu_{k}^{s-1}}{ \mu_{k} + \lambda } \cdot \sum\limits_{k=0}^{\infty} \sum\limits_{i \in \mathcal{I}_{d,k}} \mu_{k}^{-s} f_{i}^{2}  \right) \notag \\
        &\le \lambda^{2} \cdot \frac{\mu_{k}^{s-1}}{ \mu_{k} + \lambda } \cdot R_{\gamma}^{2} \notag \\
        &\le \lambda^{2} R_{\gamma}^{2}.
    \end{align}
    Further note that since $ s > 2$ and $ p \le l \le p+1$, \eqref{proof q1 2} and \eqref{proof q1 3} implies
    \begin{displaymath}
        \mathcal{Q}_{1}(\lambda) = \Theta \left( \lambda^{2} \right) = \Theta \left( \lambda^{2} d^{(2-\tilde{s})p} + \lambda d^{-(\tilde{s}-1)(p+1)} \right).
    \end{displaymath}
    Therefore, we have $ \mathcal{Q}_{1}(\lambda) = \Theta \left( \lambda^{2} d^{(2-\tilde{s})p} + \lambda d^{-(\tilde{s}-1)(p+1)} \right)$ for any $ s > 0$.
    
   Further recalling that Lemma \ref{lemma calculation n1 n2} proves $ \mathcal{N}_{1}(\lambda) = \Theta\left( \lambda^{-1} \right)$, use \eqref{m1 to q1 n1} and we finish the proof.

\end{proof}

When $ 0 < s < 1$, the following lemma gives an upper bound of $ \|f_{\lambda}\|_{L^{\infty}}$.
\begin{lemma}\label{lemma calculation flambda s le 1}
    Consider $\mathcal{X} = \mathbb{S}^{d}$ and the marginal distribution $\mu$ to be the uniform distribution. Let $k = \{ k_{d}\}_{d=1}^{\infty}$ be a sequence of inner product kernels on the sphere satisfying Assumption \ref{assumption kernel} and \ref{assumption inner product kernel}. Further suppose that Assumption \ref{assumption source condition} holds for some $ 0<s<1$. Recall the definition of $f_{\lambda}$ in \eqref{def f lambda}. By choosing $\lambda = d^{-l}$ for some $0<l<\gamma$, we have: if $ p \le l \le p+1$ for some $ p \in \{0,1,2\cdots\}$, we have
    \begin{equation}\label{inner product flambda s le 1}
        \left\| f_{\lambda} \right\|_{L^{\infty}} = O\left( d^{\frac{(1-s)p}{2}} + \lambda^{-1} d^{-\frac{(1+s)(p+1)}{2}}\right).
    \end{equation}
    The notation $O$ involves constants only depending on $s, \kappa, p,$ and $ R_{\gamma}$, where $R_{\gamma} $ is the constant from Assumption \ref{assumption source condition}.
\end{lemma}
\begin{proof}
    We need the following fact: For any $f \in \mathcal{H}$, since Assumption \ref{assumption kernel} holds, we have
    \begin{align}
        \| f_{\lambda}\|_{L^{\infty}} &\le \sup\limits_{\boldsymbol{x} \in \mathcal{X}} |f(\boldsymbol{x})| = \sup\limits_{\boldsymbol{x} \in \mathcal{X}} \left\langle f(\cdot), k(\boldsymbol{x},\cdot) \right\rangle_{\mathcal{H}} \notag \\
        &\le \sup\limits_{\boldsymbol{x} \in \mathcal{X}} \| k(\boldsymbol{x},\cdot)\|_{\mathcal{H}} \cdot \|f\|_{\mathcal{H}} \notag \\
        &\le \kappa^{2} \|f\|_{\mathcal{H}}. \notag
    \end{align}
    Therefore, by definition and notice that $ f_{\lambda} \in \mathcal{H}$, we have
    \begin{equation}
        \| f_{\lambda}\|_{L^{\infty}}^{2} \le \kappa^{4} \| f_{\lambda}\|_{\mathcal{H}}^{2} = \kappa^{4} \sum\limits_{i=1}^{\infty} \frac{\lambda_{i}}{
        \left(\lambda_{i} + \lambda\right)^{2}} f_{i}^{2} := \mathcal{Q}_{2}(\lambda).
    \end{equation}
    Since $ 0 < s <1$, $\mathcal{Q}_{2}(\lambda)$ can be viewed as $ \kappa^{4} \mathcal{M}_{\alpha, \beta}(\lambda) $ in Lemma \ref{lemma useful m} with $\alpha = 1, \beta = 2$ and the conditions \eqref{useful condition 1} are satisfied. Thus Lemma \ref{lemma useful m} shows that 
\begin{displaymath}
    \mathcal{Q}_{2}(\lambda) = \Theta\left( \lambda^{2} d^{(1-s)p} + \lambda^{-s} d^{-(1+s)(p+1)}\right).
\end{displaymath}
Taking the square root, we finish the proof.
\end{proof}

\subsection{Proof of Theorem \ref{theorem inner s ge 1}}\label{section proof of inner ge 1}
In the last subsection, we have calculated the exact convergence rates of $\mathcal{N}_{1}(\lambda), \mathcal{N}_{2}(\lambda), \mathcal{M}_{1}(\lambda), \mathcal{M}_{2}(\lambda)$ when $ \lambda = d^{-l}$ for some $ 0 < l < \gamma$. Note that we have proved in Lemma \ref{lemma inner assumption holds} that Assumption \ref{assumption eigenfunction} naturally holds for inner product kernel on the sphere. Now we are ready to apply Theorem \ref{main theorem} to prove Theorem \ref{theorem inner s ge 1}. The proof mainly consists of 3 steps: 
\begin{itemize}
    \item[(1)] For specific range of $\gamma > 0$, we use Lemma \ref{lemma calculation n1 n2} and Lemma \ref{lemma calculation m2} to derive the scale of $ \lambda_{\text{balance}}$ or $ l_{\text{balance}}$ such that $  \mathcal{N}_{2}(\lambda)/n$ and $\mathcal{M}_{2}(\lambda)$ are balanced.
    
    \item[(2)] We check that the conditions \eqref{approximation conditions} required in Theorem \ref{main theorem} are satisfied for $\lambda \gtrsim \lambda_{\text{balance}}$ (or $\lambda = d^{-l}, l \le l_{\text{balance}}$).
    
    \item[(3)] Using the monotonicity of $\mathbf{Var}(\lambda)$ with respect to $\lambda$, we demonstrate that $ \lambda = \lambda_{\text{balance}} $ is the best choice of regularization parameter, i.e., the generalization error of KRR estimator is the smallest when $ \lambda = \lambda_{\text{balance}} $. That is to say, the convergence rate of the generalization error can not be faster than the rate when choosing $ \lambda = \lambda_{\text{balance}} $.
\end{itemize}
Note that we expect $\mathbf{Bias}^{2}(\lambda) = \Theta_{\mathbb{P}}\left(\mathcal{M}_{2}(\lambda) \right)$ and $ \mathbf{Var}(\lambda) = \Theta_{\mathbb{P}}\left(\mathcal{N}_{2}(\lambda)/n \right)$. Step 1 actually indicates the regularization such that the bias and variance are balanced. Together with Theorem \ref{theorem variance approximation} and \ref{theorem bias approximation}, Step 2 further verifies that $\mathbf{Bias}^{2}(\lambda) = \Theta_{\mathbb{P}}\left(\mathcal{M}_{2}(\lambda) \right)$ and $ \mathbf{Var}(\lambda) = \Theta_{\mathbb{P}}\left(\mathcal{N}_{2}(\lambda)/n \right)$ indeed hold for those $ \lambda \gtrsim \lambda_{\text{balance}} $. Thus they are indeed balanced under the choice of $\lambda = \lambda_{\text{balance}} $.\\

\textit{Final proof of Theorem \ref{theorem inner s ge 1}.} In the following of the proof, we omit the dependence of constants on $s, \sigma, \gamma, c_{0}, \kappa, c_{1} $ and $ c_{2}$.
$ $\\
\textbf{Step 1}:
Note that we assume $ s \ge 1$ in this theorem and $ \lambda = d^{-l}, 0<l<\gamma$. For specific range of $\gamma$, we discuss the range of $ l_{\text{balance}}$. Recall that we define $ \tilde{s} = \min\{s,2\}$.
\begin{itemize}[leftmargin = 18pt]
    \item When $ l \in ( p,  p + \frac{1}{2}]$ for some integer $ p \ge 0$, Lemma \ref{lemma calculation n1 n2} and Lemma \ref{lemma calculation m2} show that
    \begin{displaymath}
        \frac{\mathcal{N}_{2}(\lambda)}{n} \asymp d^{p-\gamma};~~ \mathcal{M}_{2}(\lambda) \asymp d^{-2l+(2-\tilde{s})p},
    \end{displaymath}
    thus we have
    \begin{equation}
        l_{\text{balance}} = \frac{\gamma + p -p\tilde{s}}{2}.
    \end{equation}
    Further, letting $l_{\text{balance}} = \frac{\gamma + p -p\tilde{s}}{2} \in ( p,  p + \frac{1}{2}] $, we have 
    \begin{equation}
        \gamma \in \left( p+p\tilde{s}, p+p\tilde{s}+1 \right].
    \end{equation}

    \item When $ l \in ( p+ \frac{1}{2},  p + \frac{\tilde{s}}{2}]$, Lemma \ref{lemma calculation n1 n2} and Lemma \ref{lemma calculation m2} show that
    \begin{displaymath}
        \frac{\mathcal{N}_{2}(\lambda)}{n} \asymp d^{2l-p-1-\gamma};~~ \mathcal{M}_{2}(\lambda) \asymp d^{-2l+(2-\tilde{s})p},
    \end{displaymath}
    thus we have
    \begin{equation}
        l_{\text{balance}} = \frac{\gamma + 3p -p\tilde{s}+1}{4}.
    \end{equation}
    Further, letting $l_{\text{balance}} = \frac{\gamma + 3p -p\tilde{s}+1}{4} \in ( p+ \frac{1}{2},  p + \frac{\tilde{s}}{2}] $, we have 
    \begin{equation}
        \gamma \in \left( p+p\tilde{s}+1, p+p\tilde{s}+2\tilde{s}-1 \right].
    \end{equation}

    \item When $ l \in (p+ \frac{\tilde{s}}{2},  p + 1]$, Lemma \ref{lemma calculation n1 n2} and Lemma \ref{lemma calculation m2} show that
    \begin{displaymath}
        \frac{\mathcal{N}_{2}(\lambda)}{n} \asymp d^{2l-p-1-\gamma};~~ \mathcal{M}_{2}(\lambda) \asymp d^{-(p+1)\tilde{s}},
    \end{displaymath}
    thus we have
    \begin{equation}
        l_{\text{balance}} = \frac{\gamma + (p+1)(1-\tilde{s})}{2}.
    \end{equation}
    Further, letting $l_{\text{balance}} \in ( p+ \frac{\tilde{s}}{2},  p + 1] $, we have 
    \begin{equation}
        \gamma \in \left( p+p\tilde{s}+2\tilde{s}-1, (p+1) + (p+1)\tilde{s} \right].
    \end{equation}
\end{itemize}

\textbf{Step 2}:
In order to apply Theorem \ref{theorem variance approximation} and Theorem \ref{theorem bias approximation} so that we know the exact convergence rates of $ \mathbf{Var}(\lambda_{\text{balance}})$ and $\mathbf{Bias}^{2}(\lambda_{\text{balance}})$, we first check the approximation conditions \eqref{var conditions} and \eqref{bias conditions}, or equivalently conditions \eqref{approximation conditions}, hold for $ l = l_{\text{balance}}$. Recall that we have calculated the convergence rates of $\mathcal{N}_{1}(\lambda)$ and $\mathcal{M}_{1}(\lambda)$ in Lemma \ref{lemma calculation n1 n2} and Lemma \ref{lemma calculation m1 s ge 1}.

\begin{itemize}[leftmargin = 18pt]
    \item When $\gamma \in \left( p+p\tilde{s}, p+p\tilde{s}+1 \right] $: recall that $ l_{\text{balance}} = \frac{\gamma+p-p\tilde{s}}{2} \in (p, p+\frac{1}{2}]$.
    
    $\left( \text{\lowercase\expandafter{\romannumeral1}} \right)~$The first condition in \eqref{approximation conditions} is equivalent to
    \begin{displaymath}
        \frac{\gamma+p-p\tilde{s}}{2} < \gamma \iff \gamma > p-p\tilde{s},
    \end{displaymath}
    which naturally holds for all $\gamma \in \left( p+p\tilde{s}, p+p\tilde{s}+1 \right]  $.\\
    $ $\\
    $\left( \text{\lowercase\expandafter{\romannumeral2}} \right)~$The second condition in \eqref{approximation conditions} is equivalent to
    \begin{displaymath}
        d^{-\gamma} \cdot d^{\gamma+p-p\tilde{s}} \cdot \gamma \ln{d} \ll d^{p} \iff p-p\tilde{s} < p,
    \end{displaymath}
    which naturally holds for all $\gamma \in \left( p+p\tilde{s}, p+p\tilde{s}+1 \right]  $ and $ p \neq 0$. When $ p = 0$, we actually need to choose $ \lambda_{\text{balance}} = d^{-l_{\text{balance}}} \cdot \ln{d}$ and the second condition will hold.\\
    $ $\\
    $\left( \text{\lowercase\expandafter{\romannumeral3}} \right)~$The third condition in \eqref{approximation conditions} is equivalent to
    \begin{align}
        d^{-\gamma} \cdot d^{\frac{\gamma+p-p\tilde{s}}{4}}  \cdot \left[ d^{-\frac{\gamma+p-p\tilde{s}}{4}+\frac{(2-\tilde{s})p}{2}} + d^{-\frac{(\tilde{s}-1)(p+1)}{2}} \right] \ll d^{-\frac{1}{2}(\gamma+p-p\tilde{s}) + \frac{(2-\tilde{s})p}{2}} \notag 
    \end{align}
    $ \iff $
    \begin{displaymath}
        \gamma > p-p\tilde{s};~~ \gamma > p-3p\tilde{s}-2\tilde{s}+2,
    \end{displaymath}
    which naturally holds for all $\gamma \in\left( p+p\tilde{s}, p+p\tilde{s}+1 \right] $ and $ p \neq 0$. In addition, one can also check that the third condition in \eqref{approximation conditions} holds when $ p = 0$ and $ \lambda_{\text{balance}} = d^{-l_{\text{balance}}} \cdot \ln{d}$.

    \item When $\gamma \in ( p+p\tilde{s}+1, p+p\tilde{s}+2\tilde{s}-1 ] $: recall that $ l_{\text{balance}} = \frac{\gamma + 3p -p\tilde{s}+1}{4} \in (p+\frac{1}{2}, p+\frac{\tilde{s}}{2}]$.
    
    $\left( \text{\lowercase\expandafter{\romannumeral1}} \right)~$The first condition in \eqref{approximation conditions} is equivalent to
    \begin{displaymath}
        \frac{\gamma + 3p -p\tilde{s}+1}{4} < \gamma \iff \gamma > p - \frac{p\tilde{s}}{3}+\frac{1}{3},
    \end{displaymath}
    which naturally holds for all $\gamma \in ( p+p\tilde{s}+1, p+p\tilde{s}+2\tilde{s}-1 ]  $.\\
    $ $\\
    $\left( \text{\lowercase\expandafter{\romannumeral2}} \right)~$The second condition in \eqref{approximation conditions} is equivalent to
    \begin{displaymath}
        d^{-\gamma} \cdot d^{ \frac{\gamma + 3p -p\tilde{s}+1}{2}} \cdot \gamma \ln{d} \ll d^{ \frac{\gamma + 3p -p\tilde{s}+1}{2}-p-1} \iff \gamma > p+1,
    \end{displaymath}
    which naturally holds for all $\gamma \in ( p+p\tilde{s}+1, p+p\tilde{s}+2\tilde{s}-1 ]  $.\\
    $ $\\
    $\left( \text{\lowercase\expandafter{\romannumeral3}} \right)~$The third condition in \eqref{approximation conditions} is equivalent to
    \begin{align}
        d^{-\gamma} \cdot d^{\frac{\gamma+p-p\tilde{s}}{8}}  \cdot \left[ d^{-\frac{\gamma+p-p\tilde{s}}{8}+\frac{(2-\tilde{s})p}{2}} + d^{-\frac{(\tilde{s}-1)(p+1)}{2}} \right] \ll d^{-\frac{1}{4}(\gamma+p-p\tilde{s}) + \frac{(2-\tilde{s})p}{2}} \notag 
    \end{align}
    $ \iff $
    \begin{displaymath}
        \gamma > p-\frac{p\tilde{s}}{3}+\frac{1}{3};~~ \gamma > p-\frac{3p\tilde{s}}{5}-\frac{4\tilde{s}}{5}+\frac{7}{5},
    \end{displaymath}
    which naturally holds for all $\gamma \in ( p+p\tilde{s}+1, p+p\tilde{s}+2\tilde{s}-1 ] $.

    

    \item When $\gamma \in ( p+p\tilde{s}+2\tilde{s}-1, (p+1)+(p+1)\tilde{s} ] $: recall that $ l_{\text{balance}} = \frac{\gamma+(p+1)(1-\tilde{s})}{2} \in (p+\frac{\tilde{s}}{2}, p+1$.
    
    $\left( \text{\lowercase\expandafter{\romannumeral1}} \right)~$The first condition in \eqref{approximation conditions} is equivalent to
    \begin{displaymath}
        \frac{\gamma+(p+1)(1-\tilde{s})}{2} < \gamma \iff \gamma > p-p\tilde{s}+1-\tilde{s},
    \end{displaymath}
    which naturally holds for all $\gamma \in ( p+p\tilde{s}+2\tilde{s}-1, (p+1)+(p+1)\tilde{s} ]  $.\\
    $ $\\
    $\left( \text{\lowercase\expandafter{\romannumeral2}} \right)~$The second condition in \eqref{approximation conditions} is equivalent to
    \begin{displaymath}
        d^{-\gamma} \cdot d^{ \gamma + (p+1)(1-\tilde{s})} \cdot \gamma \ln{d} \ll d^{\gamma+(p+1)(1-\tilde{s})-p-1} \iff \gamma > p+1,
    \end{displaymath}
    which naturally holds for all $\gamma \in ( p+p\tilde{s}+2\tilde{s}-1, (p+1)+(p+1)\tilde{s} ]  $.\\
    $ $\\
    $\left( \text{\lowercase\expandafter{\romannumeral3}} \right)~$The third condition in \eqref{approximation conditions} is equivalent to
    \begin{align}
        d^{-\gamma} \cdot d^{\frac{\gamma+(p+1)(1-\tilde{s})}{4} }  \cdot \left[ d^{-\frac{\gamma+(p+1)(1-\tilde{s})}{2}+\frac{(2-\tilde{s})p}{2}} + d^{-\frac{(\tilde{s}-1)(p+1)}{2}} \right] \ll d^{-\frac{(p+1)\tilde{s}}{2}} \notag 
    \end{align}
    $ \iff $
    \begin{displaymath}
        \gamma > p+\frac{\tilde{s}}{2};~~ \gamma > p-\frac{p\tilde{s}}{3}+\frac{\tilde{s}}{3}+\frac{1}{3},
    \end{displaymath}
    which naturally holds for all $\gamma \in ( p+p\tilde{s}+2\tilde{s}-1, (p+1)+(p+1)\tilde{s} ] $.
    
\end{itemize}
Up to now, we have verified conditions \eqref{approximation conditions} for $ l = l_{\text{balance}}$. Furthermore, simple calculation shows that the order of 
\begin{equation}\label{mono of l}
   \frac{\mathcal{N}_{1}(\lambda)}{n} ;~~ n^{-1} \mathcal{N}_{1}(\lambda)^{2}  \mathcal{N}_{2}(\lambda)^{-1}; ~~ n^{-1} \mathcal{N}_{1}(\lambda)^{\frac{1}{2}} \mathcal{M}_{1}(\lambda) \mathcal{M}_{2}(\lambda)^{-\frac{1}{2}}
\end{equation}
are all non-decreasing with respect to $l$, where we choose $\lambda = d^{-l}$. Therefore, the above results indicate that conditions \eqref{approximation conditions} holds for all $ l \le l_{\text{balance}}$.\\
$ $\\
\textbf{Step 3}:
In step 2, on the one hand, we prove that by choosing $\lambda=\lambda_{\text{balance}} $, 
\begin{equation}\label{best regularization s ge 1-1}
    \mathbb{E}\left[\left\|\hat{f}_{\lambda_{\text{balance}}}-f_\rho^*\right\|_{L^2}^2 \;\Big|\; \boldsymbol{X} \right] = \Theta_{\mathbb{P}}\left( \frac{\sigma^{2} \mathcal{N}_{2}(\lambda_{\text{balance}})}{n} + \mathcal{M}_{2}(\lambda_{\text{balance}}) \right).
\end{equation}
On the other hand, we also prove that by choosing $\lambda \gtrsim \lambda_{\text{balance}} $,
\begin{equation}\label{best regularization s ge 1-2}
    \mathbb{E}\left[\left\|\hat{f}_{\lambda_{\text{balance}}}-f_\rho^*\right\|_{L^2}^2 \;\Big|\; \boldsymbol{X} \right] = \Omega_{\mathbb{P}}\left( \frac{\sigma^{2} \mathcal{N}_{2}(\lambda_{\text{balance}})}{n} + \mathcal{M}_{2}(\lambda_{\text{balance}}) \right).
\end{equation}

In the following, we handle those $\lambda \lesssim \lambda_{\text{balance}} $. Recall that we have shown in the proof of Lemma \ref{lemma var trans} that 
\begin{displaymath}
    \mathbf{Var}(\lambda) = \frac{\sigma^2}{n^2} \int_{\mathcal{X}} \mathbb{K}(\boldsymbol{x},\boldsymbol{X})(\boldsymbol{K}+\lambda)^{-2} \mathbb{K}(\boldsymbol{X},\boldsymbol{x}) ~ \mathrm{d} \mu(\boldsymbol{x}).
\end{displaymath}
One simple but critical observation from \cite{li2023_SaturationEffect} is that
\begin{align}
  (\boldsymbol{K}+\lambda_{1})^{-2} \succeq  (\boldsymbol{K}+\lambda_2)^{-2},~~ \text{if}~~ \lambda_1 \le \lambda_2, \notag 
\end{align}
where $\succeq$ represents the partial order of positive semi-definite matrices, and thus for  $\lambda_{1} \le \lambda_{2} > 0$, we have 
\begin{displaymath}
    \mathbf{Var}(\lambda_{1}) \ge \mathbf{Var}(\lambda_{2}).
\end{displaymath}
Also note that by the definition of $l_{\text{balance}}$, we actually have 
\begin{equation}
    \mathbb{E}\left[\left\|\hat{f}_{\lambda_{\text{balance}}}-f_\rho^*\right\|_{L^2}^2 \;\Big|\; \boldsymbol{X} \right] = \Theta_{\mathbb{P}}\left( \mathbf{Var}(\lambda_{\text{balance}}) \right).
\end{equation}
Therefore, for those $\lambda \lesssim \lambda_{\text{balance}} $, we have
\begin{align}\label{best regularization s ge 1-3}
    \mathbb{E}\left[\left\|\hat{f}_{\lambda}-f_\rho^*\right\|_{L^2}^2 \;\Big|\; \boldsymbol{X} \right] \ge \mathbf{Var}(\lambda) \ge \mathbf{Var}(\lambda_{\text{balance}}) \asymp \mathbb{E}\left[\left\|\hat{f}_{\lambda_{\text{balance}}}-f_\rho^*\right\|_{L^2}^2 \;\Big|\; \boldsymbol{X} \right].
\end{align}
To sum up, \eqref{best regularization s ge 1-1}, \eqref{best regularization s ge 1-2} and \eqref{best regularization s ge 1-3} show that by choosing $ \lambda = \lambda_{\text{balance}}$ as in step 1, we obtain the convergence rates of KRR estimator under the best regularization. Using Lemma \ref{lemma calculation m2} to calculate the rate of $ \mathcal{M}_{2}(\lambda_{\text{balance}})$, we finish the proof.\\

$\hfill\blacksquare$

\subsection{Proof of Theorem \ref{theorem inner s le 1}}\label{section proof of inner le 1}
Recall that in the proof of Theorem \ref{theorem bias approximation}, we use $ \mathcal{M}_{1}(\lambda)$ to bound $ \left\| f_{\lambda} - f_{\rho}^{*} \right\|_{L^{\infty}}$ and show that $ \left\| \tilde{f}_{\lambda} - f_{\lambda}\right\|_{L^{2}} = o_{\mathbb{P}}(\mathcal{M}_{2}(\lambda)^{\frac{1}{2}})$. Unfortunately, the calculation of $\mathcal{M}_{1}(\lambda) $ in Lemma \ref{lemma calculation m1 s ge 1} only holds for $ s \ge 1$ and it could be infinite when $s < 1$. Extensive literature then assume $ \left\| f_{\rho}^{*} \right\|_{L^{\infty}} $ to be bounded and use $\left\| f_{\lambda}\right\|_{L^{\infty}} + \left\| f_{\rho}^{*} \right\|_{L^{\infty}} $ to bound $\left\| f_{\lambda} - f_{\rho}^{*} \right\|_{L^{\infty}}$. When the dimension $d$ is fixed, \cite{zhang2023optimality} first use a truncation method together with the $ L^{q}$-embedding property of $[\mathcal{H}]^{s} $ to remove the boundedness assumption when $s<1$. We will see in the proof of Theorem \ref{theorem inner s le 1} that this technique still works in the large-dimensional setting.

To be specific, when we further assume $ f_{\rho}^{*} \in [\mathcal{H}]^{s}$, we have the following Theorem which is a refined version of Lemma \ref{lemma bias appr term}.

\begin{lemma}\label{lemma bias appr term s le 1}
Suppose that Assumption \ref{assumption kernel}, \ref{assumption noise} and \ref{assumption eigenfunction} hold. Further suppose that Assumption \ref{assumption source condition} holds for some $ 0 < s < 1$. If the following approximation conditions hold for some $\lambda = \lambda(d,n) \to 0$:
    \begin{align}\label{bias conditions s le 1}
    \frac{\mathcal{N}_{1}(\lambda)}{n} \ln{n} = o(1); ~~ n^{-1} \mathcal{N}_{1}(\lambda)^{\frac{1}{2}} \left\| f_{\lambda} \right\|_{L^{\infty}}  = o\left(\mathcal{M}_{2}(\lambda)^{\frac{1}{2}}\right);
\end{align}
and there exists $\epsilon > 0$, such that  
\begin{equation}\label{bias conditions s le 1-2} 
    n^{-1} \mathcal{N}_{1}(\lambda)^{\frac{1}{2}} n^{\frac{1-s}{2}+\epsilon} = o\left(\mathcal{M}_{2}(\lambda)^{\frac{1}{2}}\right),
\end{equation}
then we have
    \begin{equation}
        \left\| \tilde{f}_{\lambda} - f_{\lambda}\right\|_{L^{2}} = o_{\mathbb{P}}(\mathcal{M}_{2}(\lambda)^{\frac{1}{2}}),
    \end{equation}
    where the notation $ o_{\mathbb{P}} $ involves constants only depending on $ s$ and $\kappa$.
\end{lemma}

\begin{proof}
    Recall the decomposition \eqref{appr proof-0} in the proof of Lemma \ref{lemma bias appr term}. The first two terms in \eqref{appr proof-0} can be handled without any difference. Our goal here is to prove the following equation and we will finish the proof:
    \begin{equation}\label{bias appr s le 1 goal}
    \left\| T_{\lambda}^{-\frac{1}{2}} \left( \tilde{g}_{\boldsymbol{Z}} - T_{\boldsymbol{X} \lambda} f_{\lambda} \right) \right\|_{\mathcal{H}} = o_{\mathbb{P}}\left(\mathcal{M}_{2}(\lambda)^{\frac{1}{2}}\right).
\end{equation}
Similar as \eqref{appr proof-1}, we rewrite the left hand side of \eqref{bias appr s le 1 goal} as 
\begin{align}\label{appr proof-1 s le 1}
    \left\| T_{\lambda}^{-\frac{1}{2}} \left( \tilde{g}_{\boldsymbol{Z}} - T_{\boldsymbol{X} \lambda} f_{\lambda} \right) \right\|_{\mathcal{H}} = \left\|T_\lambda^{-\frac{1}{2}}\left[\left(\tilde{g}_{\boldsymbol{Z}} - T_ {\boldsymbol{X}} f_\lambda \right)-\left(g-T f_\lambda\right)\right]\right\|_{\mathcal{H}}.
\end{align}
Denote $\xi_{i} = \xi(\boldsymbol{x}_{i}) =  T_{\lambda}^{-\frac{1}{2}}(K_{\boldsymbol{x}_{i}} f_{\rho}^{*}(\boldsymbol{x}_{i}) - T_{\boldsymbol{x}_{i}} f_{\lambda}) $. Further consider the subset $\Omega_{1} = \{\boldsymbol{x} \in \mathcal{X}: |f_{\rho}^{*}(\boldsymbol{x})| \le t \}$ and $\Omega_{2} = \mathcal{X} \backslash \Omega_{1}$, where $t$ will be chosen appropriately later. Decompose $\xi_{i}$ as $\xi_{i} I_{\boldsymbol{x}_{i} \in \Omega_{1} } +  \xi_{i} I_{\boldsymbol{x}_{i} \in \Omega_{2} }$ and we have the following decomposition of \eqref{appr proof-1 s le 1}:
\begin{align}\label{decomposition}
    \left\|\frac{1}{n} \sum_{i=1}^n \xi_i-\mathbb{E} \xi_x\right\|_\mathcal{H} &\le \left\|\frac{1}{n} \sum_{i=1}^n \xi_i I_{\boldsymbol{x}_{i} \in \Omega_{1}}-\mathbb{E} \xi_{\boldsymbol{x}} I_{\boldsymbol{x} \in \Omega_{1}} \right\|_\mathcal{H} + \| \frac{1}{n} \sum_{i=1}^n \xi_i I_{\boldsymbol{x}_{i} \in \Omega_{2}} \|_{_\mathcal{H}} + \| \mathbb{E} \xi_{\boldsymbol{x}} I_{\boldsymbol{x} \in \Omega_{2}} \|_{_\mathcal{H}} \notag \\
    &:= \text{\uppercase\expandafter{\romannumeral1}} + \text{\uppercase\expandafter{\romannumeral2}} + \text{\uppercase\expandafter{\romannumeral3}}.
\end{align}
Next we choose $t = n^{\frac{1-s}{2} + \epsilon_{t}}, q = \frac{2}{1-s}-\epsilon_{q} $ such that 
\begin{equation}\label{choose t q}
   \epsilon_{t} < \epsilon;~~\text{and}~~ \frac{1-s}{2} + \epsilon_{t} > 1 / \left( \frac{2}{1-s}-\epsilon_{q} \right),
\end{equation}
where $ \epsilon $ is given in \eqref{bias conditions s le 1-2}. Then we can bound the three terms in \eqref{decomposition} as follows:\\
$ $\\
$\left( \text{\lowercase\expandafter{\romannumeral1}}\right)~$For the first term in \eqref{decomposition}, denoted as $\text{\uppercase\expandafter{\romannumeral1}}$, notice that
\begin{align}
     \left\| \left(f_{\lambda} - f_{\rho}^{*}\right)I_{\boldsymbol{x}_{i} \in \Omega_{1}} \right\|_{L^{\infty}} \le \left\| f_{\lambda}\right\|_{L^{\infty}} + n^{\frac{1-s}{2}+\epsilon_{t}}.
\end{align}
Imitating the procedure $\left( \text{\lowercase\expandafter{\romannumeral3}}\right)$ in the proof of Lemma \ref{lemma bias appr term} and using \eqref{bias conditions s le 1}, \eqref{bias conditions s le 1-2}, we have
\begin{equation}\label{plug s le 1-1}
    \text{\uppercase\expandafter{\romannumeral1}} = o_{\mathbb{P}}\left( \mathcal{M}_{2}(\lambda)^{\frac{1}{2}} \right).
\end{equation}
$\left(\text{\lowercase\expandafter{\romannumeral2}}\right)~$ For the second term in \eqref{decomposition}, denoted as $\text{\uppercase\expandafter{\romannumeral2}}$. Since $ q = \frac{2}{1-s}-\epsilon_{q} < \frac{2}{1-s}$, Lemma \ref{integrability of Hs constants} shows that,
\begin{align}
      [\mathcal{H}]^{s} \hookrightarrow L^{q}(\mathcal{X}, \mu),
\end{align}
with embedding norm less than a constant $C_{s,\kappa}$. Then Assumption \ref{assumption source condition} (a) implies that there exists $0 < C_{q} < \infty$ only depending on $\gamma, s$ and $\kappa$ such that $\| f_{\rho}^{*} \|_{L^{q}(\mathcal{X},\mu)} \le C_{q}$. Using the Markov inequality, we have
\begin{displaymath}
       P(\boldsymbol{x} \in \Omega_{2}) = P\Big(|f_{\rho}^{*}(\boldsymbol{x})| > t \Big) \le \frac{\mathbb{E} |f_{\rho}^{*}(\boldsymbol{x})|^{q}}{t^{q}} \le \frac{(C_{q})^{q}}{t^{q}}.
\end{displaymath}
Further, since \eqref{choose t q} guarantees $ t^{q} \gg n$, we have 
\begin{align}\label{plug s le 1-2}
    \tau_{n} := P\left(\text{\uppercase\expandafter{\romannumeral2}} \ge \mathcal{M}_{2}(\lambda)^{\frac{1}{2}}\right) 
    &\le P\Big( ~\exists \boldsymbol{x}_{i} ~\text{s.t.}~ \boldsymbol{x}_{i} \in \Omega_{2}, \Big) = 1 - P\Big(\boldsymbol{x}_{i} \notin \Omega_{2}, \forall \boldsymbol{x}_{i},i=1,2,\cdots,n \Big) \notag \\
    &= 1 - P\Big(\boldsymbol{x} \notin \Omega_{2}\Big)^{n} \notag \\
    &= 1 - P\Big( |f_{\rho}^{*}(\boldsymbol{x})| \le t\Big)^{n} \notag \\
    & \le 1 - \Big( 1 - \frac{(C_q)^{q}}{t^{q}}\Big)^{n} \to 0.
\end{align}
$\left(\text{\lowercase\expandafter{\romannumeral3}}\right)~$ For the third term in \eqref{decomposition}, denoted as $\text{\uppercase\expandafter{\romannumeral3}}$. Since Lemma \ref{due embedding bound} implies that $\| T_{\lambda}^{-\frac{1}{2}} k(\boldsymbol{x},\cdot)\|_{\mathcal{H}} \le \mathcal{N}_{1}(\lambda)^{\frac{1}{2}}, \mu \text {-a.e. } \boldsymbol{x} \in \mathcal{X},$ so
\begin{align}\label{third term}
    \text{\uppercase\expandafter{\romannumeral3}} &\le \mathbb{E}\| \xi_{\boldsymbol{x}} I_{\boldsymbol{x} \in\Omega_{2}} \|_{\mathcal{H}} \le \mathbb{E}\Big[ \| T_{\lambda}^{-\frac{1}{2}} k(\boldsymbol{x},\cdot) \|_{\mathcal{H}} \cdot \big| \big(f_{\rho}^{*}-f_{\lambda}(\boldsymbol{x}) \big) I_{\boldsymbol{x} \in\Omega_{2}}\big| \Big] \notag \\
    &\le \mathcal{N}_{1}(\lambda)^{\frac{1}{2}} \mathbb{E} \big| \big(f_{\rho}^{*}-f_{\lambda}(\boldsymbol{x}) \big) I_{\boldsymbol{x} \in\Omega_{2}}\big| \notag \\
    &\le \mathcal{N}_{1}(\lambda)^{\frac{1}{2}} \left\| f_{\rho}^{*} - f_{\lambda}\right\|_{L^{2}}^{\frac{1}{2}} \cdot P\left( \boldsymbol{x} \in \Omega_{2} \right)^{\frac{1}{2}} \notag \\
    &\le \mathcal{N}_{1}(\lambda)^{\frac{1}{2}} \mathcal{M}_{2}(\lambda)^{\frac{1}{2}} t^{-\frac{q}{2}},
\end{align}
where we use Cauchy-Schwarz inequality for the third inequality and \eqref{lemma bias main term} for the forth inequality. Recalling that the choices of $t, q$ satisfy $ t^{-q} \ll n^{-1}$ and we have assumed $\mathcal{N}_{1}(\lambda) \ln{n} / n= o(1) $, we have 
\begin{equation}\label{plug s le 1-3}
    \text{\uppercase\expandafter{\romannumeral3}} = o\left( \mathcal{M}_{2}(\lambda)^{\frac{1}{2}} \right).
\end{equation}
Plugging \eqref{plug s le 1-1}, \eqref{plug s le 1-2} and \eqref{plug s le 1-3} into \eqref{decomposition}, we finish the proof.

\end{proof}

Based on Lemma \ref{lemma bias appr term s le 1} and Lemma \ref{lemma bias main term}, we have the following theorem about the exact rate of bias term when $ 0<s<1$.

\begin{theorem}\label{theorem bias approximation s le 1}
    Suppose that Assumption \ref{assumption kernel}, \ref{assumption noise} and \ref{assumption eigenfunction} hold. Further suppose that Assumption \ref{assumption source condition} holds for some $ 0 < s < 1$. If the following approximation conditions hold for some $\lambda = \lambda(d,n) \to 0$:
   \begin{align}\label{bias conditions s le 1 in theorem}
    \frac{\mathcal{N}_{1}(\lambda)}{n} \ln{n} = o(1); ~~ n^{-1} \mathcal{N}_{1}(\lambda)^{\frac{1}{2}} \left\| f_{\lambda} \right\|_{L^{\infty}}  = o\left(\mathcal{M}_{2}(\lambda)^{\frac{1}{2}}\right);
\end{align}
and there exists $\epsilon > 0$, such that  
\begin{equation}\label{bias conditions s le 1-2 in theorem} 
    n^{-1} \mathcal{N}_{1}(\lambda)^{\frac{1}{2}} n^{\frac{1-s}{2}+\epsilon} = o\left(\mathcal{M}_{2}(\lambda)^{\frac{1}{2}}\right),
\end{equation}
then we have
\begin{equation}\label{bias theorem proof condition 2 s le 1}
    \mathbf{Bias}^{2}(\lambda) = \Theta_{\mathbb{P}}\left(\mathcal{M}_{2}(\lambda) \right),
\end{equation}
where the notation $ \Theta_{\mathbb{P}} $ involves constants only depending on $ s$ and $\kappa$.
\end{theorem}
\begin{proof}
The triangle inequality implies that 
\begin{displaymath}
    \mathrm{\textbf{Bias}}(\lambda) = \left\| \tilde{f}_{\lambda} - f_{\rho}^{*}\right\|_{L^{2}} \ge \left\| f_{\lambda} - f_{\rho}^{*}\right\|_{L^{2}} - \left\| \tilde{f}_{\lambda} - f_{\lambda}\right\|_{L^{2}},
\end{displaymath}
When $\lambda = \lambda(d,n)$ satisfies \eqref{bias conditions s le 1 in theorem} and \eqref{bias conditions s le 1-2 in theorem}, Lemma \ref{lemma bias main term} and Lemma \ref{lemma bias appr term s le 1} prove that 
\begin{displaymath}
    \left\| f_{\lambda} - f_{\rho}^{*}\right\|_{L^{2}} = \mathcal{M}_{2}(\lambda)^{\frac{1}{2}};~~ \left\| \tilde{f}_{\lambda} - f_{\lambda}\right\|_{L^{2}} = o_{\mathbb{P}}(\mathcal{M}_{2}(\lambda)^{\frac{1}{2}}),
\end{displaymath}
which directly prove \eqref{bias theorem proof condition 2 s le 1}.
\end{proof}

Now we are ready to prove Theorem \ref{theorem inner s le 1}. Since we do not claim that the regularization choice in Theorem \ref{theorem inner s le 1} is the best, we only need the first two steps in the proof of Theorem \ref{theorem inner s ge 1}.\\

\textit{Final proof of Theorem \ref{theorem inner s le 1}.} In the following of the proof, we omit the dependence of constants on $s, \sigma, \gamma, c_{0}, \kappa, c_{1} $ and $ c_{2}$.
$ $\\
\textbf{Step 1}:
Note that we assume $ 0 < s < 1$ in this theorem and $ \lambda = d^{-l}, 0<l<\gamma$. For specific range of $\gamma$, we discuss the range of $ l_{\text{balance}}$. 
\begin{itemize}[leftmargin = 18pt]
    \item When $ l \in ( p,  p + \frac{s}{2}]$ for some integer $ p \ge 0$, Lemma \ref{lemma calculation n1 n2} and Lemma \ref{lemma calculation m2} show that
    \begin{displaymath}
        \frac{\mathcal{N}_{2}(\lambda)}{n} \asymp d^{p-\gamma};~~ \mathcal{M}_{2}(\lambda) \asymp d^{-2l+(2-s)p},
    \end{displaymath}
    thus we have
    \begin{equation}
        l_{\text{balance}} = \frac{\gamma + p-ps}{2}.
    \end{equation}
    Further, letting $l_{\text{balance}} = \frac{\gamma + p -ps}{2} \in ( p,  p + \frac{s}{2}] $, we have 
    \begin{equation}
        \gamma \in \left( p+ps, p+ps+s \right].
    \end{equation}

    \item When $ l \in ( p+ \frac{s}{2},  p + \frac{1}{2}]$, Lemma \ref{lemma calculation n1 n2} and Lemma \ref{lemma calculation m2} show that
    \begin{displaymath}
        \frac{\mathcal{N}_{2}(\lambda)}{n} \asymp d^{p-\gamma};~~ \mathcal{M}_{2}(\lambda) \asymp d^{-(p+1)s},
    \end{displaymath}
    thus the above two terms are equal if and only if
    \begin{equation}
        \gamma = p+ps+s.
    \end{equation}

    \item When $ l \in (p+ \frac{1}{2},  p + 1]$, Lemma \ref{lemma calculation n1 n2} and Lemma \ref{lemma calculation m2} show that
    \begin{displaymath}
        \frac{\mathcal{N}_{2}(\lambda)}{n} \asymp d^{2l-p-1-\gamma};~~ \mathcal{M}_{2}(\lambda) \asymp d^{-(p+1)s},
    \end{displaymath}
    thus we have
    \begin{equation}
        l_{\text{balance}} = \frac{\gamma + (p+1)(1-s)}{2}.
    \end{equation}
    Further, letting $l_{\text{balance}} \in ( p+ \frac{s}{2},  p + 1] $, we have 
    \begin{equation}
        \gamma \in \left( p+ps+s, (p+1) + (p+1)s \right].
    \end{equation}
\end{itemize}
Note that the present result is different from the result of the Step 1 in the proof of Theorem \eqref{theorem inner s ge 1}. There are only two intervals of $\gamma$, i.e.,
\begin{displaymath}
    \gamma \in \left( p+ps, p+ps+s \right]; ~~\text{and}~~ \gamma \in \left( p+ps+s, (p+1) + (p+1)s \right].
\end{displaymath}

It is worth mentioning that in the second interval of $\gamma$, we can actually choose $\lambda = d^{-l}, \forall l \in \left[ p+\frac{s}{2}, l_{\text{balance}} \right]$ and we have
\begin{equation}
    \frac{\mathcal{N}_{2}(\lambda)}{n} \lesssim \mathcal{M}_{2}(\lambda); ~~ \mathcal{M}_{2}(\lambda) = \mathcal{M}_{2}(\lambda_{\text{balance}}).
\end{equation}
That is to say, we can choose smaller $l$ and the rate of $ \mathcal{N}_{2}(\lambda) / n + \mathcal{M}_{2}(\lambda)$ will remain unchanged. We have shown in \eqref{mono of l} and the discussion below it that the approximation conditions are easier to satisfied for smaller $l$. Therefore, in the following of the proof, we define
\begin{equation}
    l_{\text{opt}} = p+\frac{s}{2},~~ \text{when}~~ \gamma \in \left( p+ps+s, (p+1) + (p+1)s \right],
\end{equation}
and verify the approximation conditions for $ \lambda_{\text{opt}} = d^{-l_{\text{opt}}}$. For consistency of notation, we also define
\begin{equation}
    l_{\text{opt}} = \frac{\gamma+p-ps}{2},~~ \text{when}~~ \gamma \in \left( p+ps, p+ps+s \right].
\end{equation}
$ $\\
\textbf{Step 2}:
In order to apply Theorem \ref{theorem variance approximation} and Theorem \ref{theorem bias approximation s le 1} so that we know the exact convergence rates of $ \mathbf{Var}(\lambda_{\text{opt}})$ and $\mathbf{Bias}^{2}(\lambda_{\text{opt}})$, we first check the approximation conditions \eqref{var conditions}, \eqref{bias conditions s le 1 in theorem} and \eqref{bias conditions s le 1-2 in theorem} hold for $ l = l_{\text{opt}}$. We first list all the approximation conditions below:
\begin{equation}\label{list again conditions}
\begin{aligned}
&\frac{\mathcal{N}_{1}(\lambda)}{n} \ln{n} = o(1); ~~n^{-1} \mathcal{N}_{1}(\lambda)^{2} \ln{n} = o(\mathcal{N}_{2}(\lambda));\\
&n^{-1} \mathcal{N}_{1}(\lambda)^{\frac{1}{2}} \left\| f_{\lambda} \right\|_{L^{\infty}}  = o\left(\mathcal{M}_{2}(\lambda)^{\frac{1}{2}}\right); n^{-1} \mathcal{N}_{1}(\lambda)^{\frac{1}{2}} n^{\frac{1-s}{2}+\epsilon} = o\left(\mathcal{M}_{2}(\lambda)^{\frac{1}{2}}\right).
\end{aligned}
\end{equation}
Recall that we have calculated the convergence rates of $\mathcal{N}_{1}(\lambda)$ and $\|f_{\lambda}\|_{L^{\infty}}$ in Lemma \ref{lemma calculation n1 n2} and Lemma \ref{lemma calculation flambda s le 1}.

\begin{itemize}[leftmargin = 18pt]
    \item When $\gamma \in \left( p+ps, p+ps+s \right] $: recall that $ l_{\text{opt}} = \frac{\gamma+p-ps}{2} \in (p, p+\frac{s}{2}]$.
    
    $\left( \text{\lowercase\expandafter{\romannumeral1}} \right)~$The first condition in \eqref{list again conditions} is equivalent to
    \begin{displaymath}
        \frac{\gamma+p-ps}{2} < \gamma \iff \gamma > p-ps,
    \end{displaymath}
    which naturally holds for all $\gamma \in \left( p+ps, p+ps+s \right]  $.\\
    $ $\\
    $\left( \text{\lowercase\expandafter{\romannumeral2}} \right)~$The second condition in \eqref{list again conditions} is equivalent to
    \begin{displaymath}
        d^{-\gamma} \cdot d^{\gamma+p-ps} \cdot \gamma \ln{d} \ll d^{p} \iff p-ps < p,
    \end{displaymath}
    which naturally holds for all $\gamma \in \left( p+ps, p+ps+s \right]  $ and $ p \neq 0$. When $ p = 0$, we actually need to choose $ \lambda_{\text{opt}} = d^{-l_{\text{opt}}} \cdot \ln{d}$ and the second condition will hold.\\
    $ $\\
    $\left( \text{\lowercase\expandafter{\romannumeral3}} \right)~$The third condition in \eqref{list again conditions} is equivalent to
    \begin{align}
        d^{-\gamma} \cdot d^{\frac{\gamma+p-ps}{4}}  \cdot \left[ d^{\frac{p}{2}-\frac{ps}{2}} + d^{\frac{\gamma+p-ps}{2}-\frac{(1+s)(p+1)}{2}} \right] \ll d^{-\frac{1}{2}(\gamma+p-ps) + \frac{(2-s)p}{2}} \notag 
    \end{align}
    $ \iff $
    \begin{displaymath}
        \gamma > p-3ps;~~ \gamma < p+5ps+2s+2,
    \end{displaymath}
    which naturally holds for all $\gamma \in\left( p+ps, p+ps+s \right] $ and $ p \neq 0$. In addition, one can also check that the third condition in \eqref{list again conditions} holds when $ p = 0$ and $ \lambda_{\text{opt}} = d^{-l_{\text{opt}}} \cdot \ln{d}$.
    $ $\\
    $\left( \text{\lowercase\expandafter{\romannumeral4}} \right)~$The forth condition in \eqref{list again conditions} is equivalent to
    \begin{align}
        d^{-\gamma} \cdot d^{\frac{\gamma+p-ps}{4}} \cdot  d^{\frac{\gamma}{2}-\frac{\gamma s}{2}} \ll d^{-\frac{1}{2}(\gamma+p-ps) + \frac{(2-s)p}{2}} \notag 
    \end{align}
    $ \iff $
    \begin{equation}\label{condition 4-1}
        (1-2s)\gamma < p+ps.
    \end{equation}
    If $ \frac{1}{2} < s < 1$, \eqref{condition 4-1} naturally holds for all $\gamma \in\left( p+ps, p+ps+s \right] $ and $ p \neq 0$. In addition, one can also check that the forth condition in \eqref{list again conditions} holds when $ p = 0$ and $ \lambda_{\text{opt}} = d^{-l_{\text{opt}}} \cdot \ln{d}$.\\
    If $ 0 < s \le \frac{1}{2}$, \eqref{condition 4-1} only holds for $\gamma \in\left( p+ps, p+ps+s \right] $ and $ p \neq 0$. That is to say, we can not verify \eqref{condition 4-1} holds for
    \begin{equation}\label{can not 1}
        \gamma \in (0,s], ~~ 0<s<\frac{1}{2}.
    \end{equation}

    \item When $\gamma \in ( p+ps+s, (p+1)+(p+1)s ] $: recall that $ l_{\text{opt}} = p+\frac{s}{2}$.
    
    $\left( \text{\lowercase\expandafter{\romannumeral1}} \right)~$The first condition in \eqref{list again conditions} is equivalent to
    \begin{displaymath}
        p+\frac{s}{2} < \gamma ,
    \end{displaymath}
    which naturally holds for all $\gamma \in ( p+ps+s, (p+1)+(p+1)s ]$.\\
    $ $\\
    $\left( \text{\lowercase\expandafter{\romannumeral2}} \right)~$The second condition in \eqref{list again conditions} is equivalent to
    \begin{displaymath}
        d^{-\gamma} \cdot d^{ 2p+s} \cdot \gamma \ln{d} \ll d^{ p} \iff \gamma > p+s,
    \end{displaymath}
    which naturally holds for all $\gamma \in ( p+ps+s, (p+1)+(p+1)s ]$.\\
    $ $\\
    $\left( \text{\lowercase\expandafter{\romannumeral3}} \right)~$The third condition in \eqref{list again conditions} is equivalent to
    \begin{align}
        d^{-\gamma} \cdot d^{\frac{p}{2}+\frac{s}{4}}  \cdot \left[ d^{\frac{p}{2}-\frac{ps}{2}} + d^{p+\frac{s}{2}-\frac{(1+s)(p+1)}{2}} \right] \ll d^{-\frac{(p+1)s}{2}} \notag 
    \end{align}
    $ \iff $
    \begin{displaymath}
        \gamma > p+\frac{3s}{4};~~ \gamma > p+\frac{3s}{4}-\frac{1}{2},
    \end{displaymath}
    which naturally holds for all $\gamma \in ( p+ps+s, (p+1)+(p+1)s ] $.
    $ $\\
    $\left( \text{\lowercase\expandafter{\romannumeral4}} \right)~$The forth condition in \eqref{list again conditions} is equivalent to
    \begin{align}
        d^{-\gamma} \cdot d^{\frac{p}{2}+\frac{s}{4}} \cdot  d^{\frac{\gamma}{2}-\frac{\gamma s}{2}} \ll d^{-\frac{(p+1)s}{2}} \notag 
    \end{align}
    $ \iff $
    \begin{equation}\label{condition 4-2}
        \gamma > \frac{2p+3s+2ps}{2(s+1)}.
    \end{equation}
    If $ 1/2 < s < 1$, \eqref{condition 4-2} naturally holds for all $\gamma \in ( p+ps+s, (p+1)+(p+1)s ] $ and $ p \neq 0$. In addition, one can also check that the forth condition in \eqref{list again conditions} holds when $ p = 0$ and $ \lambda_{\text{opt}} = d^{-l_{\text{opt}}} \cdot \ln{d}$.\\
    If $ 0 < s \le 1/2$, \eqref{condition 4-2} only holds for $\gamma \in ( p+ps+s, (p+1)+(p+1)s ] $ and $ p \neq 0$. That is to say, we can not verify \eqref{condition 4-1} holds for
    \begin{equation}\label{can not 2}
        \gamma \in \left(0,\frac{3s}{2(s+1)}\right], ~~ 0<s<\frac{1}{2}.
    \end{equation}
\end{itemize}
Up to now, we have verified the approximation conditions \eqref{list again conditions} for
\begin{align}
    \forall \gamma &> 0, ~~\text{if}~~ \frac{1}{2} < s < 1; 
\end{align}
and
\begin{equation}
    \forall \gamma > \frac{3s}{2(s+1)}, ~~\text{if}~~ 0<s\le\frac{1}{2}.
\end{equation}
Using Lemma \ref{lemma calculation m2} to calculate the rate of $ \mathcal{M}_{2}(\lambda_{\text{opt}})$, we finish the proof.\\

$\hfill\blacksquare$

\section{Proof of Minimax lower bound}\label{section proofs minimax}
\subsection{More preliminaries about minimax lower bound}
Let's first introduce several concepts about minimax lower bound which can be frequently found in related literature \cite{Yang_Density_1999,lu2023optimal}, etc..

Suppose that $(\mathcal{Z},d)$ is a topological space with a compatible loss function $d$, which are mappings from $ \mathcal{Z} \times  \mathcal{Z}$ to  $\mathbb{R}_{\geq 0}$ with $d(f, f)=0$ and $d(f, f^{\prime}) >0$ for $f \neq f^{\prime}$. We call such a loss function a \textit{distance}. We introduce the packing entropy and covering entropy below:

\begin{definition}[Packing entropy]
A finite set $N_{\epsilon} \subset \mathcal{Z}$ is said to be an $\epsilon$-packing set in $\mathcal{Z}$ with separation $\epsilon>0$, if for any $f, f^{\prime} \in N_{\epsilon}, f \neq f^{\prime}$, we have $d\left(f, f^{\prime}\right)>\epsilon$. The logarithm of the maximum cardinality of $\epsilon$-packing set is called the $\epsilon$-packing entropy or Kolmogorov capacity of $\mathcal{Z}$ with distance $d$ and is denoted by 
$M_{d}(\epsilon,\mathcal{Z})$.
\end{definition}

\begin{definition}[Covering entropy]\label{def:covering_entropy}
A set $G_{\epsilon} \subset \mathcal{Z}$ is said to be an $\epsilon$-net for $\mathcal{Z}$ if for any $\tilde{f} \in \mathcal{Z}$, there exists an $f_0 \in G_{\epsilon}$ such that $d(\tilde{f}, f_0) \leq \epsilon$. The logarithm of the minimum cardinality of $\epsilon$-net is called the $\epsilon$-covering entropy of $\mathcal{Z}$ and is denoted by
$V_{d}(\epsilon,\mathcal{Z})$.
\end{definition}

Let $\mathcal{B} = \left\{ f \in \mathcal{H},~ \| f \|_{[\mathcal{H}]^{s}} \le R_{\gamma}\right\}$, where $ R_{\gamma}$ is the constant from Assumption \ref{assumption source condition}.  Without loss of generality, we can consider $\mathcal{B}$ be the unit ball in $[\mathcal{H}]^{s}$. Let $M_{2}(\epsilon,\mathcal{B})$ be the $\epsilon$-packing entropy of $(\mathcal{B}, d^2=\|\cdot\|_{L^2}^2)$ and $V_{2}(\epsilon,\mathcal{B})$ be the $\epsilon$-covering entropy of $(\mathcal{B}, d^2=\|\cdot\|_{L^2}^2)$. Recalling that $\mu$ is the marginal distribution on $\mathcal{X},$ we further define
\begin{align*}
    \mathcal{D}=\left\{ \rho_{f}~\bigg|~ \mbox{ joint distribution of $(y,\boldsymbol{x}$) where } \boldsymbol{x}\sim \mu, y=f(\boldsymbol{x})+\epsilon, \epsilon\sim N(0,\sigma^{2}),
    f\in \mathcal{B} \right\},
\end{align*}
and let $V_{K}(\epsilon,\mathcal{D})$ be the $\epsilon$-covering entropy of $(\mathcal{D}, d^2=\text{ KL divergence })$. It is easy to see that $\mathcal{D} $ is an subset of $\mathcal{P}$ which is defined in Theorem \ref{theorem lower bound}, i.e., $\mathcal{D} \subset \mathcal{P} $.

The following lemmas give useful characterizations of $M_{2}(\epsilon,\mathcal{B}), V_{2}(\epsilon,\mathcal{B}) $ and $ V_{K}(\epsilon,\mathcal{D}) $. We refer to Lemma A.5, Lemma A.7 and Lemma A.8 in \cite{lu2023optimal} for their proofs.

\begin{lemma}\label{lemma_M_2_and_V_2} 
For any $\epsilon>0$, we have $M_{2}(2\epsilon,\mathcal{B}) \leq V_{2}(\epsilon,\mathcal{B}) \leq M_{2}(\epsilon,\mathcal{B}).$

\end{lemma}

\begin{lemma}\label{claim:d_K_and_d_2}
$V_{2}\left(\epsilon, \mathcal{B} \right) = V_{K}\left(\frac{\epsilon}{\sqrt{2}\sigma}, \mathcal{D}\right)$.    
\end{lemma}

\begin{lemma}
\label{lemma_entropy_of_RKHS}
Let $\{\lambda_{j}\}_{j=1}^{\infty} $ be the eigenvalues of $\mathcal{H}$. For any $\epsilon>0$, let $K(\epsilon)=\frac{1}{2}\sum\limits_{j: \lambda_j^{s} > \epsilon^2} \ln\left({\lambda_j^{s}}/{\epsilon^2}\right)$. We have
\begin{equation}\label{eqn:137}
	\begin{aligned}
 V_{2}(6\epsilon, \mathcal{B})\leq K(\epsilon) \leq  V_{2}(\epsilon, \mathcal{B}).
	\end{aligned}
\end{equation}
\end{lemma}

The following important lemma is a modification of Theorem 1 and Corollary 1 in \cite{Yang_Density_1999}. We refer to Lemma 4.1 in \cite{lu2023optimal} for the proof.
\begin{lemma}\label{thm_lower_ultimate_tech}
Let $\mathfrak{c}\in (0,1)$ be  a constant only depending on $c_{1}$, $c_{2}$, and $\gamma$, where $c_{1}, c_{2} $ are the constants given in Theorem \ref{theorem lower bound}. For any $0<\tilde\epsilon_1, \tilde\epsilon_2<\infty$ 
only depending on $n$, $d$, $\{\lambda_j\}$, $c_{1}$, $c_{2}$, and $\gamma$
and satisfying
\begin{equation}
    \frac{V_K(\tilde\epsilon_2, \mathcal{D}) + n\tilde\epsilon_2^2 + \ln{2}}{V_2(\tilde\epsilon_1, \mathcal{B})} \leq \mathfrak{c},
\end{equation}
 we have 
\begin{equation}
\min _{\hat{f}} \max _{\rho_{f^{*}} \in \mathcal{D}} \mathbb{E}_{(\boldsymbol{X}, \mathbf{y}) \sim \rho_{f^{*}}^{\otimes n}}
\left\|\hat{f} - f^{*}\right\|_{L^2}^2
\geq \frac{1 - \mathfrak{c}}{4} \tilde\epsilon_1^2.
\end{equation}
\end{lemma}

\subsection{Proof of Theorem \ref{theorem lower bound}}
Now we are ready to use the lemmas in the last subsection to prove Theorem \ref{theorem lower bound}. The proof is divided into two parts, dealing with the two cases of the interval in which $\gamma$ falls into.\\

\textit{Proof of Theorem \ref{theorem lower bound} $\left( \text{\lowercase\expandafter{\romannumeral1}}\right)$. } In this case, we have assumed $ \gamma \in \left( p+ps, p+ps+s \right]$ for some integer $ p \ge 0$. Let $\tilde\epsilon_{2} = C_{2} d^{-(\gamma-p)}$, where we will choose the constant $C_{2}$ later. Note that we have $ \gamma - p \in (ps, (p+1)s ]$. Lemma \ref{lemma inner eigen} implies that we can choose $C_{2}$ only depending on $p$ (ignoring the dependence on $\{a_{j}\}_{j=0}^{\infty}$) such that for any $d \geq \mathfrak{C}$ ($ \mathfrak{C} $ is a constant only depending on $\epsilon, s$ and $ p$), we have
\begin{equation}
    \mu_{p+1}^{s} < \tilde \epsilon_{2}^{2} < \mu_{p}^{s}.
\end{equation}
Next we can choose $\tilde\epsilon_{1} = d^{-(\gamma-p+\epsilon)} $, where $\epsilon$ can be any positive real number. Since $ \gamma-p+\epsilon > ps$, when $d \ge \mathfrak{C}$, where $ \mathfrak{C} $ is a constant only depending on $\epsilon, s$ and $ p$, we have
\begin{equation}
    \tilde \epsilon_{1}^{2} < \mu_{p}^{s}.
\end{equation}
Therefore, using Lemma \ref{lemma_entropy_of_RKHS} and Lemma \ref{lemma inner eigen}, for any $d \ge \mathfrak{C}$, we have
\begin{align}\label{lower case 1 eq 1}
    V_{2}\left( \tilde \epsilon_{1},\mathcal{B} \right) &\ge K(\tilde \epsilon_{1}) \ge \frac{1}{2} N(d,p) \ln{\left(\frac{\mu_{p}^{s}}{\tilde \epsilon_{1}^{2}}\right)} \notag \\
    &\ge \frac{1}{2} N(d,p) \ln{\left(\frac{\mathfrak{C}_{1} d^{-ps}}{d^{-(\gamma-p+\epsilon )}}\right)} \notag \\
    &= \frac{1}{2} N(d,p)  \left( \ln{\mathfrak{C}_{1}} + (\gamma-p+\epsilon-ps) \ln{d}     \right).
\end{align}
In addition, using Lemma \ref{lemma inner eigen} and Lemma \ref{lemma:monotone_of_eigenvalues_of_inner_product_kernels}, we have the following claim.
\begin{claim}\label{claim_1}
Suppose that $ \gamma \in \left( p+ps, p+ps+s \right]$ for some integer $ p \ge 0$. Let $\tilde\epsilon_2^2$ be defined as above. For any $\epsilon_{0} > 0$, there exists a sufficiently large constant $\mathfrak{C}$ only depending on $s, p$ and $\epsilon_{0}$, such that for any $d \geq \mathfrak{C}$, we have
\begin{equation*}
\begin{aligned}
&K\left( \sqrt{2}\sigma \tilde\epsilon_2 / 6 \right) \leq
 \left(1+\epsilon_{0} \right) \frac{1}{2} N(d,p)\log\left(\frac{18\mu_p^{s}}{\sigma^2\tilde\epsilon_2^2}\right).
\end{aligned}
\end{equation*}
\end{claim}
Therefore, for any $d \geq \mathfrak{C}$, where $\mathfrak{C}$ is a constant only depending on $s$ and $ p$, we have
\begin{align}\label{lower case 1 eq 2}
    V_{K}\left(\tilde \epsilon_{2}, \mathcal{D} \right) &= V_{2}\left( \sqrt{2} \sigma \tilde \epsilon_{2}, \mathcal{B} \right) \le K\left(\frac{\sqrt{2} \sigma \tilde \epsilon_{2}}{6}\right) \notag \\
    &\le \left(1+\epsilon_{0} \right) \frac{1}{2} N(d,p)\ln\left(\frac{18\mu_p^{s}}{\sigma^2\tilde\epsilon_2^2}\right) \notag \\
    &\le \left(1+\epsilon_{0} \right) \frac{1}{2} N(d,p)\ln\left(\frac{18 \mathfrak{C}_{2} d^{-ps}}{\sigma^2 C_{2} d^{-(\gamma-p)} }\right) \notag \\
    &= \left(1+\epsilon_{0} \right) \frac{1}{2} N(d,p) \left( \ln{\frac{18 \mathfrak{C}_{2}}{\sigma^2 C_{2}}} + (\gamma-p-ps) \ln{d} \right),
\end{align}
where we use Lemma \ref{claim:d_K_and_d_2} and Lemma \ref{lemma_entropy_of_RKHS} for the first line and use Lemma \ref{lemma inner eigen} for the third line.

Using \eqref{lower case 1 eq 1} and \eqref{lower case 1 eq 2}, also recalling that we assume $ c_{1} d^{\gamma} \le n \le c_{2} d^{\gamma}$, we have
\begin{equation}\label{lower case 1 eq 3}
    \frac{V_K(\tilde\epsilon_2, \mathcal{D}) + n\tilde\epsilon_2^2 + \ln{2}}{V_2(\tilde\epsilon_1, \mathcal{B})} \le \frac{ \left(1+\epsilon_{0} \right) \frac{1}{2} N(d,p) \left( \ln{\frac{18 \mathfrak{C}_{2}}{\sigma^2 C_{2}}} + (\gamma-p-ps) \ln{d} \right) + c_{2} d^{\gamma} \cdot C_{2} d^{-(\gamma-p)} + \ln{2} }{ \frac{1}{2} N(d,p)  \left( \ln{\mathfrak{C}_{1}} + (\gamma-p+\epsilon-ps) \ln{d} \right) }.
\end{equation}
Recalling that Lemma \ref{lemma Ndk} shows $\mathfrak{C}_{3}d^{p} \le N(d,p) \le \mathfrak{C}_{4}d^{p}$ when $d \ge \mathfrak{C}$, the dominant terms in \eqref{lower case 1 eq 3} are:
\begin{equation}
    \frac{ \frac{1}{2} \left(1+\epsilon_{0} \right)(\gamma-p-ps) N(d,p)  \ln{d}  }{ \frac{1}{2} (\gamma-p+\epsilon-p s) N(d,p)  \ln{d} }.
\end{equation}
Therefore, for any $ \epsilon > 0$, we can choose $ \epsilon_{0}$ small enough such that 
\begin{equation}
    \frac{V_K(\tilde\epsilon_2, \mathcal{D}) + n\tilde\epsilon_2^2 + \ln{2}}{V_2(\tilde\epsilon_1, \mathcal{B})} \le \eqref{lower case 1 eq 3} := \mathfrak{c}< 1.
\end{equation}
Then using Lemma \ref{thm_lower_ultimate_tech}, we have 
\begin{displaymath}
    \min _{\hat{f}} \max _{\rho_{f^{*}} \in \mathcal{D}} \mathbb{E}_{(\boldsymbol{X}, \mathbf{y}) \sim \rho_{f^{*}}^{\otimes n}}
\left\|\hat{f} - f^{*}\right\|_{L^2}^2
\geq \frac{1 - \mathfrak{c}}{4} \tilde\epsilon_1^2 = \frac{1 - \mathfrak{c}}{4} d^{-(\gamma-p-\epsilon)}.
\end{displaymath}
Further recalling that $ \mathcal{D} \subset \mathcal{P}$, we have 
\begin{equation}
    \min _{\hat{f}} \max _{\rho \in \mathcal{P}} \mathbb{E}_{(\boldsymbol{X}, \mathbf{y}) \sim \rho^{\otimes n}} \left\|\hat{f} - f_{\rho}^{*}\right\|_{L^2}^2 \ge \min _{\hat{f}} \max _{\rho_{f^{*}} \in \mathcal{D}} \mathbb{E}_{(\boldsymbol{X}, \mathbf{y}) \sim \rho_{f^{*}}^{\otimes n}} \left\|\hat{f} - f^{*}\right\|_{L^2}^2
\geq \frac{1 - \mathfrak{c}}{4} d^{-(\gamma-p-\epsilon)}.
\end{equation}
We finis the proof of Theorem \ref{theorem lower bound} $\left( \text{\lowercase\expandafter{\romannumeral1}}\right)$.\\

$\hfill\blacksquare$

\textit{Proof of Theorem \ref{theorem lower bound} $\left( \text{\lowercase\expandafter{\romannumeral2}}\right)$. } In this case, we have assumed $ \gamma \in \left( p+ps+s,(p+1)+(p+1)s \right]$ for some integer $ p \ge 0$. Let $\tilde\epsilon_{2} = C_{2} d^{-(p+1)s} \ln{d}$, where we will choose the constant $C_{2}$ later. Then Lemma \ref{lemma inner eigen} implies that there exists a constant $\mathfrak{C} $ only depending on $s $ and $ p$ such that for any $d \geq \mathfrak{C}$, we have
\begin{equation}
    \mu_{p+1}^{s} < \tilde \epsilon_{2}^{2} < \mu_{p}^{s}.
\end{equation}
Next we can choose $\tilde\epsilon_{1} = C_{1} d^{-(p+1)s} $. Using Lemma \ref{lemma inner eigen}, we can choose $C_{1} < \mathfrak{C}_{1}^{s}$, where $\mathfrak{C}_{1} $ is the constant in Lemma \ref{lemma inner eigen}, such that for any $d \ge \mathfrak{C}$, where $ \mathfrak{C} $ is a constant only depending on $s$ and $p$, we have
\begin{equation}
    \tilde \epsilon_{1}^{2} < \mu_{p+1}^{s}.
\end{equation}
Therefore, using Lemma \ref{lemma_entropy_of_RKHS} and Lemma \ref{lemma inner eigen}, for any $d \ge \mathfrak{C}$, we have
\begin{align}\label{lower case 2 eq 1}
    V_{2}\left( \tilde \epsilon_{1},\mathcal{B} \right) &\ge K(\tilde \epsilon_{1}) \ge \frac{1}{2} N(d,p+1) \ln{\left(\frac{\mu_{p+1}^{s}}{\tilde \epsilon_{1}^{2}}\right)} \notag \\
    &\ge \frac{1}{2} N(d,p+1) \ln{\left(\frac{\mathfrak{C}_{1} d^{-(p+1)s}}{C_{1} d^{-(p+1)s}}\right)} \notag \\
    &= \frac{1}{2} N(d,p+1) \ln{\frac{\mathfrak{C}_{1}}{C_{1}}}.
\end{align}
In addition, using Lemma \ref{lemma inner eigen} and Lemma \ref{lemma:monotone_of_eigenvalues_of_inner_product_kernels}, we have the following claim.
\begin{claim}\label{claim_2}
Suppose that $ \gamma \in \left( p+ps+s, (p+1)+(p+1)s \right]$ for some integer $ p \ge 0$. Let $\tilde\epsilon_2^2$ be defined as above. For any $\epsilon_{0} > 0$, there exists a sufficiently large constant $\mathfrak{C}$ only depending on $s, p$ and $\epsilon_{0}$, such that for any $d \geq \mathfrak{C}$, we have
\begin{equation*}
\begin{aligned}
&K\left( \sqrt{2}\sigma \tilde\epsilon_2 / 6 \right) \leq
 \left(1+\epsilon_{0} \right) \frac{1}{2} N(d,p)\log\left(\frac{18\mu_p^{s}}{\sigma^2\tilde\epsilon_2^2}\right).
\end{aligned}
\end{equation*}
\end{claim}
Therefore, for any $d \geq \mathfrak{C}$, where $\mathfrak{C}$ is a constant only depending on $s, p$ and $\{a_{j}\}_{j \le p+1}$, we have
\begin{align}\label{lower case 2 eq 2}
    V_{K}\left(\tilde \epsilon_{2}, \mathcal{D} \right) &= V_{2}\left( \sqrt{2} \sigma \tilde \epsilon_{2}, \mathcal{B} \right) \le K\left(\frac{\sqrt{2} \sigma \tilde \epsilon_{2}}{6}\right) \notag \\
    &\le \left(1+\epsilon_{0} \right) \frac{1}{2} N(d,p)\ln\left(\frac{18\mu_p^{s}}{\sigma^2\tilde\epsilon_2^2}\right) \notag \\
    &\le \left(1+\epsilon_{0} \right) \frac{1}{2} N(d,p)\ln\left(\frac{18 \mathfrak{C}_{2} d^{-ps}}{\sigma^2 C_{2} d^{-(p+1)s} }\right) \notag \\
    &= \left(1+\epsilon_{0} \right) \frac{1}{2} N(d,p) \left( \ln{\frac{18 \mathfrak{C}_{2}}{\sigma^2 C_{2}}} + s \ln{d} \right),
\end{align}
where we use Lemma \ref{claim:d_K_and_d_2} and Lemma \ref{lemma_entropy_of_RKHS} for the first line and use Lemma \ref{lemma inner eigen} for the third line.

Using \eqref{lower case 1 eq 1} and \eqref{lower case 1 eq 2}, also recalling that we assume $ c_{1} d^{\gamma} \le n \le c_{2} d^{\gamma}$, we have
\begin{equation}\label{lower case 2 eq 3}
    \frac{V_K(\tilde\epsilon_2, \mathcal{D}) + n\tilde\epsilon_2^2 + \ln{2}}{V_2(\tilde\epsilon_1, \mathcal{B})} \le \frac{ \left(1+\epsilon_{0} \right) \frac{1}{2} N(d,p) \left( \ln{\frac{18 \mathfrak{C}_{2}}{\sigma^2 C_{2}}} + s \ln{d} \right) + c_{2} d^{\gamma} \cdot C_{2} d^{-(p+1)s} + \ln{2} }{ \frac{1}{2} N(d,p+1) \ln{\frac{\mathfrak{C}_{1}}{C_{1}}} }.
\end{equation}
Recalling that Lemma \ref{lemma Ndk} shows $ N(d,p) \le \mathfrak{C}_{4}d^{p}$ and $ N(d,p+1) \ge \mathfrak{C}_{3}d^{p+1}$ when $d \ge \mathfrak{C}$, the dominant terms in \eqref{lower case 2 eq 3} are:
\begin{equation}
    \frac{ c_{2} C_{2} d^{\gamma-(p+1)s}}{ \frac{1}{2} \ln{\frac{\mathfrak{C}_{1}}{C_{1}}} N(d,p+1)  \ln{d} }.
\end{equation}
Further noticing that $\gamma-(p+1)s \le p+1$ for any $ \gamma \in (p+ps+s, (p+1)+(p+1)s]$, so we can choose $ C_{2}$ small enough and only depending on $s, \sigma, \gamma, \kappa, c_{1}, c_{2} $, such that 
\begin{equation}
    \frac{V_K(\tilde\epsilon_2, \mathcal{D}) + n\tilde\epsilon_2^2 + \ln{2}}{V_2(\tilde\epsilon_1, \mathcal{B})} \le \eqref{lower case 2 eq 3} := \mathfrak{c}< 1.
\end{equation}
Then using Lemma \ref{thm_lower_ultimate_tech} again, we have 
\begin{displaymath}
    \min _{\hat{f}} \max _{\rho_{f^{*}} \in \mathcal{D}} \mathbb{E}_{(\boldsymbol{X}, \mathbf{y}) \sim \rho_{f^{*}}^{\otimes n}}
\left\|\hat{f} - f^{*}\right\|_{L^2}^2
\geq \frac{1 - \mathfrak{c}}{4} \tilde\epsilon_1^2 = \frac{1 - \mathfrak{c}}{4} C_{1} d^{-(p+1)s}.
\end{displaymath}
Further recalling that $ \mathcal{D} \subset \mathcal{P}$, we have 
\begin{equation}
    \min _{\hat{f}} \max _{\rho \in \mathcal{P}} \mathbb{E}_{(\boldsymbol{X}, \mathbf{y}) \sim \rho^{\otimes n}} \left\|\hat{f} - f_{\rho}^{*}\right\|_{L^2}^2 \ge \min _{\hat{f}} \max _{\rho_{f^{*}} \in \mathcal{D}} \mathbb{E}_{(\boldsymbol{X}, \mathbf{y}) \sim \rho_{f^{*}}^{\otimes n}} \left\|\hat{f} - f^{*}\right\|_{L^2}^2
\geq \frac{1 - \mathfrak{c}}{4} C_{1} d^{-(p+1)s}.
\end{equation}
We finis the proof of Theorem \ref{theorem lower bound} $\left( \text{\lowercase\expandafter{\romannumeral2}}\right)$.\\  

$\hfill\blacksquare$

\section{Auxiliary results}\label{section auxiliart}

The following proposition about estimating the $L^{2}$ norm with empirical norm is from \citet[Proposition C.9]{li2023_SaturationEffect}, which dates back to \cite{caponnetto2010cross}.
\begin{proposition}
    \label{prop:SampleNormEstimation}
    Let $\mu$ be a probability measure on $\mathcal{X}$, $f \in L^2(\mathcal{X},\mu)$ and $\|{f}\|_{L^\infty} \leq M$.
    Suppose we have $\boldsymbol{x}_1,\dots,\boldsymbol{x}_n$ sampled i.i.d.\ from $\mu$.
    Then for $\delta \in (0,1)$, the following holds with probability at least $1-\delta$:
    \begin{align}
      \frac{1}{2}\|{f}\|_{L^2}^2 - \frac{5M^2}{3n}\ln \frac{2}{\delta} \leq \|{f}\|_{L^2,n}^2 \leq
      \frac{3}{2}\|{f}\|_{L^2}^2 + \frac{5M^2}{3n}\ln \frac{2}{\delta}.
    \end{align}
  \end{proposition}

The following concentration inequality about self-adjoint Hilbert-Schmidt operator
valued random variables is frequently used in related literature, e.g., \citet[Theorem 27]{fischer2020_SobolevNorm} and \citet[Lemma 26]{lin2020_OptimalConvergence}.
\begin{lemma}\label{lemma concentration of operator}
   Let $(\mathcal{X}, \mathcal{B}, \mu)$ be a probability space, $\mathcal{H}$ be a separable Hilbert space. Suppose that $ A_{1}, \cdots, A_{n}$ are i.i.d. random variables with values in the set of self-adjoint Hilbert-Schmidt operators. If  $\mathbb{E} A_{i} = 0$, and the operator norm $ \| A_{i} \| \le L ~~ \mu \text {-a.e. } \boldsymbol{x} \in \mathcal{X}$, and there exists a self-adjoint positive semi-definite trace class operator $V$ with $\mathbb{E} A_{i}^{2} \preceq V $. Then for $\delta \in (0,1)$, with probability at least $1 - \delta$, we have 
   \begin{align}
        \left\| \frac{1}{n}\sum_{i=1}^n A_i \right\|
        \leq \frac{2L\beta}{3n} + \sqrt {\frac{2 \| V \| \beta}{n}},\quad
        \beta = \ln \frac{4 \rm{tr} V}{\delta \| V \|}. \notag
   \end{align}
\end{lemma}

The following Bernstein inequality about vector-valued random variables is frequently used, e.g., \citet[Proposition 2]{Caponnetto2007OptimalRF} and \citet[Theorem 26]{fischer2020_SobolevNorm}.
\begin{lemma}[Bernstein inequality]\label{bernstein}
   Let $(\Omega,\mathcal{B},P)$ be a probability space, $H$ be a separable Hilbert space, and $\xi: \Omega \to H$ be a random variable with 
   \begin{displaymath}
     \mathbb{E}\|\xi\|_H^m \leq \frac{1}{2} m ! \sigma^2 L^{m-2},
   \end{displaymath}
   for all $m>2$. Then for $\delta \in (0,1)$, $\xi_{i}$ are i.i.d. random variables, with probability at least $1 - \delta$, we have
   \begin{displaymath}
       \left\|\frac{1}{n} \sum_{i=1}^n \xi_{i} - \mathbb{E} \xi\right\|_H \le 4\sqrt{2} \ln{\frac{2}{\delta}} \left(\frac{L}{n} + \frac{\sigma}{\sqrt{n}}\right).
   \end{displaymath}
\end{lemma}

\begin{lemma}\label{due embedding bound}
Given the definition of $\mathcal{N}_{1}(\lambda)$ as in \eqref{n1 n2 m1 m2}. If condition \eqref{assumption eigen - n1} in Assumption \ref{assumption eigenfunction} holds, we have
   \begin{align}\label{bound of Tk}
      \|T_{\lambda}^{-\frac{1}{2}} k(\boldsymbol{x},\cdot) \|_{\mathcal{H}}^{2} \le \mathcal{N}_{1}(\lambda),~~ \mu \text {-a.e. } \boldsymbol{x} \in \mathcal{X}.
  \end{align}
\end{lemma}
\begin{proof}
   \begin{align}
      \|T_{\lambda}^{-\frac{1}{2}} k(\boldsymbol{x},\cdot) \|_{\mathcal{H}}^{2} &= \Big\| \sum\limits_{i=1}^{\infty} ( \frac{1}{\lambda_{i} + \lambda})^{\frac{1}{2}} \lambda_{i} e_{i}(\boldsymbol{x}) e_{i}(\cdot)  \Big\|_{\mathcal{H}}^{2} \notag \\
      &=  \sum\limits_{i =1}^{\infty}  \frac{\lambda_{i}}{\lambda_{i} + \lambda} e_{i}^{2}(\boldsymbol{x}) \notag \\
      & \le \mathcal{N}_{1}(\lambda), \quad \mu \text {-a.e. } \boldsymbol{x} \in \mathcal{X}. \notag  
  \end{align}
\end{proof}

Lemma \ref{due embedding bound} has a direct corollary.
\begin{lemma}\label{emb norm}
Given the definition of $\mathcal{N}_{1}(\lambda)$ as in \eqref{n1 n2 m1 m2}. If condition \eqref{assumption eigen - n1} in Assumption \ref{assumption eigenfunction} holds, we have
\begin{displaymath}
    \| T_{\lambda}^{-\frac{1}{2}} T_{\boldsymbol{x}} T_{\lambda}^{-\frac{1}{2}}\| \le \mathcal{N}_{1}(\lambda), \quad \mu \text {-a.e. } \boldsymbol{x} \in \mathcal{X}.
\end{displaymath}
\end{lemma}
\begin{proof}
    Note that for any $f \in \mathcal{H}$,
  \begin{align}
      T_{\lambda}^{-\frac{1}{2}} T_{\boldsymbol{x}} T_{\lambda}^{-\frac{1}{2}} f &= T_{\lambda}^{-\frac{1}{2}} K_{\boldsymbol{x}} K_{\boldsymbol{x}}^{*}  T_{\lambda}^{-\frac{1}{2}} f \notag \\
      &= T_{\lambda}^{-\frac{1}{2}} K_{\boldsymbol{x}} \langle k(\boldsymbol{x},\cdot), T_{\lambda}^{-\frac{1}{2}} f \rangle_{\mathcal{H}} \notag \\
      &= T_{\lambda}^{-\frac{1}{2}} K_{\boldsymbol{x}} \langle T_{\lambda}^{-\frac{1}{2}} k(\boldsymbol{x},\cdot),  f \rangle_{\mathcal{H}} \notag \\
      &=  \langle T_{\lambda}^{-\frac{1}{2}} k(\boldsymbol{x},\cdot),  f \rangle_{\mathcal{H}} \cdot T_{\lambda}^{-\frac{1}{2}} k(\boldsymbol{x},\cdot). \notag
  \end{align}
  So $\| T_{\lambda}^{-\frac{1}{2}} T_{\boldsymbol{x}} T_{\lambda}^{-\frac{1}{2}} \| = \sup\limits_{\| f\|_{\mathcal{H}}=1} \| T_{\lambda}^{-\frac{1}{2}} T_{\boldsymbol{x}} T_{\lambda}^{-\frac{1}{2}} f\|_{\mathcal{H}} = \sup\limits_{\| f\|_{\mathcal{H}}=1} \langle T_{\lambda}^{-\frac{1}{2}} k(\boldsymbol{x},\cdot),  f \rangle_{\mathcal{H}} \cdot \|T_{\lambda}^{-\frac{1}{2}} k(\boldsymbol{x},\cdot) \|_{\mathcal{H}} = \|T_{\lambda}^{-\frac{1}{2}} k(\boldsymbol{x},\cdot) \|_{\mathcal{H}}^{2}$. 
  Using Lemma \ref{due embedding bound}, we finish the proof.
\end{proof}

We state the following three lemmas without proof, since the proofs are classical and the verification about the constants is tedious. We refer to Appendix A in \cite{zhang2023optimality} and the reference therein for the proof, the definition of \textit{Lorentz space} $L^{p,q}(\mathcal{X},\mu) $ and \textit{real interpolation} $\left(\cdot, \cdot \right)_{\theta, q}  $.
\begin{lemma}\label{Lp interpolation}
  Let $\mu $ be the probability distribution on $\mathcal{X}$. For $ 1 < p_{1} \neq p_{2} < \infty$, $ 1 \le q \le \infty$ and $ 0 < \theta < 1$, we have
  \begin{displaymath}
      \left( L^{p_{1}}(\mathcal{X},\mu), L^{p_{2}}(\mathcal{X},\mu) \right)_{\theta, q} \cong L^{p_{\theta},q}(\mathcal{X},\mu) ,\quad  \frac{1}{p_{\theta}} = \frac{1-\theta}{p_{1}} + \frac{\theta}{p_{2}},
  \end{displaymath}
  where $L^{p_{\theta},q}(\mathcal{X}, \mu)$ is the Lorentz space and the equivalent norm only involves absolute constants.
\end{lemma}

\begin{lemma}\label{mono of LS}
  Let $\mu $ be the probability distribution on $\mathcal{X}$. If $ 1 < p < \infty$ and $ 1\le q_{1} \le q_{2} \le \infty $, we have 
  \begin{displaymath}
      L^{p,q_{1}}(\mathcal{X},\mu) \hookrightarrow L^{p,q_{2}}(\mathcal{X},\mu),
  \end{displaymath}
  and the operator norm are upper bounded by an absolute constant.
\end{lemma}

\begin{lemma}\label{cong of lorentz}
  Let $\mu $ be the probability distribution on $\mathcal{X}$. For $ 1 < p < \infty$, we have
  \begin{displaymath}
      L^{p,p}(\mathcal{X},\mu) \cong L^{p}(\mathcal{X},\mu); \quad L^{p,\infty}(\mathcal{X},\mu) \cong L^{p,w}(\mathcal{X},\mu),
  \end{displaymath}
  where $L^{p,w}(\mathcal{X},\mu)$ denotes the weak $L^{p}$ space and the equivalent norm only involves absolute constants.
\end{lemma}

\begin{theorem}[$L^{q}$-embedding property]\label{integrability of Hs constants}
  Suppose that $\mathcal{H}$ is the RKHS associated with a continuous, positive-definite and symmetric kernel $k$ on a compact set $\mathcal{X} \subset \mathbb{R}^{d}$ and the probability distribution on $\mathcal{X}$ is $\mu$. Further suppose that $ \sup\limits_{\boldsymbol{x} \in \mathcal{X}} |k(x,x)| \le \kappa^{2}$, where $ \kappa$ is an absolute constant. Then for any $0 < s < 1$, we have 
  \begin{align}
      [\mathcal{H}]^{s} \hookrightarrow L^{q_{s}}(\mathcal{X}, \mu),\quad \forall q_{s} < \frac{2}{1 - s}, 
  \end{align}
  and there exists a constant $C_{s,\kappa}$ only depending on $s$ and $\kappa$, such that the operator norm of the embedding operator satisfies
  \begin{equation}
      \left\| [\mathcal{H}]^{s} \hookrightarrow L^{q_{s}}(\mathcal{X}, \mu) \right\| \le C_{s,\kappa}.
  \end{equation}
\end{theorem}
\begin{proof}
     Denote $\left(E_{0}, E_{1} \right)_{\theta, q}  $ as the real interpolation of two normed spaces. \citet[Theorem 4.6]{steinwart2012_MercerTheorem} shows that for $0 < s < 1$,
    \begin{equation}\label{inter relation}
      [\mathcal{H}]^{s} \cong \left(L^{2}(\mathcal{X},\mu), [\mathcal{H}]^{1} \right)_{s,2},
    \end{equation}
    where the equivalent norm involves constants only depending on $ s $.
  
  Since $ \sup\limits_{\boldsymbol{x} \in \mathcal{X}} |k(\boldsymbol{x},\boldsymbol{x})| \le \kappa^{2}$ implies that the operator norm of embedding $ I_{1}: \mathcal{H} \hookrightarrow L^{\infty} $ satisfies $ \left\| I_{1} \right\| \le \kappa^{2}$. Define $I_{2}: \left(L^{2}, \mathcal{H}\right)_{\theta,2} \hookrightarrow \left(L^{2}, L^{\infty}\right)_{\theta,2} $ for some $\theta \in (0,1)$, the definition of real interpolation through \textit{K-Method} (see Chapter 22 in \citealt{tartar2007introduction}) actually implies $ \| I_{2} \| \le \max\{1, \kappa^{2}\}$. Then any $ 0 < M < \infty$, using Lemma \ref{Lp interpolation}, we have
  \begin{align}
       [\mathcal{H}]^{s} \hookrightarrow \left(L^{2}(\mathcal{X}, \mu), L^{M}(\mathcal{X}, \mu) \right)_{s,2} \cong L^{q_{s}^{\prime},2}(\mathcal{X}, \mu), \notag
  \end{align}
  where $ \frac{1}{q_{s}^{\prime}} = \frac{1- s}{2} + \frac{s}{M}$.

  For any $ q_{s} < \frac{2}{1-s}$, we can choose $M$ large enough such that $ q_{s}^{\prime} > q_{s}$. Further, since $ 0<s<1$ and thus $ q_{s}^{\prime} > q_{s} > 2$, using Lemma \ref{mono of LS} and Lemma \ref{cong of lorentz}, we have
  \begin{displaymath}
      L^{q_{s}^{\prime},2}(\mathcal{X}, \mu) \hookrightarrow L^{q_{s}^{\prime},q_{s}^{\prime}}(\mathcal{X}, \mu) \cong L^{q_{s}^{\prime}}(\mathcal{X}, \mu) \hookrightarrow L^{q_{s}}(\mathcal{X}, \mu).
  \end{displaymath}
  We finish the proof.
\end{proof}

In the following, we provide some remark on Theorem \ref{integrability of Hs constants}.
\begin{remark}
    Intuitively, their should be a constant depending on the dimension $d$ in the operator norm of the embedding. For instance, their are extensive literature studying the dependence of the embedding constants on $d$ in the Sobolev type inequalities \citep{cotsiolis2004best,mizuguchi2016embedding,novak2018reproducing}.

    Denote $ (I-\Delta)^{-\frac{r}{2}}, r>0, $ as the Bessel potential operators (see, e.g., Section 2 of \citealt{cotsiolis2004best}). Then the fractional Sobolev space can be defined as $ H^{r}(\mathbb{R}^{d}) = \left\{ f \in L^{2}(\mathbb{R}^{d}) ~\Big|~ \| (I-\Delta)^{\frac{r}{2}} f\|_{L^{2}} < \infty \right\} $, $r>0$. It is well known that when $r > \frac{d}{2}$, $H^{r}(\mathbb{R}^{d}) $ is an RKHS with bounded kernel function. When $ r < \frac{d}{2}$, denote $I_{d,r}$ as the embedding from $ H^{r}(\mathbb{R}^{d}) $ to $ L^{\frac{2d}{d-2r}}(\mathbb{R}^{d})$. Theorem 1.1 in \cite{cotsiolis2004best} gives an upper bound of the operator norm $ \|I_{d,r}\|$ for any $ d>0, 0< r < \frac{d}{2}$ (note that $\| (-\Delta)^{\frac{r}{2}} f\|_{L^{2}} \le \| (I-\Delta)^{\frac{r}{2}} f\|_{L^{2}}$). Since for $0<s<1$, $ [H^{\frac{d}{2}}(\mathbb{R}^{d})]^{s} \cong (L^{2}(\mathbb{R}^{d}), H^{\frac{d}{2}}(\mathbb{R}^{d}))_{s,2} \cong H^{\frac{ds}{2}}(\mathbb{R}^{d}))$, letting $r=\frac{d}{2}$, we have
     \begin{displaymath}
         H^{r}(\mathbb{R}^{d}) = [\mathcal{H}_{d}]^{\frac{2}{3}}, ~~\text{where}~~ \mathcal{H}_{d} = H^{\frac{d}{2}}(\mathbb{R}^{d}) ~~\text{is an RKHS}.
     \end{displaymath}
     (With a little abusement of notation, we consider $ H^{\frac{d}{2}}(\mathbb{R}^{d}) $ as an RKHS). Then $ I_{d,\frac{d}{2}}$ can also be interpreted as 
     \begin{equation}
         I_{d,\frac{d}{2}}: [\mathcal{H}_{d}]^{s} \hookrightarrow L^{\frac{2}{1-s}}(\mathbb{R}^{d}), ~\text{with}~ s = \frac{2}{3}.
     \end{equation}
     Detailed calculation about the constant in Theorem 1.1 in \cite{cotsiolis2004best} shows that the operator norm of $ I_{d,\frac{d}{2}}$ decreases to 0, i.e., 
    \begin{equation}
        \|I_{d,\frac{d}{2}}\| \to 0, ~~ \text{as}~~ d \to \infty.
    \end{equation}
    This indicates that although the embedding norm may depend on $d$, it can always be upper bounded by a constant. This shows the consistency with Theorem \ref{integrability of Hs constants} in our paper. 
     
    So when will the embedding norm tend to 0 and when will it remain as a constant? We can get some inspiration from the operator norm of $ I_{1}: \mathcal{H}_{d} \hookrightarrow L^{\infty} $. Recall that we assume $ \sup\limits_{\boldsymbol{x} \in \mathcal{X}} |k(\boldsymbol{x},\boldsymbol{x})| \le \kappa^{2}$ in Theorem \ref{integrability of Hs constants}, where $ \kappa$ is an absolute constant. This directly implies $ \| I_{1}\| \le \kappa^{2}$. This assumption is appropriate for some RKHSs, for instance, the inner product kernel in Assumption \ref{assumption inner product kernel} and NTK on the sphere. For theses RKHSs, $ \sup\limits_{\boldsymbol{x} \in \mathcal{X}} |k(x,x)| \le \kappa^{2}$ will not change as $d \to \infty$. However, we conjecture that the bound $\| I_{1} \| \le \kappa^{2}$ may be too loose for other RKHSs when $d \to \infty$.
    
    Let us see the example of fractional Sobolev space again. 
    For $d, r \in \mathbb{N}$ with $ r >\frac{d}{2}$, denote $I_{d,r, \infty}$ as the embedding from $ H^{r}(\mathbb{R}^{d}) $ to $ L^{\infty}(\mathbb{R}^{d})$. Theorem 11 in \cite{novak2018reproducing} shows that $\left\|I_{d, r, \infty}\right\| \to 0$ as $d \to \infty$.
    Note that the definition of $ H^{r}(\mathbb{R}^{d}) $ in this section is actually the same as the definition in Section 4.1 of \cite{novak2018reproducing}.

\end{remark}

\newpage
\vskip 0.2in
\bibliography{sample}

\begin{thebibliography}{52}
\providecommand{\natexlab}[1]{#1}
\providecommand{\url}[1]{\texttt{#1}}
\expandafter\ifx\csname urlstyle\endcsname\relax
  \providecommand{\doi}[1]{doi: #1}\else
  \providecommand{\doi}{doi: \begingroup \urlstyle{rm}\Url}\fi

\bibitem[Aerni et~al.(2022)Aerni, Milanta, Donhauser, and
  Yang]{aerni2022strong}
M.~Aerni, M.~Milanta, K.~Donhauser, and F.~Yang.
\newblock Strong inductive biases provably prevent harmless interpolation.
\newblock In \emph{The Eleventh International Conference on Learning
  Representations}, 2022.

\bibitem[Arora et~al.(2019)Arora, Du, Hu, Li, Salakhutdinov, and
  Wang]{Arora_on_2019}
S.~Arora, S.~S. Du, W.~Hu, Z.~Li, R.~R. Salakhutdinov, and R.~Wang.
\newblock On exact computation with an infinitely wide neural net.
\newblock \emph{Advances in Neural Information Processing Systems}, 32, 2019.

\bibitem[Bartlett et~al.(2020)Bartlett, Long, Lugosi, and
  Tsigler]{bartlett2020benign}
P.~L. Bartlett, P.~M. Long, G.~Lugosi, and A.~Tsigler.
\newblock Benign overfitting in linear regression.
\newblock \emph{Proceedings of the National Academy of Sciences}, 117\penalty0
  (48):\penalty0 30063--30070, 2020.

\bibitem[Bauer et~al.(2007)Bauer, Pereverzyev, and
  Rosasco]{bauer2007_RegularizationAlgorithms}
F.~Bauer, S.~Pereverzyev, and L.~Rosasco.
\newblock On regularization algorithms in learning theory.
\newblock \emph{Journal of complexity}, 23\penalty0 (1):\penalty0 52--72, 2007.

\bibitem[Beaglehole et~al.(2023)Beaglehole, Belkin, and
  Pandit]{beaglehole2023inconsistency}
D.~Beaglehole, M.~Belkin, and P.~Pandit.
\newblock On the inconsistency of kernel ridgeless regression in fixed
  dimensions.
\newblock \emph{SIAM Journal on Mathematics of Data Science}, 5\penalty0
  (4):\penalty0 854--872, 2023.

\bibitem[Bietti and Mairal(2019)]{Bietti2019OnTI}
A.~Bietti and J.~Mairal.
\newblock On the inductive bias of neural tangent kernels.
\newblock \emph{Advances in Neural Information Processing Systems}, 32, 2019.

\bibitem[Bordelon et~al.(2020)Bordelon, Canatar, and
  Pehlevan]{Bordelon2020SpectrumDL}
B.~Bordelon, A.~Canatar, and C.~Pehlevan.
\newblock Spectrum dependent learning curves in kernel regression and wide
  neural networks.
\newblock In \emph{International Conference on Machine Learning}, pages
  1024--1034. PMLR, 2020.

\bibitem[Buchholz(2022)]{buchholz2022_KernelInterpolation}
S.~Buchholz.
\newblock Kernel interpolation in {{Sobolev}} spaces is not consistent in low
  dimensions.
\newblock In P.-L. Loh and M.~Raginsky, editors, \emph{Proceedings of Thirty
  Fifth Conference on Learning Theory}, volume 178 of \emph{Proceedings of
  Machine Learning Research}, pages 3410--3440. {PMLR}, July 2022.

\bibitem[Caponnetto(2006)]{caponnetto2006optimal}
A.~Caponnetto.
\newblock Optimal rates for regularization operators in learning theory.
\newblock Technical report, MASSACHUSETTS INST OF TECH CAMBRIDGE COMPUTER
  SCIENCE AND ARTIFICIAL~…, 2006.

\bibitem[Caponnetto and de~Vito(2007)]{Caponnetto2007OptimalRF}
A.~Caponnetto and E.~de~Vito.
\newblock Optimal rates for the regularized least-squares algorithm.
\newblock \emph{Foundations of Computational Mathematics}, 7:\penalty0
  331--368, 2007.

\bibitem[Caponnetto and Yao(2010)]{caponnetto2010cross}
A.~Caponnetto and Y.~Yao.
\newblock Cross-validation based adaptation for regularization operators in
  learning theory.
\newblock \emph{Analysis and Applications}, 8\penalty0 (02):\penalty0 161--183,
  2010.

\bibitem[Cotsiolis and Tavoularis(2004)]{cotsiolis2004best}
A.~Cotsiolis and N.~K. Tavoularis.
\newblock Best constants for sobolev inequalities for higher order fractional
  derivatives.
\newblock \emph{Journal of mathematical analysis and applications},
  295\penalty0 (1):\penalty0 225--236, 2004.

\bibitem[Cui et~al.(2021)Cui, Loureiro, Krzakala, and
  Zdeborov{\'a}]{Cui2021GeneralizationER}
H.~Cui, B.~Loureiro, F.~Krzakala, and L.~Zdeborov{\'a}.
\newblock Generalization error rates in kernel regression: The crossover from
  the noiseless to noisy regime.
\newblock \emph{Advances in Neural Information Processing Systems},
  34:\penalty0 10131--10143, 2021.

\bibitem[Dai and Xu(2013)]{dai2013_ApproximationTheory}
F.~Dai and Y.~Xu.
\newblock \emph{Approximation Theory and Harmonic Analysis on Spheres and
  Balls}.
\newblock Springer {{Monographs}} in {{Mathematics}}. {Springer New York}, {New
  York, NY}, 2013.
\newblock ISBN 978-1-4614-6659-8 978-1-4614-6660-4.
\newblock \doi{10.1007/978-1-4614-6660-4}.

\bibitem[Donhauser et~al.(2021)Donhauser, Wu, and Yang]{Donhauser_how_2021}
K.~Donhauser, M.~Wu, and F.~Yang.
\newblock How rotational invariance of common kernels prevents generalization
  in high dimensions.
\newblock In \emph{International Conference on Machine Learning}, pages
  2804--2814. PMLR, 2021.

\bibitem[Fischer and Steinwart(2020)]{fischer2020_SobolevNorm}
S.-R. Fischer and I.~Steinwart.
\newblock Sobolev norm learning rates for regularized least-squares algorithms.
\newblock \emph{Journal of Machine Learning Research}, 21:\penalty0
  205:1--205:38, 2020.

\bibitem[Gallier et~al.(2020)Gallier, Quaintance, Gallier, and
  Quaintance]{Gallier2009SphericalHA}
J.~Gallier, J.~Quaintance, J.~Gallier, and J.~Quaintance.
\newblock Spherical harmonics and linear representations of lie groups.
\newblock \emph{Differential Geometry and Lie Groups: A Second Course}, pages
  265--360, 2020.

\bibitem[Gerfo et~al.(2008)Gerfo, Rosasco, Odone, Vito, and
  Verri]{gerfo2008_SpectralAlgorithms}
L.~L. Gerfo, L.~Rosasco, F.~Odone, E.~D. Vito, and A.~Verri.
\newblock Spectral algorithms for supervised learning.
\newblock \emph{Neural Computation}, 20\penalty0 (7):\penalty0 1873--1897,
  2008.

\bibitem[Ghorbani et~al.(2020)Ghorbani, Mei, Misiakiewicz, and
  Montanari]{Ghorbani_When_2021}
B.~Ghorbani, S.~Mei, T.~Misiakiewicz, and A.~Montanari.
\newblock When do neural networks outperform kernel methods?
\newblock \emph{Advances in Neural Information Processing Systems},
  33:\penalty0 14820--14830, 2020.

\bibitem[Ghorbani et~al.(2021)Ghorbani, Mei, Misiakiewicz, and
  Montanari]{Ghorbani2019LinearizedTN}
B.~Ghorbani, S.~Mei, T.~Misiakiewicz, and A.~Montanari.
\newblock {Linearized two-layers neural networks in high dimension}.
\newblock \emph{The Annals of Statistics}, 49\penalty0 (2):\penalty0 1029 --
  1054, 2021.
\newblock \doi{10.1214/20-AOS1990}.
\newblock URL \url{https://doi.org/10.1214/20-AOS1990}.

\bibitem[Ghosh et~al.(2021)Ghosh, Mei, and Yu]{Ghosh_three_2021}
N.~Ghosh, S.~Mei, and B.~Yu.
\newblock The three stages of learning dynamics in high-dimensional kernel
  methods.
\newblock In \emph{International Conference on Learning Representations}, 2021.

\bibitem[Hastie et~al.(2022)Hastie, Montanari, Rosset, and
  Tibshirani]{10.1214/21-AOS2133}
T.~Hastie, A.~Montanari, S.~Rosset, and R.~J. Tibshirani.
\newblock {Surprises in high-dimensional ridgeless least squares
  interpolation}.
\newblock \emph{The Annals of Statistics}, 50\penalty0 (2):\penalty0 949 --
  986, 2022.
\newblock \doi{10.1214/21-AOS2133}.
\newblock URL \url{https://doi.org/10.1214/21-AOS2133}.

\bibitem[Hu and Lu(2022)]{hu2022sharp}
H.~Hu and Y.~M. Lu.
\newblock Sharp asymptotics of kernel ridge regression beyond the linear
  regime.
\newblock \emph{arXiv preprint arXiv:2205.06798}, 2022.

\bibitem[Jacot et~al.(2018)Jacot, Gabriel, and
  Hongler]{jacot2018_NeuralTangent}
A.~Jacot, F.~Gabriel, and C.~Hongler.
\newblock Neural tangent kernel: {{Convergence}} and generalization in neural
  networks.
\newblock In S.~Bengio, H.~Wallach, H.~Larochelle, K.~Grauman,
  N.~{Cesa-Bianchi}, and R.~Garnett, editors, \emph{Advances in Neural
  Information Processing Systems}, volume~31. {Curran Associates, Inc.}, 2018.

\bibitem[Karoui(2010)]{Karoui_spectrum_2010}
N.~E. Karoui.
\newblock {The spectrum of kernel random matrices}.
\newblock \emph{The Annals of Statistics}, 38\penalty0 (1):\penalty0 1 -- 50,
  2010.
\newblock \doi{10.1214/08-AOS648}.
\newblock URL \url{https://doi.org/10.1214/08-AOS648}.

\bibitem[Lai et~al.(2023)Lai, Xu, Chen, and Lin]{jianfa2022generalization}
J.~Lai, M.~Xu, R.~Chen, and Q.~Lin.
\newblock Generalization ability of wide neural networks on $\mathbb{R}$.
\newblock \emph{arXiv preprint arXiv:2302.05933}, 2023.

\bibitem[Lee et~al.(2019)Lee, Xiao, Schoenholz, Bahri, Novak, Sohl-Dickstein,
  and Pennington]{lee2019wide}
J.~Lee, L.~Xiao, S.~Schoenholz, Y.~Bahri, R.~Novak, J.~Sohl-Dickstein, and
  J.~Pennington.
\newblock Wide neural networks of any depth evolve as linear models under
  gradient descent.
\newblock \emph{Advances in neural information processing systems}, 32, 2019.

\bibitem[Li et~al.(2023{\natexlab{a}})Li, Zhang, and Lin]{Li2023KernelIG}
Y.~Li, H.~Zhang, and Q.~Lin.
\newblock Kernel interpolation generalizes poorly.
\newblock \emph{arXiv preprint arXiv:2303.15809}, 2023{\natexlab{a}}.

\bibitem[Li et~al.(2023{\natexlab{b}})Li, Zhang, and
  Lin]{li2023_SaturationEffect}
Y.~Li, H.~Zhang, and Q.~Lin.
\newblock On the saturation effect of kernel ridge regression.
\newblock In \emph{International {{Conference}} on {{Learning
  Representations}}}, Feb. 2023{\natexlab{b}}.

\bibitem[Li et~al.(2023{\natexlab{c}})Li, Zhang, and Lin]{li2023asymptotic}
Y.~Li, H.~Zhang, and Q.~Lin.
\newblock On the asymptotic learning curves of kernel ridge regression under
  power-law decay.
\newblock In \emph{Thirty-seventh Conference on Neural Information Processing
  Systems}, 2023{\natexlab{c}}.

\bibitem[Liang and Rakhlin(2020)]{Liang_Just_2019}
T.~Liang and A.~Rakhlin.
\newblock {Just interpolate: Kernel “Ridgeless” regression can generalize}.
\newblock \emph{The Annals of Statistics}, 48\penalty0 (3):\penalty0 1329 --
  1347, 2020.
\newblock \doi{10.1214/19-AOS1849}.
\newblock URL \url{https://doi.org/10.1214/19-AOS1849}.

\bibitem[Liang et~al.(2020)Liang, Rakhlin, and Zhai]{liang2020multiple}
T.~Liang, A.~Rakhlin, and X.~Zhai.
\newblock On the multiple descent of minimum-norm interpolants and restricted
  lower isometry of kernels.
\newblock In \emph{Conference on Learning Theory}, pages 2683--2711. PMLR,
  2020.

\bibitem[Lin and Cevher(2020)]{lin2020_OptimalConvergence}
J.~Lin and V.~Cevher.
\newblock Optimal convergence for distributed learning with stochastic gradient
  methods and spectral algorithms.
\newblock \emph{Journal of Machine Learning Research}, 21:\penalty0 147--1,
  2020.

\bibitem[Lin et~al.(2018)Lin, Rudi, Rosasco, and Cevher]{lin2018_OptimalRates}
J.~Lin, A.~Rudi, L.~Rosasco, and V.~Cevher.
\newblock Optimal rates for spectral algorithms with least-squares regression
  over {{Hilbert}} spaces.
\newblock \emph{Applied and Computational Harmonic Analysis}, 48:\penalty0
  868--890, 2018.

\bibitem[Liu et~al.(2021)Liu, Liao, and Suykens]{Liu_kernel_2021}
F.~Liu, Z.~Liao, and J.~Suykens.
\newblock Kernel regression in high dimensions: Refined analysis beyond double
  descent.
\newblock In \emph{International Conference on Artificial Intelligence and
  Statistics}, pages 649--657. PMLR, 2021.

\bibitem[Lu et~al.(2023)Lu, Zhang, Li, Xu, and Lin]{lu2023optimal}
W.~Lu, H.~Zhang, Y.~Li, M.~Xu, and Q.~Lin.
\newblock Optimal rate of kernel regression in large dimensions.
\newblock \emph{arXiv preprint arXiv:2309.04268}, 2023.

\bibitem[Mei and Montanari(2022)]{mei2022generalization2}
S.~Mei and A.~Montanari.
\newblock The generalization error of random features regression: Precise
  asymptotics and the double descent curve.
\newblock \emph{Communications on Pure and Applied Mathematics}, 75\penalty0
  (4):\penalty0 667--766, 2022.

\bibitem[Mei et~al.(2022)Mei, Misiakiewicz, and
  Montanari]{mei2022generalization}
S.~Mei, T.~Misiakiewicz, and A.~Montanari.
\newblock Generalization error of random feature and kernel methods:
  Hypercontractivity and kernel matrix concentration.
\newblock \emph{Applied and Computational Harmonic Analysis}, 59:\penalty0
  3--84, 2022.

\bibitem[Misiakiewicz(2022)]{misiakiewicz_spectrum_2022}
T.~Misiakiewicz.
\newblock Spectrum of inner-product kernel matrices in the polynomial regime
  and multiple descent phenomenon in kernel ridge regression.
\newblock \emph{arXiv preprint arXiv:2204.10425}, 2022.

\bibitem[Mizuguchi et~al.(2016)Mizuguchi, Takayasu, Kubo, and
  Oishi]{mizuguchi2016embedding}
M.~Mizuguchi, A.~Takayasu, T.~Kubo, and S.~Oishi.
\newblock On the embedding constant of the sobolev type inequality for
  fractional derivatives.
\newblock \emph{Nonlinear Theory and Its Applications, IEICE}, 7\penalty0
  (3):\penalty0 386--394, 2016.

\bibitem[Muthukumar et~al.(2020)Muthukumar, Vodrahalli, Subramanian, and
  Sahai]{9051968}
V.~Muthukumar, K.~Vodrahalli, V.~Subramanian, and A.~Sahai.
\newblock Harmless interpolation of noisy data in regression.
\newblock \emph{IEEE Journal on Selected Areas in Information Theory},
  1\penalty0 (1):\penalty0 67--83, 2020.
\newblock \doi{10.1109/JSAIT.2020.2984716}.

\bibitem[Novak et~al.(2018)Novak, Ullrich, Wo{\'z}niakowski, and
  Zhang]{novak2018reproducing}
E.~Novak, M.~Ullrich, H.~Wo{\'z}niakowski, and S.~Zhang.
\newblock Reproducing kernels of sobolev spaces on r d and applications to
  embedding constants and tractability.
\newblock \emph{Analysis and Applications}, 16\penalty0 (05):\penalty0
  693--715, 2018.

\bibitem[Rakhlin and Zhai(2019)]{rakhlin2019consistency}
A.~Rakhlin and X.~Zhai.
\newblock Consistency of interpolation with laplace kernels is a
  high-dimensional phenomenon.
\newblock In \emph{Conference on Learning Theory}, pages 2595--2623. PMLR,
  2019.

\bibitem[Steinwart and Christmann(2008)]{Steinwart2008SupportVM}
I.~Steinwart and A.~Christmann.
\newblock Support vector machines.
\newblock In \emph{Information Science and Statistics}, 2008.

\bibitem[Steinwart and Scovel(2012)]{steinwart2012_MercerTheorem}
I.~Steinwart and C.~Scovel.
\newblock Mercer's theorem on general domains: {{On}} the interaction between
  measures, kernels, and {{RKHSs}}.
\newblock \emph{Constructive Approximation}, 35\penalty0 (3):\penalty0
  363--417, 2012.

\bibitem[Steinwart et~al.(2009)Steinwart, Hush, and
  Scovel]{steinwart2009_OptimalRates}
I.~Steinwart, D.~Hush, and C.~Scovel.
\newblock Optimal rates for regularized least squares regression.
\newblock In \emph{{{COLT}}}, pages 79--93, 2009.

\bibitem[Tartar(2007)]{tartar2007introduction}
L.~Tartar.
\newblock \emph{An introduction to Sobolev spaces and interpolation spaces},
  volume~3.
\newblock Springer Science \& Business Media, 2007.

\bibitem[Tsigler and Bartlett(2023)]{tsigler2020benign}
A.~Tsigler and P.~L. Bartlett.
\newblock Benign overfitting in ridge regression.
\newblock \emph{Journal of Machine Learning Research}, 24\penalty0
  (123):\penalty0 1--76, 2023.

\bibitem[Xiao et~al.(2022)Xiao, Hu, Misiakiewicz, Lu, and
  Pennington]{xiao2022precise}
L.~Xiao, H.~Hu, T.~Misiakiewicz, Y.~Lu, and J.~Pennington.
\newblock Precise learning curves and higher-order scaling limits for dot
  product kernel regression.
\newblock In \emph{Thirty-sixth Conference on Neural Information Processing
  Systems (NeurIPS)}, 2022.

\bibitem[Yang and Barron(1999)]{Yang_Density_1999}
Y.~Yang and A.~Barron.
\newblock {Information-theoretic determination of minimax rates of
  convergence}.
\newblock \emph{The Annals of Statistics}, 27\penalty0 (5):\penalty0 1564 --
  1599, 1999.
\newblock \doi{10.1214/aos/1017939142}.
\newblock URL \url{https://doi.org/10.1214/aos/1017939142}.

\bibitem[Zhang et~al.(2023{\natexlab{a}})Zhang, Li, and
  Lin]{zhang2023optimality}
H.~Zhang, Y.~Li, and Q.~Lin.
\newblock On the optimality of misspecified spectral algorithms.
\newblock \emph{arXiv preprint arXiv:2303.14942}, 2023{\natexlab{a}}.

\bibitem[Zhang et~al.(2023{\natexlab{b}})Zhang, Li, Lu, and
  Lin]{zhang2023optimality_2}
H.~Zhang, Y.~Li, W.~Lu, and Q.~Lin.
\newblock On the optimality of misspecified kernel ridge regression.
\newblock \emph{arXiv preprint arXiv:2305.07241}, 2023{\natexlab{b}}.

\end{thebibliography}

\end{document}